\documentclass[letterpaper]{article} 
\usepackage{aaai2026}  
\usepackage{times}  
\usepackage{helvet}  
\usepackage{courier}  
\usepackage[hyphens]{url}  
\usepackage{graphicx} 
\usepackage{natbib}  
\usepackage{caption} 
\usepackage{algorithm}
\usepackage{algorithmic}
\usepackage{hyperref}

\usepackage{newfloat}
\usepackage{listings}
\DeclareCaptionStyle{ruled}{labelfont=normalfont,labelsep=colon,strut=off} 
\floatstyle{ruled}
\newfloat{listing}{tb}{lst}{}
\floatname{listing}{Listing}
\usepackage{amsmath,amssymb}
\usepackage{xcolor}
\usepackage{srcltx}
\usepackage{etoolbox}
\usepackage{nameref}
\usepackage{comment}
\usepackage{mathrsfs}
\usepackage{enumerate}
\usepackage[shortlabels]{enumitem}
\usepackage{amsfonts} 
\usepackage{nicefrac} 
\usepackage{mathtools}
\usepackage{dsfont,mathrsfs}
\usepackage{cleveref}
\usepackage{xargs}
\usepackage{multirow}

\usepackage{aliascnt}
\usepackage{autonum}
\usepackage{upgreek}

\usepackage{subcaption}

\usepackage{amsthm}

\newtheorem{assum}{A\hspace{-2pt}}

\newtheorem{theorem}{Theorem}
\newtheorem{lemma}{Lemma}
\newtheorem{remark}{Remark}
\newtheorem{corollary}{Corollary}
\newtheorem{proposition}{Proposition}

\crefname{assum}{A\hspace{-2pt}}{A\hspace{-2pt}}
\Crefname{assum}{\textbf{A}\hspace{-2pt}}{\textbf{A}\hspace{-2pt}}
\crefname{assumb}{B\hspace{-2pt}}{B\hspace{-2pt}}
\crefname{assumUGE}{UGE\hspace{-1pt}}{UGE\hspace{-1pt}}
\crefname{assumID}{IND\hspace{-1pt}}{IND\hspace{-1pt}}
\crefname{assumUE}{UE\hspace{-1pt}}{UE\hspace{-1pt}}
\crefname{assumSUP}{M\hspace{-1pt}}{M\hspace{-1pt}}

\def\supconsteps{\supnorm{\funnoisew}}

\newcommand{\PE}{\mathbb{E}}
\newcommand{\PP}{\mathbb{P}}
\newcommandx{\genericb}[1][1=]{b_{#1}}
\newcommandx{\Constros}[1][1=]{\operatorname{C}_{\operatorname{Ros},#1}}
\newcommandx{\Constburk}[1][1=]{\operatorname{C}_{\operatorname{Burk}}}
\newcommandx{\driftW}[1][1=]{W_{#1}}

\def\Auxconst{\operatorname{c}}

\def\metricz{\mathsf{d}_{\Zset}}
\newcommandx{\metricd}[1][1=]{\mathsf{d}_{#1}}

\newcommandx\invmeasure[1][1=]{\Pi_{#1}}
\newcommandx{\PPjoint}[1][1=]{\PP^{\MKjoint[#1]}}
\newcommandx{\PEjoint}[1][1=]{\PE^{\MKjoint[#1]}}
\newcommandx{\PEMID}[1][1=\alpha]{\PE^{\MK[#1]}}
\newcommandx{\PPMID}[1][1=\alpha]{\PP^{\MK[#1]}}

\newcommand{\supnorm}[1]{\norm{ #1 }[\infty]}
\newcommandx{\MKjoint}[1][1=]{\bar{\operatorname{P}}_{#1}}
\newcommandx\costw[1][1=]{\mathsf{c}_{#1}}

\newcommandx\Intergrdist[1][1=]{\mathbb{M}_{1}(#1)}

\def\wasserdist{\mathbf{W}}
\newcommandx{\mmarkov}[1][1=0]{m^{(\Markov)}_{#1}}

\def\Xset{\mathsf{X}}
\def\Xsigma{\mathcal{X}}

\def\Zset{\mathsf{Z}}
\def\Zsigma{\mathcal{Z}}

\def\rset{\mathbb{R}}

\def\nset{\ensuremath{\mathbb{N}}}
\def\nsets{\ensuremath{\mathbb{N}^*}}
\def\nsetm{\ensuremath{\mathbb{N}_-}}
\def\zset{\ensuremath{\mathbb{Z}}}

\newcommand{\msi}{\mathsf{I}}
\newcommand{\msj}{\mathsf{J}}

\newcommand{\bConst}[1]{\operatorname{C}_{{\bf #1}}}

\newcommand{\smallConst}[1]{\operatorname{c}_{{#1}}}

\def\thetainit{\theta_0}

\newcommandx\sequence[4][2=,3=,4=]
{\ifthenelse{\equal{#3}{}}{\ensuremath{\{ #1_{#2 #4}\}}}{\ensuremath{\{ #1_{#2 #4} \}_{#2 \in #3}}}}

\newcommandx\sequenceD[2][2=]
{\ifthenelse{\equal{#2}{}}{\ensuremath{\{ #1\}}}{\ensuremath{\{ #1\!~:\!~#2  \}}}}
\newcommandx\sequenceDouble[4][3=,4=]
{\ifthenelse{\equal{#3}{}}{\ensuremath{\{ (#1_{#3},#2_{#3})\}}}{\ensuremath{\{ (#1_{#3},#2_{#3})\}_{#3 \in #4}}}}

\newcommandx{\sequencen}[2][2=n\in\nset]{\ensuremath{\{ #1, \eqsp #2 \}}}
\newcommandx\sequencens[2][2=n]
{\ensuremath{\{ #1_{#2} \!~:\!~#2\in\nsets\}}}
\newcommandx\sequencet[4]
{\ensuremath{\{ #1{#2_{#3}} \, : \, \eqsp #3 \in #4 \}}}
\def\PE{\mathbb{E}}

\def\ProdB{\Gamma}

\def\ProdBa{\ProdB^{(\alpha)}}

\newcommandx{\PVar}[1][1=]{\ensuremath{\operatorname{Var}_{#1}}}

\def\noisecov{\Sigma_\varepsilon}

\newcommandx{\MK}[1][1=\alpha]{\mathrm{P}_{#1}}
\newcommandx\MKK[1][1=\alpha]{\mathrm{K}_{#1}}
\def\MKQ{\mathrm{Q}}

\newcommandx{\PEtilde}[1][1=]{\PE^{\mathrm{K}_{#1}}}
\newcommandx{\PPtilde}[1][1=]{\PP^{\mathrm{K}_{#1}}}
\newcommand{\PEext}{\tilde{\PE}}
\newcommand{\PPext}{\tilde{\PP}}

\newcommandx{\norm}[2][2=]{\Vert#1 \Vert_{{#2}}}
\newcommandx{\normLigne}[2][2=]{\Vert#1 \Vert_{{#2}}}
\newcommandx{\normLine}[2][2=]{\Vert#1 \Vert_{{#2}}}
\newcommandx{\normop}[2][2=]{\Vert{#1}\Vert_{{#2}}}
\newcommandx{\normopLigne}[2][2=]{\Vert{#1}\Vert_{{#2}}}
\newcommandx{\normopLine}[2][2=]{\Vert{#1}\Vert_{{#2}}}
\newcommandx{\osc}[2][1=]{\mathrm{osc}_{#1}(#2)}

\newcommandx{\normlip}[2][2=\operatorname{Lip}]{\Vert#1 \Vert_{{#2}}}
\newcommand{\lip}{\operatorname{L}}
\newcommandx{\lipspace}[1]{\lip_{#1}}

\newcommandx{\CPP}[3][1=]
{\ifthenelse{\equal{#1}{}}{{\mathbb P}\left(\left. #2 \, \right| #3 \right)}{{\mathbb P}_{#1}\left(\left. #2 \, \right | #3 \right)}}
\newcommandx{\CPPtilde}[3][1=]
{\ifthenelse{\equal{#1}{}}{{\tilde{\mathbb P}}\left(\left. #2 \, \right| #3 \right)}{{\tilde{\mathbb P}}_{#1}\left(\left. #2 \, \right | #3 \right)}}

\newcommand{\ensemble}[2]{\left\{#1\,:\eqsp #2\right\}}
\newcommand{\setens}[2]{\ensemble{#1}{#2}}

\def\wrt{with respect to}

\def\rhs{right-hand side}

\newcommandx{\as}[1][1=\PP]{\ensuremath{#1\, -\mathrm{a.s.}}}

\newcommand{\eqsp}{\;}
\newcommand{\eqspp}{\ \ }

\newcommand{\Id}{\mathrm{I}}

\def\utheta{\tilde{\theta}^{\sf (tr)}}
\def\vtheta{\tilde{\theta}^{\sf (fl)}}

\newcommand{\ConstD}{\mathsf{D}}
\newcommand{\ConstDM}{\ConstD^{(\operatorname{M})}}

\newcommandx{\boundmetric}[1][1=]{\kappa_{\MKK[#1]}}

\newcommand{\Jnalpha}[2]{J_{#1}^{(#2)}}
\newcommand{\Jnalphat}[2]{\tilde{J}_{#1}^{(#2)}}
\newcommand{\Hnalpha}[2]{H_{#1}^{(#2)}}
\newcommand{\thalpha}[2]{\theta_{#1}^{(#2)}}
\newcommand{\thalphat}[2]{\tilde{\theta}_{#1}^{(#2)}}
\newcommand{\prthalpha}[2]{\bar{\theta}_{#1}^{(#2)}}

\newcommand{\Slnalpha}[3]{S_{#1:#2}^{(#3)}}
\newcommand{\Slnalphaext}[3]{\tilde{S}_{#1:#2}^{(#3)}}

\newcommand{\Tterm}[2]{T_{#1}^{(#2)}}
\newcommand{\Ttermext}[2]{\tilde{T}_{#1}^{(#2)}}

\newcommand{\ocint}[1]{\left(#1\right]}

\newcommand{\ccint}[1]{\left[#1\right]}

\newcommandx{\Nnorm}[2][1=V]{[ #2]_{#1}}
\newcommandx{\lipnorm}[2][1=g]{[ #1]_{#2}}

\newcommandx{\CPE}[3][1=]{{\mathbb E}^{#3}_{#1}\left[#2\right]}
\newcommandx{\CPEext}[3][1=]{\tilde{\mathbb E}^{#3}_{#1}\left[#2\right]}
\newcommandx{\CPEtilde}[3][1=]{{\tilde{\mathbb E}}^{#3}_{#1}\left[#2\right]}
\newcommandx{\CPEs}[3][1=]{{\mathbb E}^{#3}_{#1}[#2]}
\newcommandx\rate[1][1=\alpha]{\rho_{#1}}

\def\thetalim{\theta^\star}

\def\trace{\operatorname{Tr}}

\newcommand{\tvnorm}[1]{\left\Vert #1 \right\Vert_{\mathrm{TV}}}

\newcommand{\rme}{\mathrm{e}}
\newcommand{\rmd}{\mathrm{d}}
\def\funcAw{\mathbf{A}}
\def\funcBw{\mathbf{B}}
\newcommand{\funcA}[1]{\funcAw(#1)}
\newcommand{\funcF}[2]{\mathbf{F}_{#1}(#2)}
\def\funcbw{\mathbf{b}}
\newcommand{\funcb}[1]{\funcbw(#1)}
\newcommandx{\zmfuncA}[2][1=]{\tilde{\funcAw}^{#1}(#2)}
\newcommandx{\zmfuncAw}[1][1=]{\tilde{\funcAw}_{#1}}
\newcommandx{\zmfuncb}[2][1=]{\tilde{\funcbw}^{#1}(#2)}
\def\funnoisew{\varepsilon}
\newcommand{\funcnoise}[1]{\funnoisew(#1)}

\newcommandx{\funcct}[2][1=]{\funcctilde^{#1}(#2)}

\def\trace{\operatorname{Tr}}

\def\qcond{\kappa_{Q}}

\def\State{Z}

\def\taumix{t_{\operatorname{mix}}}

\newcommand{\indiacc}[1]{\boldsymbol{1}_{\{#1\}}}

\newcommandx{\CovC}[1][1=u]{\operatorname{C}_{#1}}

\def\msa{\mathsf{A}}
\def\msb{\mathsf{B}}
\def\msz{\mathsf{Z}}
\def\mcz{\mathcal{Z}}
\newcommand\borel[1]{\mathcal{B}(#1)}

\def\plusinfty{+\infty}

\DeclareMathAlphabet{\mathpzc}{OT1}{pzc}{m}{it}

\def\lyapW{\mathpzc{W}}

\newcommand{\txts}{\textstyle}

\newcommandx{\bias}[1][1=\alpha]{\operatorname{B}_{#1}}

\newcommandx\probaMarkovTilde[2][2=]
{\ifthenelse{\equal{#2}{}}{{\widetilde{\mathbb{P}}_{#1}}}{\widetilde{\mathbb{P}}_{#1}\left[ #2\right]}}

\def\mcf{\mathcal{F}}

\def\half{\nicefrac{1}{2}}

\def\bA{\bar{\mathbf{A}}}

\def\thetas{\thetalim}
\def\sphere{\mathbb{S}}

\def\transpose{\top}

\def\funcctilde{\tilde{c}_u}

\def\barb{\bar{\mathbf{b}}}
\newcommandx{\driftb}[1][1=p]{\bar{b}_{#1}}

\def\transpose{\top}

\def\barA{\bar{A}}

\def\eps{\varepsilon}

\newcommandx{\boldb}[1][1={q}]{\mathsf{b}_{#1}}

\newcommandx{\ConstGW}[1][1={n,\lyapW}]{\operatorname{G}_{#1}}

\newcommandx{\ConstMW}[1][1={n,\lyapW}]{\operatorname{M}_{#1}}

\Crefname{assumptionC}{\textbf{C}\hspace{-1pt}}{\textbf{C}\hspace{-1pt}}
\crefname{assumptionC}{\textbf{C}}{\textbf{C}}

\newtheorem{assumptionM}{\textbf{UGE}\hspace{-1pt}}
\Crefname{assumptionM}{\textbf{UGE}\hspace{-1pt}}{\textbf{UGE}\hspace{-1pt}}
\crefname{assumptionM}{\textbf{UGE}}{\textbf{UGE}}

\def\distance{\mathsf{d}}

\newcommandx{\vartconstwas}[1][1=V]{c_{#1}}

\newcommandx{\deltawas}[1][1=*]{\delta_{#1}}

\newcommandx{\wasser}[4][1=\distance,4=]{\mathbf{W}_{#1}^{#4}\left(#2,#3\right)}
\newcommandx{\covcoeff}[2]{\rho_{#1}^{(#2)}}

\def\smallAconst{\smallConst{\funcAw}}

\def\QQ{\mathbb{Q}}
\def\shift{\operatorname{S}}
\def\mcg{\mathcal{G}}
\newcommand{\restric}[2]{\left(#1\right)_{#2}}
\newcommand{\dobrush}{\mathsf{\Delta}}
\newcommandx{\dobru}[3][1=,3=]{\dobrush_{#1}^{#3}( #2)}  

\def\invariantQ{\pi}

\def\eqdef{:=}

\newcommand{\Etralpha}[1]{E^{({\sf tr},#1)}}

\newcommand{\Eflalpha}[1]{E^{({\sf fl},#1)}}

\def\tZs{\tilde{Z}^{\star}}
\def\tZ{\tilde{Z}}
\def\tmszn{\tilde{\msz}_{\nset}}
\def\tmczn{\tilde{\mcz}_{\nset}}

\def\Markov{\mathrm{M}}

\def\btheta{\bar{\theta}}

\def\rhotw{\rho}
\newcommand{\PEcoupling}[2]{\tilde{\PE}_{#1,#2}}
\newcommand{\PPcoupling}[2]{\tilde{\PP}_{#1,#2}}

\newcommandx{\dlim}[1]{\ensuremath{\stackrel{#1}{\Longrightarrow}}}

\newif\ifshowappendix
\showappendixfalse 

\title{High-Order Error Bounds for Markovian LSA with Richardson-Romberg Extrapolation}
\author{
    Ilya Levin\textsuperscript{\rm 1}, 
    Alexey Naumov\textsuperscript{\rm 1}, 
    Sergey Samsonov\textsuperscript{\rm 1}
}
\affiliations{
    \textsuperscript{\rm 1}HSE University\\
    ivlevin@hse.ru, anaumov@hse.ru, svsamsonov@hse.ru

}

\begin{document}

\maketitle

\begin{abstract}
In this paper, we study the bias and high-order error bounds of the Linear Stochastic Approximation (LSA) algorithm with Polyak-Ruppert (PR) averaging under Markovian noise. We focus on the version of the algorithm with constant step size and propose a novel decomposition of the bias via a linearization technique. We analyze the structure of the bias and show that the leading-order term is linear in the step size and cannot be eliminated by PR averaging. To address this, we apply the Richardson-Romberg (RR) extrapolation procedure, which effectively cancels the leading bias term. We derive high-order moment bounds for the RR iterates and show that the leading error term aligns with the asymptotically optimal covariance matrix of the vanilla averaged LSA iterates. We validate applicability of our findings for the temporal difference algorithm in reinforcement learning.
\end{abstract}

\begin{links}
    \link{Extended version}{https://arxiv.org/abs/2508.05570}
\end{links}

\section{Introduction}
Stochastic approximation (SA) algorithms \cite{robbins1951stochastic} play a foundational role in modern machine learning due to their various applications in reinforcement learning \cite{sutton:book:2018} and empirical risk minimization. In this paper, we consider the simplified setting of linear SA (LSA) algorithms, which estimate a solution of the linear system $\bA \thetalim = \barb$. For a sequence of step sizes $\{\alpha_k\}_{k \in \nset}$, a burn-in period $n_0 \in \nset$, and an initialization $\theta_0 \in \rset^{d}$, we consider the sequences of estimates $\{\theta_{k} \}_{k \in \nset}$ and $\{ \btheta_{n} \}_{n \geq n_0+1}$ given by
\begin{equation}
\label{eq:lsa}
\begin{split}
\textstyle \theta_{k} &= \theta_{k-1} - \alpha_k \{ \funcA{Z_k} \theta_{k-1} - \funcb{Z_k} \} \eqsp,~~ k \geq 1, \\
\textstyle \btheta_{n} &= (n-n_0)^{-1} \sum_{k=n_0}^{n-1} \theta_{k} \eqsp, ~~n \geq n_0+1 \eqsp.
\end{split}
\end{equation}
Here, $\btheta_{n}$ corresponds to the Polyak-Ruppert averaged estimator \cite{ruppert1988efficient,polyak1992acceleration}, a popular instrument for accelerating the convergence of stochastic approximation algorithms. In \eqref{eq:lsa}, $\{Z_k\}_{k \in \nset}$ is a sequence of random variables taking values in some measurable space $(\Zset,\Zsigma)$, and $\funcA{Z_k}$ and $\funcb{Z_k}$ are stochastic estimates of $\bA$ and $\barb$, respectively. In this paper, we focus on the setting where $\{Z_k\}_{k \in \nset}$ is a Markov chain.
\par
One of the key questions related to the recurrence \eqref{eq:lsa} is the choice of step sizes $\{\alpha_k\}_{k \in \nset}$. While the classical SA schemes \cite{robbins1951stochastic,polyak1992acceleration} correspond to the setting of decreasing step sizes, a lot of recent contributions \cite{huo2024collusion,meyn_bias_sa} focus on the setting of constant step sizes $\alpha_k = \alpha > 0$. This setting is of particular interest because it enables geometrically fast forgetting of the initialization \cite{dieuleveut2020bridging} and is often easier to use in practice. At the same time, the solution of the SA problem obtained with a constant step size suffers from an inevitable \emph{bias}, which arises in non-linear problems \cite{dieuleveut2020bridging} or even in linear SA \eqref{eq:lsa} when the sequence of noise variables $\{Z_k\}_{k \in \nset}$ forms a Markov chain, see e.g., \cite{meyn_bias_sa,durmus2025finite,huo2023bias}. This problem can be partially mitigated using the Richardson-Romberg (RR) extrapolation method. To formally define this method, we denote the LSA iterations \eqref{eq:lsa} with a constant step size $\alpha$ and define the corresponding Polyak-Ruppert averaged iterates as
\begin{align}
\label{eq:lsa_alpha}
&\thalpha{k}{\alpha} = \thalpha{k-1}{\alpha} - \alpha \{ \funcA{Z_k} \thalpha{k-1}{\alpha} - \funcb{Z_k} \} \eqsp,\\
&\prthalpha{n}{\alpha} = (n-n_0)^{-1} \sum_{k=n_0}^{n-1} \thalpha{k}{\alpha}\eqsp.
\end{align}
The next steps of the Richardson-Romberg (RR) procedure rely on the fact that the bias of $\prthalpha{n}{\alpha}$ is linear in $\alpha$ and is of order $\mathcal{O}(\alpha)$, see e.g., \cite{huo2023bias}. To proceed further, a learner considers two sequences $\{\thalpha{k}{\alpha}, k \in \nset\}$ and $\{\thalpha{k}{2\alpha}, k \in \nset \}$ with the same noise sequence $\{Z_k\}_{k \in \nset}$. Then for any $n \geq n_0 + 1$, one can set
\begin{align}
\prthalpha{n}{\alpha, \sf RR} = 2 \prthalpha{n}{\alpha} - \prthalpha{n}{2\alpha} \eqsp.
\end{align}
The non-asymptotic analysis of Richardson-Romberg extrapolation has recently attracted a lot of contributions in the context of linear SA \cite{huo2023bias}, stochastic gradient descent (SGD) \cite{durmusSGDRR2016,dieuleveut2020bridging}, and non-linear SA problems \cite{huo2024collusion,gast2024}. At the same time, a large and relatively unexplored gap is related to the question of the optimality of the leading term of the error bounds for $\prthalpha{n}{\alpha, \sf RR} - \thetas$. To properly define what "optimality" means in this context, note that in the context of linear SA problems with a decreasing step size \eqref{eq:lsa}, the sequence $\{\btheta_{n}\}_{n \in \nset}$ is asymptotically normal under appropriate conditions on $\{\alpha_k\}_{k \in \nset}$, that is
\begin{equation}
\sqrt{n}(\btheta_{n} - \thetalim) \xrightarrow{d} \mathcal{N}(0, \Sigma_{\infty}), \quad n \to \infty \eqsp.
\end{equation}
The covariance matrix $\Sigma_{\infty}$ here is known to be asymptotically optimal both in a sense of the Rao-Cramer lower bound and in a sense that it corresponds to the last iterate of the modified process $\tilde{\theta}_k$, which uses the optimal preconditioner matrix ($\barA^{-1}$ in the context of linear SA). Details can be found in the papers \cite{polyak1992acceleration,fort:clt:markov:2015}. A precise expression for $\Sigma_{\infty}$ is given later in the current paper, see \eqref{eq:sigma_inf}. It is known for SGD methods with i.i.d. noise and averaging that the Richardson-Romberg estimator achieves mean-squared error bounds (MSE) with the leading term, which aligns with $\Sigma_{\infty}$; that is,
\[
\PE^{1/2}[\norm{\prthalpha{n}{\alpha, \sf RR} - \thetas}^2] \leq \frac{\sqrt{\trace{\Sigma_{\infty}}}}{\sqrt{n}} + \mathcal{O}\left(\frac{1}{n^{1/2+\delta}}\right)\eqsp,
\]
for some $\delta > 0$. This result is due to \cite{sheshukova2024nonasymptotic}. To the best of our knowledge, there is no result of this kind available for the setting of Markovian SA. In this paper, we aim to close this gap for the setting of linear SA, yet we expect that the developed method can be useful for a more general setting. The main contributions of this paper are as follows:

\begin{itemize}
\item We propose a novel technique to quantify the asymptotic bias of $\thalpha{n}{\alpha}$. Our approach considers the limiting distribution $\Pi_{\alpha}$ of the joint Markov chain $\{(\thalpha{k}{\alpha},Z_{k+1})\}_{k\in \nset}$ and analyzes the bias $\Pi_{\alpha}(\theta_0) - \thetas$. Then, we apply the linearization method for $\thalpha{k}{\alpha}$ from \cite{aguech2000perturbation}. This allows us to study the limiting distribution of the components, whose average values are shown to be ordered by powers of $\alpha$.

\item We establish high-order moment error bounds for the Richardson-Romberg method, where the leading term aligns with the asymptotically optimal covariance $\Sigma_{\infty}$. We analyze its dependence on the number of steps $n$, step size $\alpha$, and the mixing time $\taumix$.
\end{itemize}

\section{Related Work}
The stochastic approximation scheme is widely studied for reinforcement learning (RL) \cite{sutton:td:1988,sutton:book:2018}. The well-known Temporal-Difference (TD) algorithm with linear function approximation \cite{bertsekas1996neuro} can be represented as the LSA problem. Originally, this method was proposed in \cite{robbins1951stochastic} with a diminishing step size. While asymptotic convergence results were first studied, non-asymptotic analysis later became of particular interest. For general SA, non-asymptotic bounds were investigated in \cite{moulines2011non,gadat2023optimal}. For LSA with a constant step size, finite-time analysis was presented in \cite{mou2020linear,mou2022optimal,durmus2025finite}.
\par
The bias and MSE for non-linear problems with i.i.d. noise have been studied for SGD in \cite{dieuleveut2020bridging,yu2021,sheshukova2024nonasymptotic}, and, recently, with both i.i.d. and Markovian noise in \cite{zhang2024constant,zhang2024prelimit,huo2024collusion,gast2024}. Another source of bias arises under Markovian noise and cannot be eliminated using averaging, as shown in \cite{meyn_bias_sa,lauand2023curse}. MSE bounds for Markovian LSA have been studied in several works, including \cite{srikant2019finite,mou2022optimal,durmus2025finite}. In \cite{mou2022optimal} and \cite{durmus2025finite}, the authors derive the leading term, which aligns with the optimal covariance $\Sigma_{\infty}$, but they do not eliminate the effect of the asymptotic bias.
\par
Further, when studying Markovian LSA, in \cite{meyn_bias_sa} the authors address the problem of bias, which can't be eliminated using PR averaging. In the work \cite{lauand2023curse}, the authors establish weak convergence of the Markov chain $(\theta_n, Z_{n+1})$ and also provide a decomposition for the limiting covariance of the iterations. In our work, we establish a similar result in \Cref{theo:existence_pi_alpha}. The work \cite{lauand2024revisiting} extends results on bias and convergence of Polyak-Ruppert iterations to diminishing step sizes $\alpha_k = \alpha_0 k^{-\rho}$ with $\rho \in (0, 1/2)$.
\par
The non-asymptotic analysis of Richardson-Romberg has been carried out in \cite{durmusSGDRR2016,huo2024collusion,sheshukova2024nonasymptotic,gast2024} for general SA, with particular applications to SGD. Further, in \cite{huo2023bias} and \cite{huo2023effectivenessconstantstepsizemarkovian}, the authors derive bounds for the LSA problem. In the work \cite{huo2024collusion}, the authors establish a bias decomposition for general SA up to the linear term in the step size $\alpha$ and derive MSE bounds dependent on $\alpha$ and the mixing time. For LSA, \cite{huo2023bias} extends this analysis by deriving a bias decomposition via an infinite series expansion in $\alpha$ and examining the MSE under the RR procedure, which eliminates arbitrary leading-order terms. Both works demonstrate that the RR technique accelerates convergence and maintains the proper scaling with the mixing time. However, neither work explicitly identifies the leading-term coefficient, and their results primarily address the improvement of higher-order terms in $\alpha$. Additionally, \cite{huo2023bias} imposes a restrictive reversibility assumption on the underlying Markov chain, limiting its applicability. Separately, \cite{huo2023effectivenessconstantstepsizemarkovian} explores the role of the RR procedure in statistical inference, particularly in constructing confidence intervals. Further, in \cite{zhang2024constant,kwontwo} authors consider the application of the RR procedure for Q-learning and two-timescale SA. A comparison of the bias decompositions known in the literature with our approach can be found in \Cref{par:bias_expansion_LSA}.

\section{Notations}
\label{sec:notation}
Consider a Polish space $\msz$ and a Markov kernel $\MKQ$ on $(\msz, \mcz)$ endowed with its Borel $\sigma$-field denoted by $\mcz$ and let $(\msz^{\nset}, \mcz^{\otimes \nset})$ be the corresponding canonical space. Consider a Markov kernel $\MKQ$ on $\msz \times \mcz$ and denote by $\PP_{\xi}$ and $\PE_{\xi}$ the corresponding probability distribution and expectation with initial distribution $\xi$. Without loss of generality, assume that $(Z_k)_{k \in \nset}$  is the associated canonical process. By construction, for any $\msa \in \mcz$, $\CPP[\xi]{Z_k \in \msa}{Z_{k-1}}= \MKQ(Z_{k-1},\msa)$, $\PP_\xi$-a.s. In the case $\xi = \updelta_z$, $z \in \msz$, $\PP_{\xi}$ and $\PE_{\xi}$ are  denoted by $\PP_{z}$ and $\PE_{z}$. Also, for any measurable space $(\Xset, \mathcal{G})$ with the signed measure $\mu$, we define the total variation norm $\norm{\mu}[\sf{TV}] = |\mu|(\Xset)$.

Let $(\Xset, \mathcal{G})$ be a complete separable metric space equipped with its Borel $\sigma$-algebra $\mathcal{G}$. We call $c: \Xset \times \Xset \to \rset_+$ a distance-like function, if it is symmetric, lower semi-continuous and $c(x, y) = 0$ if and only if $x = y$, and there exists $q \in \nset$ such that $d(x, y)^{q} \leq c(x, y)$. We denote by $\mathcal{H}(\xi, \xi')$ the set of couplings of probability measures $\xi$ and $\xi'$, that is, a set of probability measures on $(\Xset \times \Xset, \mathcal{G} \otimes \mathcal{G})$, such that for any $\Gamma \in \mathcal{H}(\xi, \xi')$ and any $A \in \mathcal{G}$ it holds $\Gamma(\Xset \times A) = \xi'(A)$ and $\Gamma(A \times \Xset) = \xi(A)$. We define the Wasserstein semimetric associated to the distance-like functiton $c^p(\cdot, \cdot)$, as
\begin{equation}
\label{eq:wasser_cost_def}
    \wasserdist_{c,p}(\xi, \xi') = \inf_{\Gamma \in \mathcal{H}(\xi, \xi')}\int_{\Xset \times \Xset} c^p(x, x')  \Gamma(dx, dx') \eqsp.
\end{equation}
We also denote $\wasserdist_{c}(\xi, \xi') \eqdef \wasserdist_{c,1}(\xi, \xi')$.

\section{Bias of the LSA Iterates}
\label{sec:stochastic-expansion}
In this section we aim to study the properties of the sequence $\thalpha{k}{\alpha}$ given by \eqref{eq:lsa_alpha} based on theory of Markov chains. Using the definition \eqref{eq:lsa_alpha} and some elementary algebra, we obtain
\begin{equation}
\label{eq:lsa-main-eq}
\thalpha{k}{\alpha} - \thetas = (\Id - \alpha \funcA{Z_k})(\thalpha{k-1}{\alpha} - \thetas) - \alpha \funcnoise{\State_{k}}\eqsp,
\end{equation}
where we have set
\begin{align}
\label{eq:def_center_version_and_noise}
\funcnoise{z} =  \zmfuncA{z} \thetas &- \zmfuncb{z}\eqsp, \quad \zmfuncA{z}  = \funcA{z} - \bA \eqsp,\\
 &\zmfuncb{z} = \funcb{z} - \barb \eqsp.
\end{align}
We consider the following assumptions on the noise variables $\{Z_k\}$:
\begin{assumptionM}
\label{assum:drift}
$\{Z_k\}_{k \in \nset}$ is a Markov chain with the Markov kernel $\MKQ$ taking values in complete separable metric space $(\Zset,\Zsigma)$. Moreover, $\MKQ$ admits $\invariantQ$ as an invariant distribution and is uniformly geometrically ergodic, that is, there exists $\taumix \in \nsets$ such that for all $k \in \nsets$,
\begin{equation}
\label{eq:drift-condition}
\dobru{\MKQ^k} \leq (1/4)^{\lfloor k / \taumix \rfloor} \eqsp,
\end{equation}
where $\dobru{\MKQ^k}$ is Dobrushin coefficient defined as
\begin{equation}
   \dobru{\MKQ^k} = \sup_{z,z' \in \Zset} (1/2)\norm{\MKQ^k(z, \cdot) - \MKQ^k(z',\cdot)}[\sf{TV}] \eqsp.
\end{equation}
Equivalently, there exist constants $\zeta > 0$ and $\rho \in (0, 1)$ such that for all $k \geq 1$,
\begin{equation}
\label{eq:def_rho}
    \sup_{z \in \Zset} \norm{\MKQ^k(z, \cdot) - \pi}[\sf{TV}] \leq \zeta \rho^{k} \eqsp.
\end{equation}
For convinience, we set $\rho = 4^{-1/\taumix}$.
\end{assumptionM}
Here, $\taumix$ is the mixing time of $\MKQ$.
\Cref{assum:drift} implies, in particular, that $\pi$ is the unique invariant distribution of $\MKQ$. We also define the noise covariance matrix 
\begin{equation}
\noisecov^{(\Markov)} = \PE_{\pi}[\funcnoise{Z_0}\funcnoise{Z_0}^T] + 2 \sum_{\ell=1}^{\infty} \PE_{\pi}[\funcnoise{Z_0}\funcnoise{Z_{\ell}}^T] \eqsp.
\end{equation}
This covariance is limiting for the sum $n^{-1/2}\sum_{t=0}^{n-1} \funcnoise{Z_t}$, see \cite{douc:moulines:priouret:soulier:2018}[Theorem 21.2.10]. Due to \cite{fort:clt:markov:2015}, the asymptotically optimal covariance matrix $\Sigma_{\infty}$ is defined as
\begin{equation}
\label{eq:sigma_inf}
    \Sigma_{\infty} = (\bA)^{-1} \noisecov^{(\Markov)}(\bA)^{-T} \eqsp.
\end{equation}
In the considered setting when $\{Z_k\}_{k \in \nset}$ is a Markov chain, the sequence $\{\thalpha{k}{\alpha}\}$ given by \eqref{eq:lsa-main-eq}, considered separately from $\{Z_k\}_{k \in \nset}$, might fail to be a Markov chain. This is not the case in the setting when $Z_k$ are i.i.d. random variables, see e.g. \cite{mou2020linear,durmus2025finite}. That is why, in the current paper we need to consider the joint process $(\thalpha{k}{\alpha},Z_{k+1})$, which is a Markov chain with the kernel $\MKjoint[\alpha]$, specified below. For any measurable and bounded function  $f: \rset^d \times \msz \to \rset_+$, $(\theta,z) \in \rset^d \times \msz$, we define $\MKjoint[\alpha]$ as
\begin{equation}
\label{eq:joint_MK_Markov}
\begin{split}
&\MKjoint[\alpha]f(\theta,z)  =  \int_{\msz}\MKQ(z, \rmd z')f(\funcF{z'}{\theta},z') \eqsp, \\
&\funcF{z}{\theta} = (\Id - \alpha \funcA{z}) \theta + \alpha \funcb{z} \eqsp.
\end{split}
\end{equation}
Thus, our next aim is to perform a quantitative analysis of $\MKjoint[\alpha]$. In particular, we show below that under appropriate regularity conditions, $\MKjoint[\alpha]$ admits a unique invariant distribution $\invmeasure[\alpha]$. Specifically, we impose the following assumptions:
\begin{assum}
\label{assum:A-b}
$\bConst{A} = \sup_{z \in \msz} \normop{\funcA{z}} \vee \sup_{z \in \msz} \normop{\zmfuncA{z}} < \infty$ and  the matrix $-\bA$ is Hurwitz.
\end{assum}

In particular, the condition that $-\bA$ is Hurwitz implies that the linear system $\bA \theta = \barb$ has a unique solution $\thetalim$. We further require the following assumptions on the noise term $\funcnoise{z}$ and the stationary distribution $\invariantQ$ of the sequence $\sequence{Z}[k][\nsets]$:
\begin{assum}
\label{assum:noise-level}
$\int_{\Zset}\funcA{z}\rmd \pi(z) = \bA$ and $\int_{\Zset}\funcb{z}\rmd \pi(z) = \barb$. Moreover, $\supconsteps = \sup\limits_{z \in \msz}\normop{\funcnoise{z}} < \plusinfty$.
\end{assum}

\begin{theorem}
\label{theo:existence_pi_alpha}
Assume \Cref{assum:A-b}, \Cref{assum:noise-level}, and \Cref{assum:drift}. Let $2 \leq p \leq q$. Then, for any $\alpha \in (0,(\alpha^{(M)}_{q,\infty} \wedge a^{-1})\taumix^{-1})$, the Markov kernel $\MKjoint[\alpha]$ admits a unique invariant distribution $\invmeasure[\alpha]$, such that  $\invmeasure[\alpha](\norm{\theta_0 - \thetas}) < \infty$. Here $\alpha^{(\Markov)}_{q,\infty}$ is a constant depending upon $q$ and other problem characteristics, and is defined in \eqref{eq:alpha_infty_makov}.
\end{theorem}
\begin{proof}[Proof sketch]
    We consider two noise sequences, $\{Z_n, n \in \nset\}$ and $\{\tilde{Z}_n, n \in \nset\}$, with a coupling time $T$. They evolve separately before time $T$ and coincide afterwards. See more details on coupling construction in \Cref{appendix:proof_existence_pi_alpha}. To prove the statement, we first establish the result on the contraction of the Wasserstein semimetric \eqref{eq:wasser_cost_def} with the cost function $\costw[0]$, defined as
    \begin{align}
    \costw[0]((\theta,z),(\theta',z')) &= (\|\theta - \theta'\| + \indiacc{z \neq z'})\\
    &\times \bigl(1 + \norm{\theta - \thetas} + \norm{\theta' - \thetas}\bigr)\eqsp,
   \end{align}
   where $(\theta,z), (\theta',z') \in \rset^{d} \times \Zset$. To do that, we consider two coupled Markov chains $\{(\thalpha{k}{\alpha}, Z_{k+1}), k \geq 0\}$ and $\{(\thalphat{k}{\alpha}, \tilde{Z}_{k+1}), k \geq 0\}$, starting from $(\theta, z)$ and $(\tilde{\theta}, \tilde{z})$ respectively. For $n \geq 1$ and $\theta,\tilde{\theta} \in \rset$, we define:
\begin{align}
   \thalpha{n}{\alpha} &= \thalpha{n-1}{\alpha} - \alpha \{ \funcA{Z_n} \thalpha{n-1}{\alpha} - \funcb{Z_n} \}, \quad \theta_0 = \theta \eqsp,\\
    \thalphat{n}{\alpha} &= \thalphat{n-1}{\alpha} - \alpha \{ \funcA{\tilde{Z}_n} \thalphat{n-1}{\alpha} - \funcb{\tilde{Z}_n} \},
\quad \theta_0 = \tilde{\theta} \eqsp.
\end{align}
 Then, for any $z, z' \in \Zset$, from the result in \Cref{lem:cost_function_decrease}, we get:
 \begin{align}
 \label{eq:cost_func_bound_main}
     \PEcoupling{z}{\tilde{z}}[\costw[0]((\thalpha{n}{\alpha},\State_n),(\thalphat{n}{\alpha},&\tilde{\State_n}))] \\
    &\lesssim \rate[\alpha]^n \costw[0]((z,\theta), (\tilde{z},\tilde{\theta})) \eqsp,
\end{align}
 where $\rate[\alpha] = \rme^{-\alpha a/24}$ and the expectation is taken over the coupling measure. Finally, the existence and uniqueness of the invariant measure $\Pi_{\alpha}$ follows from the contraction inequality \eqref{eq:cost_func_bound_main} in conjunction with \cite[Theorem 20.3.4]{douc:moulines:priouret:soulier:2018}. The detailed proof is provided in \Cref{appendix:proof_existence_pi_alpha}.
\end{proof}
Our next goal is to quantify the bias 
\[
\invmeasure[\alpha][\theta_0] - \thetas\eqsp.
\]
Towards this aim, we consider the perturbation-expansion framework of \cite{aguech2000perturbation}, see also \cite{durmus2025finite}. We define the product of random matrices
\begin{equation} 
\label{eq:definition-Phi} 
\textstyle
\ProdBa_{m:n}  = \prod_{i=m}^n (\Id - \alpha \funcA{Z_i} ) \eqsp, \quad m \leq n \eqsp,
\end{equation}
with the convention, $\ProdBa_{m:n}=\Id$ for $m > n$. Then we consider the decomposition of the error into the transient and fluctuation terms
\begin{equation}
\label{eq:tr_fl_decomp}
    \thalpha{n}{\alpha} - \thetas = \utheta_{n} + \vtheta_{n}\eqsp,
\end{equation}
where
\begin{align}
\label{eq:LSA_recursion_expanded}
\textstyle
&\utheta_{n} =  \ProdBa_{1:n} \{ \theta_0 - \thetas \} \eqsp,\\
&\vtheta_{n} = - \alpha \sum_{j=1}^n \ProdBa_{j+1:n} \funcnoise{Z_j}\eqsp.
\end{align}

\paragraph{Bounding the transient and fluctuation terms}
To bound the transient term, we apply the result on exponential stability of the random matrix product from \cite[Proposition 7]{durmus2025finite}. For the fluctuation term $\vtheta_{n}$ we use the perturbation expansion technique formalized in \cite{aguech2000perturbation} and later applied to obtain the high-probability bounds in \cite{durmus2025finite}. For this decomposition, we define for any $\l \geq 0$ the vectors $\{\Jnalpha{n}{l,\alpha}$, $\Hnalpha{n}{l,\alpha}\}$ which can be computed from the recursion relations
\begin{align}
\label{eq:jn0_main}
&\Jnalpha{n}{0,\alpha} =\left(\Id - \alpha \bA\right) \Jnalpha{n-1}{0,\alpha} - \alpha \funcnoise{\State_{n}} \eqsp, \\
\label{eq:hn0_main}
&\Hnalpha{n}{0,\alpha} =\left( \Id - \alpha \funcA{\State_{n}} \right) \Hnalpha{n-1}{0,\alpha} - \alpha \zmfuncA{\State_{n}} \Jnalpha{n-1}{0,\alpha} \eqsp,
\end{align}
where $\Jnalpha{0}{0,\alpha} = \Hnalpha{0}{0,\alpha} = 0$. It is easy to check that 
\[
\vtheta_{n} = \Jnalpha{n}{0,\alpha} + \Hnalpha{n}{0,\alpha}\eqsp.
\]
Moreover, the term $\Hnalpha{n}{0,\alpha}$ can be further decomposed similalry to \eqref{eq:jn0_main} - \eqref{eq:hn0_main}. Precisely, for any $L \in \nsets$ and $\ell \in \{1,\ldots,L\}$, we consider 
\begin{align}
\label{eq:jn_allexpansion_main}
\Jnalpha{n}{\ell,\alpha} =\left(\Id - \alpha \bA\right) \Jnalpha{n-1}{\ell,\alpha} -  \alpha \zmfuncA{\State_{n}} \Jnalpha{n-1}{\ell-1,\alpha}  \eqsp,
\end{align}
and 
\begin{align}
\label{eq:hn_allexpansion_main}
\Hnalpha{n}{\ell,\alpha} =\left( \Id - \alpha \funcA{\State_{n}} \right) \Hnalpha{n-1}{\ell,\alpha} - \alpha \zmfuncA{\State_{n}} \Jnalpha{n-1}{\ell,\alpha} \eqsp,
\end{align}
where we set $\Jnalpha{0}{l, \alpha} = \Hnalpha{0}{l, \alpha} = 0$. It is easy to check that, in this setting, 
\begin{equation}
    \Hnalpha{n}{l,\alpha} = \Jnalpha{n}{l+1,\alpha} + \Hnalpha{n}{l+1,\alpha} \eqsp,
\end{equation}
and 
\begin{equation}
\label{eq:fluct_decomp}
\vtheta_{n} = \sum_{\ell=0}^{L}\Jnalpha{n}{0,\alpha} + \Hnalpha{n}{L,\alpha}\eqsp.
\end{equation}
To analyze the bias $\invmeasure[\alpha][\theta_0] - \thetas$, we consider this expansion with $L=2$. That is, combining \eqref{eq:jn0_main}, \eqref{eq:hn0_main} and \eqref{eq:jn_allexpansion_main}, we obtain the decomposition which is the cornerstone of our analysis:
\begin{equation}
\label{eq:error_decomposition_LSA}
\textstyle \thalpha{n}{\alpha} - \thetas =  \utheta_{n} + \Jnalpha{n}{0,\alpha} + \Jnalpha{n}{1,\alpha} + \Jnalpha{n}{2,\alpha} + \Hnalpha{n}{2,\alpha}\eqsp.
\end{equation}
Following the arguments in \cite{durmus2021tight}, this decomposition can be used to obtain sharp bounds on the $p$-th moment of the final LSA iterate  $\thalpha{n}{\alpha}$.

\paragraph{Bias expansion for LSA}
\label{par:bias_expansion_LSA}
Similarly to \eqref{eq:lsa-main-eq}, we can not consider the process $\{\Jnalpha{k}{\ell,\alpha}\}$ separately, as it might fail to be a Markov chain. Instead, we again consider the joint process
\begin{equation}
\label{eq:mc_J1_joint_def}
Y_t = (Z_{t+1}, \Jnalpha{t}{0,\alpha}, \Jnalpha{t}{1,\alpha})
\end{equation}
with the Markov kernel $\MKQ_{J^{(1)}}$, which can be defined formally in the similar way as $\MKjoint[\alpha]$.
We need to refine our assumptions on the step-size compared to \eqref{eq:def_alpha_p_infty_Markov}. More specifically, for any $2 \leq p < \infty$, we set
\begin{align}
\label{eq:def_alpha_tmix_p}
\alpha_{p,\infty}^{(\sf b)} = \left(\alpha_{p(1+\log{d}), \infty}^{(\Markov)} \wedge {1 \over 1 + \bConst{A}} \wedge {1\over ap}\right) \taumix^{-1} \eqsp,
\end{align}
where $\alpha_{\infty}^{(\Markov)}$ is defined in \eqref{eq:def_alpha_p_infty_Markov}. For ease of notation, we set $\alpha_{\infty}^{(\sf b)} := \alpha_{2,\infty}^{(\sf b)}$. Note that the established step size suggests to take smaller step sizes in order to control higher moments. We also note that under \Cref{assum:drift}, it holds that $\alpha_{p,\infty}^{(\sf b)} < {\log\rho^{-1} \over \alpha p}$ and it simplifies the step size assumption in \Cref{cor:rosenthal_J_nonstat_main}. 
\begin{proposition}
\label{lem:prop_inv_distribution}
Assume \Cref{assum:A-b}, \Cref{assum:noise-level} and \Cref{assum:drift}. Let $\alpha \in (0,\alpha_{\infty}^{(\sf b)})$. Then the process $\{Y_t\}_{t \in \nset}$ is a Markov chain with a unique stationary distribution $\Pi_{J^{(1)}, \alpha}$ .
\end{proposition}
\begin{proof}[Proof sketch]
    We consider the Markov chain $\{Y_t, t\geq0\}$ with kernel $\MKQ_{J^{(1)}}$, where $Y_t = (Z_{t+1}, \Jnalpha{t}{0,\alpha}, \Jnalpha{t}{1,\alpha})$. Our approach involves analyzing the convergence of this Markov chain using the Wasserstein semimetric, defined in \eqref{eq:wasser_cost_def}, with a properly chosen cost function. Denoting $Y = (z, J^{(0)}, J^{(1)})$ and $\tilde{Y} = (\tilde{z}, \tilde{J}^{(0)}, \tilde{J}^{(1)})$, where $Y, \tilde{Y} \in \Zset \times \rset^d \times \rset^d$, we define the cost function as:
\begin{align}
\label{eq:const_func_triplet_def_main}
 \costw(Y, \tilde{Y}) &= \norm{J^{(0)} - \tilde{J}^{(0)}} + \norm{J^{(1)} - \tilde{J}^{(1)}} \\
 &+ (\norm{J^{(0)}} + \norm{\tilde{J}^{(0)}} \\
 &+ \norm{J^{(1)}} + \norm{\tilde{J}^{(1)}} + \sqrt{\alpha a}\supconsteps)\indiacc{z \neq \tilde{z}} \eqsp.
\end{align}
Note, the term $\sqrt{\alpha a}\supconsteps$ is introduced to account for the fluctuations of $\Jnalpha{n}{0,\alpha}$ and $\Jnalpha{n}{1,\alpha}$, whose magnitudes do not exceed the order of this term. Now, we introduce the result on the contraction of the Wasserstein semimetric for two coupled Markov chains $\{Y_t\}$ and $\{\tilde{Y}_t\}$ starting from different points. Choosing $J^{(0)}, \tilde{J}^{(0)}, J^{(1)}, \tilde{J}^{(1)} \in \rset^d$ and $z, \tilde{z} \in \Zset$, we denote $y = (z, J^{(0)}, J^{(1)})$ and $\tilde{y} = (\tilde{z}, \tilde{J}^{(0)}, \tilde{J}^{(1)})$ such that $y \neq \tilde{y}$. Then, by \Cref{lem:contr_wasser_J1} with $p = 1$, for any $n\geq 1$, we have
\begin{align}
\label{eq:wasser_dist_contr_tripl}
    \wasserdist_{\costw}(\delta_y \MKQ_{J^{(1)}}^n, \delta_{\tilde{y}}\MKQ_{J^{(1)}}^n) \lesssim \rho_{1,\alpha}^{n} \sqrt{\log{(1/\alpha a)}} \costw(y, \tilde{y}) \eqsp,
     \end{align}
where $\rho_{1,\alpha} = e^{-\alpha a / 12}$. Thus, the existence of invariant distribution $\Pi_{J^{(1)},\alpha}$ directly follows from \eqref{eq:wasser_dist_contr_tripl} and \cite[Theorem~20.3.4]{douc:moulines:priouret:soulier:2018}; for more details, see \Cref{appendix:inv_dist_J1_exist}.
\end{proof}
We denote random variables
\[(Z_{\infty+1},\Jnalpha{\infty}{0,\alpha}, \Jnalpha{\infty}{1,\alpha})
\]
with distribution $\Pi_{J^{(1)},\alpha}$. Under stationary distribution, we have $\PE[\Jnalpha{\infty}{0,\alpha}] = 0$. Consider now the component that corresponds to $\Jnalpha{\infty}{1,\alpha}$. The following proposition holds:
\begin{proposition}
\label{prop:Jnalpha_asymp_exp}
Assume \Cref{assum:A-b}, \Cref{assum:noise-level} and \Cref{assum:drift}. Then for $\alpha \in (0, \alpha_{\infty}^{(\sf b)})$, it holds that 
\begin{equation}
\lim_{n\to \infty} \PE[\Jnalpha{n}{1, \alpha}] = \PE[\Jnalpha{\infty}{1, \alpha}] = \alpha \Delta + R(\alpha)\eqsp,
\end{equation}
where $\Delta \in \rset^{d}$ is defined as
\begin{equation}
\Delta = \bA^{-1} \sum_{k=1}^\infty \PE[\zmfuncA{Z_{\infty+k}}\funcnoise{Z_{\infty}}] \eqsp,
\end{equation}
and $R(\alpha)$ is a remainder term which can be bounded as
\begin{equation}
\norm{R(\alpha)} \leq 12 \norm{\bA^{-1}}\bConst{A}^2 \taumix^2 \alpha^2 \supconsteps \eqsp.
\end{equation}
\end{proposition}

\begin{corollary}
\label{prop:bias_asymp_exp}
Under the setting of \Cref{prop:Jnalpha_asymp_exp}, we get the following expansion for the asymptotic bias
\begin{equation}
\label{eq:bias_asymp_exp}
\lim_{n \to \infty} \PE[\theta_n] = \Pi_{\alpha}(\theta_0) = \thetalim + \alpha \Delta + O(\alpha^{3/2}) \eqsp.
\end{equation}
\end{corollary}
\begin{proof}
From \Cref{prop:bound_Jn2_alpha}, we deduce that  
$\lim_{n \to \infty} \PE\big[\norm{\Jnalpha{n}{2,\alpha}}\big] \lesssim \alpha^{3/2}$  
and $\lim_{n \to \infty} \PE\big[\norm{\Hnalpha{n}{2,\alpha}}\big] \lesssim \alpha^{3/2}$.  
This implies that the term $\Jnalpha{n}{1,\alpha}$ should be the leading term in the bias decomposition.  
Together with the analysis of $\Jnalpha{n}{2,\alpha}$, this confirms that  
$\{\Jnalpha{n}{l+1,\alpha}, l \geq 0\}$ provides the proper linearization of the bias  
in powers of $\alpha$, giving rigorous justification for our decomposition approach.
For the complete proof, we refer to \Cref{appendix:bias_asymp_exp}.
\end{proof}

\begin{remark}
\label{rmk:bias_remainder}
    By sequentially analyzing the terms $\{\Jnalpha{n}{k,\alpha}, k \geq 2\}$ in the decomposition of $\theta_n$, we can obtain the bias decomposition as a power series in $\alpha$. Additionally, in \Cref{prop:bias_decomp_J2}, we show that
    \begin{align}
        \lim_{n\to\infty} \PE[\Jnalpha{n}{2,\alpha}] = \alpha^2 \Delta_2 + R_2(\alpha)\eqsp,
    \end{align}
where
    \begin{align}
        &\Delta_2 = -\sum_{k=1}^{\infty} \sum_{i=0}^{\infty} \PE[\zmfuncA{Z_{\infty+k+i+1}}\zmfuncA{Z_{\infty+i+1}}\funcnoise{Z_{\infty}}]
     \end{align}
and $\norm{R_2(\alpha)}\lesssim \alpha^{5/2}$. Furthermore, we expect that the remainder term $H_n^{(2,\alpha)} = J_n^{(3,\alpha)} + H_n^{(3,\alpha)}$ in the bias decomposition \eqref{eq:error_decomposition_LSA} is of order $\mathcal{O}(\alpha^{2})$. To establish this result, we can use a technique similar to the one used for the $p$-th moment of $\Jnalpha{n}{2,\alpha}$ in \Cref{prop:bound_Jn2_alpha}. This suggests that $\PE^{1/p}[\norm{\Jnalpha{n}{3,\alpha}}^p] \lesssim \alpha^{2}$. Therefore, we conclude that the rate $\mathcal{O}(\alpha^2)$ could be achieved in the remainder term of \eqref{eq:bias_asymp_exp}.
\end{remark}

\paragraph{Discussion} Our coefficient $\Delta$ in the linear term matches the representation derived in \cite[Theorem 2.5]{lauand2023curse}, but that work does not analyze MSE with reduced bias. To observe the next result, we define an adjoint kernel $\MKQ^*$ such that for the invariant measure $\pi$, we have $\pi \otimes \MKQ(A \times B) = \pi \otimes \MKQ^*(B \times A)$. Additionally, we define the independent kernel $\Pi$ such that for any $z \in \Zset$ and $A \in \Zsigma$, $\Pi(z, A) = \pi(A)$. Under these notations, the authors in \cite{huo2023bias} considered the bias expansion arising from the Neumann series for the operator $(\Id - \MKQ^* + \Pi)^{-1}(\MKQ^* - \Pi)$. Furthermore, adapting the proof of \Cref{prop:Jnalpha_asymp_exp}, our result can be reformulated in terms of $\MKQ^*$. This representation is less desirable because it requires reversibility of the Markov kernel $\MKQ$, as discussed in \cite{huo2023bias}.

\section{Analysis of Richardson-Romberg Procedure}
\label{sec:RR_analysis}
\label{sec:RR_interpolation}
A natural way to reduce the bias in \eqref{eq:bias_asymp_exp} is to use the Richardson-Romberg extrapolation \cite{hildebrand1987introduction}
\begin{equation}
\label{eq:RR_iter_def}
\prthalpha{n}{\alpha, \sf RR} = 2 \prthalpha{n}{\alpha} - \prthalpha{n}{2\alpha} \eqsp.
\end{equation}
After this procedure the remainder term in the bias has order $O(\alpha^{3/2})$. Before the main theorem of this section, we establish our key technical results. For that, we consider another Markov chain $\{V_t\}_{t\in \nset}$ with kernel $\MKQ_J$, where we set $V_t = (J_t, Z_{t+1})$. In fact, it is closely related to the one described in \eqref{eq:mc_J1_joint_def} and also converges geometrically fast to the unique stationary distribution as stated by \Cref{cor:cor_inv_distr_J0_main}.
\begin{corollary}
\label{cor:cor_inv_distr_J0_main}
Assume \Cref{assum:A-b}, \Cref{assum:noise-level} and \Cref{assum:drift}. Let $\alpha \in (0,\alpha_{\infty}^{(\sf b)})$. Then the process $\{V_t\}_{t \in \nset}$ is a Markov chain with a unique stationary distribution $\Pi_{J, \alpha}$ .
\end{corollary}
\begin{proof}[Proof sketch]
    We define the cost function $\costw[J]: \rset^d \times \Zset \times \rset^d \times \Zset \to \rset_{+}$, as:
    \begin{align}
     \costw[J]((J, z), (\tilde{J}, \tilde{z})) &= \norm{J - \tilde{J}} \\
     &+ (\norm{J} + \norm{\tilde{J}} + \sqrt{\alpha a} \supconsteps)\indiacc{z \neq \tilde{z}} \eqsp.
    \end{align}
    The result on the contraction of the Wasserstein semimetric with cost function $\costw[J]$ for $\{V_t\}_{t\in \nset}$ can be also obtained independently using the technique from \Cref{lem:prop_inv_distribution}. However, we derive a weaker result directly from \Cref{lem:prop_inv_distribution}, showing that $\wasserdist_{\costw[J], p}(\delta_y \MKQ_J^n, \delta_{\tilde{y}}\MKQ_J^n) \leq \wasserdist_{\costw, p}(\delta_y \MKQ_{J^{(1)}}^n, \delta_{\tilde{y}}\MKQ_{J^{(1)}}^n)$. Hence, using the similar arguments, we conclude that the Markov chain $\{V_t, t \geq 0\}$ admits invariant distribution $\Pi_{J,\alpha}$. The detailed proof can be found in \Cref{appendix:rosenthal_Jnalpha}.
\end{proof}
Note that the invariant distribution $\Pi_{J,\alpha}$ coincides with the distribution of $(\Jnalpha{\infty}{0,\alpha}, Z_{\infty + 1})$. For any $J \in \rset^d, z \in \Zset$, we define:
\begin{align}
     &\bar{\psi}(J, z) = \psi(J, z) - \PE_{\pi_J}[\psi_0],\\
     &\bar{\psi}_t = \bar{\psi}(\Jnalpha{t}{0,\alpha}, Z_{t+1}) \eqsp.
\end{align}
The cost functions $\costw[J]$ and $\costw[J^{(1)}]$ are designed such that the function $\psi(J, z) = \zmfuncA{z}J$ for $J \in \rset^d, z \in \Zset$ is Lipschitz, specifically:
\begin{equation}
 \norm{\psi(J, z) - \psi(\tilde{J}, \tilde{z})} \leq 2\bConst{A} \costw[J]((J, z), (\tilde{J}, \tilde{z})) \eqsp.
\end{equation}
This Lipschitz property is necessary for our analysis of \Cref{thm:error_RR_iter}.
The following result concerns the magnitude of $\sum_{t=n_0}^{n-1} \psi_t$, which appears in the decomposition \eqref{eq:RR_err_decompose}. It has a non-zero bias, and thus, a direct estimation leads to non-optimal behavior. However, after centering, the result in \Cref{cor:rosenthal_J_nonstat_main} suggests that it can be estimated effectively. This provides a theoretical justification for the numerical experiments presented in \Cref{sec:exps}.

\begin{proposition}
\label{cor:rosenthal_J_nonstat_main}
     Assume \Cref{assum:A-b}, \Cref{assum:noise-level}, and \Cref{assum:drift}. Then for any probability measure $\xi$ on $\Zset$, setting $J_0 = 0$, $Z_1 \sim \xi$, $2 \leq p < \infty$ and $\alpha \in (0, \alpha_{p,\infty}^{(\sf b)} \wedge )$, we get
     \begin{align}
        \label{eq:rosenthal_J_nonstat_main}
        \PE_{\xi}^{1/p}[\norm{\sum_{t=0}^{n-1} \bar{\psi}_t}^p] &\leq \Auxconst_{W,1}^{(2)} p^{3/2} (\alpha n)^{1/2}\\
        &+ \Auxconst_{W,2}^{(2)} p^3 \alpha^{-1/2} \sqrt{\log{(1/\alpha a)}} \eqsp,
     \end{align}
     where the constants $\Auxconst_{W,1}^{(2)}, \Auxconst_{W,2}^{(2)}$ are defined in the supplement paper, see \eqref{eq:const_def_rosent_nonstat}.
\end{proposition}
Now, we conclude the result on $p$-th moment for error of the RR iteration \eqref{eq:RR_iter_def}.
\begin{theorem}
    \label{thm:error_RR_iter}
    Assume \Cref{assum:A-b}, \Cref{assum:noise-level} and \Cref{assum:drift}. Fix $2 \leq p < \infty$, then for any $n \geq \taumix$, $\alpha \in (0, \alpha_{p,\infty}^{(\sf b)}/2)$ and initial probability measure $\xi$ on $(\Zset, \mcz)$, we have
        \begin{align}
        \label{eq:error_RR_iter}
            \PE_{\xi}^{1/p}[\norm{\bA(&\prthalpha{n}{\alpha, \sf RR} - \thetalim)}^p] \\
            &\leq 2\bConst{\sf{Rm}, 1} \{\trace \noisecov^{(\Markov)}\}^{1/2} p^{1/2} n^{-1/2} + R_{n,p,\alpha}^{(\sf fl)}\\
            &+ R_{n,p,\alpha}^{(\sf tr)} \norm{\theta_0 - \thetalim} \exp\{{-\alpha a n/24}\} \eqsp,
        \end{align}
    where $R_{n,p,\alpha}^{(\sf tr)}$, $R_{n,p,\alpha}^{(\sf fl)}$ are provided in \eqref{eq:def_fl_tr_RR}, and $\bConst{\sf{Rm}, 1} = 60 \rme$ is obtained from the Rosenthal inequality(see \Cref{theo:rosenthal_uge_arbitrary_init}).
\end{theorem}
The quantities $R_{n,p,\alpha}^{(\sf tr)}$ and $R_{n,p,\alpha}^{(\sf fl)}$ correspond to the fluctuation and transient terms in the error decomposition. We set them as follows
\begin{align}
\label{eq:def_fl_tr_RR}
    R_{n,p,\alpha}^{(\sf fl)} &\lesssim p n^{-3/4} \\
    &+ (p^{2}(\alpha n)^{-1/2}\sqrt{\log(1/\alpha a)} + \alpha^{1/2})p^{3/2} n^{-1/2}\\
    &+ p^{7/2} \alpha^{3/2}\log^{3/2}(1/\alpha a) \eqsp,\\
    R_{n,p,\alpha}^{(\sf tr)} &\lesssim (\alpha n)^{-1} \eqsp,
\end{align}
Here $\lesssim$ stands for the inequality up a constant which doesn't depend on $p, n$ and $\alpha$. Precise expressions for the terms $R_{n,p,\alpha}^{(\sf fl)}$ and $R_{n,p,\alpha}^{(\sf tr)}$ are given in the supplement paper, see equation \eqref{eq:constants_RR}.
\begin{proof}[Proof sketch of \Cref{thm:error_RR_iter}]
    Using \eqref{eq:lsa} and the definition of the noise term $\funcnoise{\cdot}$ in \eqref{eq:def_center_version_and_noise}, we can write the decomposition for the Richardson-Romberg iterations
\begin{align}
    \label{eq:RR_err_decompose}
    &\bA(\prthalpha{n}{\alpha, \sf RR} - \thetalim) \\
    &= \{2\alpha(n - n_0)\}^{-1}(4\thalpha{n_0}{\alpha} - \thalpha{n_0}{2\alpha} -  (4\thalpha{n}{\alpha} - \thalpha{n}{2\alpha}))\\
    &+ \{n-n_0\}^{-1} \sum_{t=n_0}^{n-1} \{ e ( \thalpha{t}{2\alpha}, \State_{t+1}  ) - 2 e ( \thalpha{t}{\alpha}, \State_{t+1} ) \} \eqsp.
\end{align}
The leading term can be bounded using the result for the $p$-th moment of the last iteration in \Cref{lem:drift_condition_LSA_Markov}. The last term  can be further decomposed using
\begin{equation}
\label{eq:decompo_e_theta_z}
\textstyle \sum_{t=n_0}^{n-1} e\left(\thalpha{t}{\alpha}, Z_{t+1} \right)= \Etralpha{\alpha}_{n} + \Eflalpha{\alpha}_{n}\eqsp,
\end{equation}
where we have set
\begin{align}
\label{eq:def:Etr}
\Etralpha{\alpha}_{n}  &= \textstyle \sum_{t=n_0}^{n-1} \zmfuncA{\State_{t+1}} \ProdBa_{1:t} \{ \theta_0 - \thetas \} \eqsp,\\
\label{eq:def:Efl}
\Eflalpha{\alpha}_{n} &= \textstyle \sum_{t=n_0}^{n-1} \funcnoise{\State_{t+1}}\\
&+ \sum_{\ell=0}^{2}\sum_{t=n_0}^{n-1} \zmfuncA{\State_{t+1}} \Jnalpha{t}{\ell, \alpha}
 + \sum_{t=n_0}^{n-1} \zmfuncA{\State_{t+1}} \Hnalpha{t}{2, \alpha} \eqsp.
\end{align}
 The first term in $\Eflalpha{\alpha}_{n}$ is linear statistics of Markov chain $\{Z_k, k\in \nset\}$. Therefore, we can bound it using the version of Rosenthal inequality for Markov chains from \cite{moulines23_rosenthal}. For the term involving $\Jnalpha{t}{0,\alpha}$, we employ the expansion from \eqref{eq:exp_Jn0}, yielding a centered random variable component plus a bias term. This decomposition allows direct application of the inequality in \Cref{cor:rosenthal_J_nonstat_main} to the sum of centered random variables, which yields the bound $\mathcal{O}((\alpha /n)^{1/2} + \alpha^{-1/2}n^{-1})$. Combining this with the result from \Cref{prop:Jnalpha_asymp_exp}, we conclude that the remaining term is $\mathcal{O}(\alpha^2)$. 
 
 Then, we apply \Cref{prop:bound_Jn2_alpha} to control the statistic $\sum_{t=n_0}^{n-1}\zmfuncA{Z_{t+1}}\Jnalpha{t}{1,\alpha}$, which we express in terms of $\Jnalpha{n}{2,\alpha}$ via the expansion in \eqref{eq:jn_allexpansion_main}. For the analogous term involving $\Hnalpha{n}{2,\alpha}$, we establish the required bound in \Cref{prop:bound_Hn2}.
 The detailed proof can be found in \Cref{appendix:proof_error_RR_iter}.
\end{proof}

\begin{figure*}[ht!]
    \centering
    \begin{subfigure}[b]{0.24\textwidth}
        \centering
        \includegraphics[width=\textwidth]{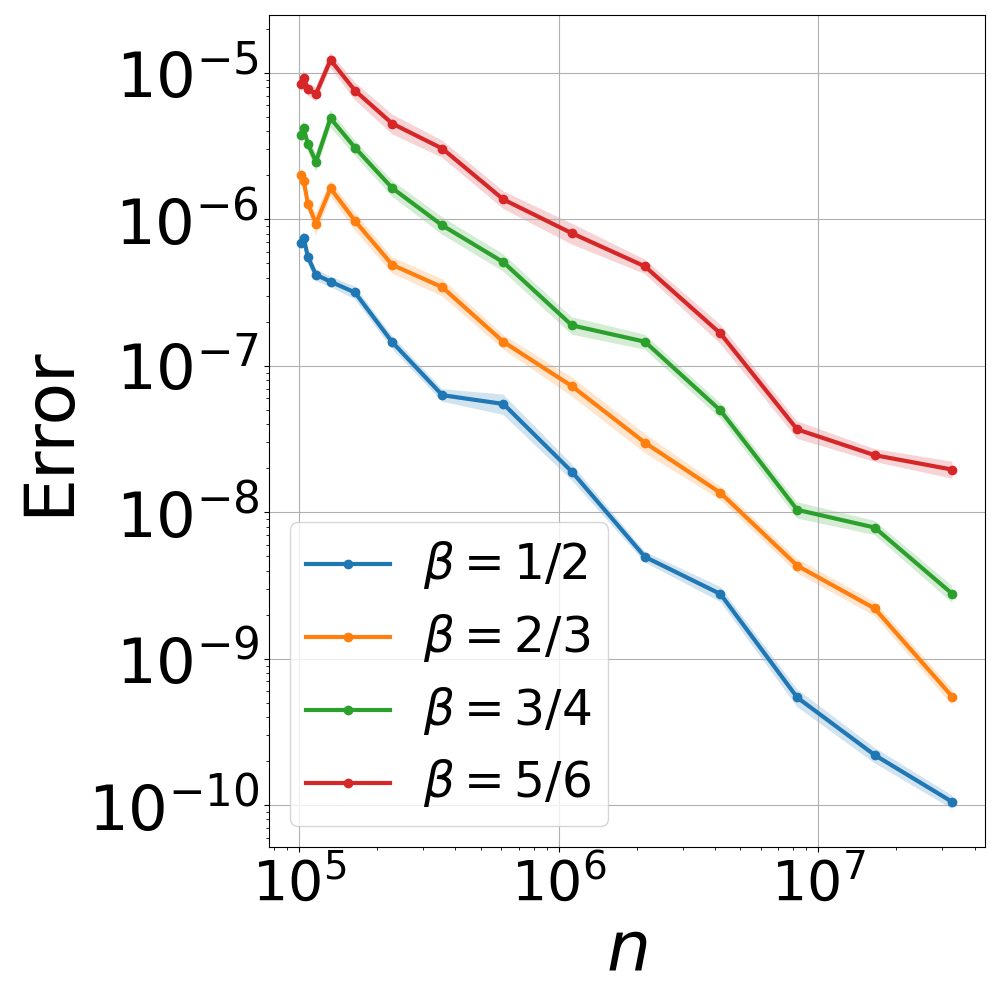}
        \caption{}
        \label{subfig:mse_remainder_pr_rr}
    \end{subfigure}
    \begin{subfigure}[b]{0.24\textwidth}
        \centering
        \includegraphics[width=\textwidth]{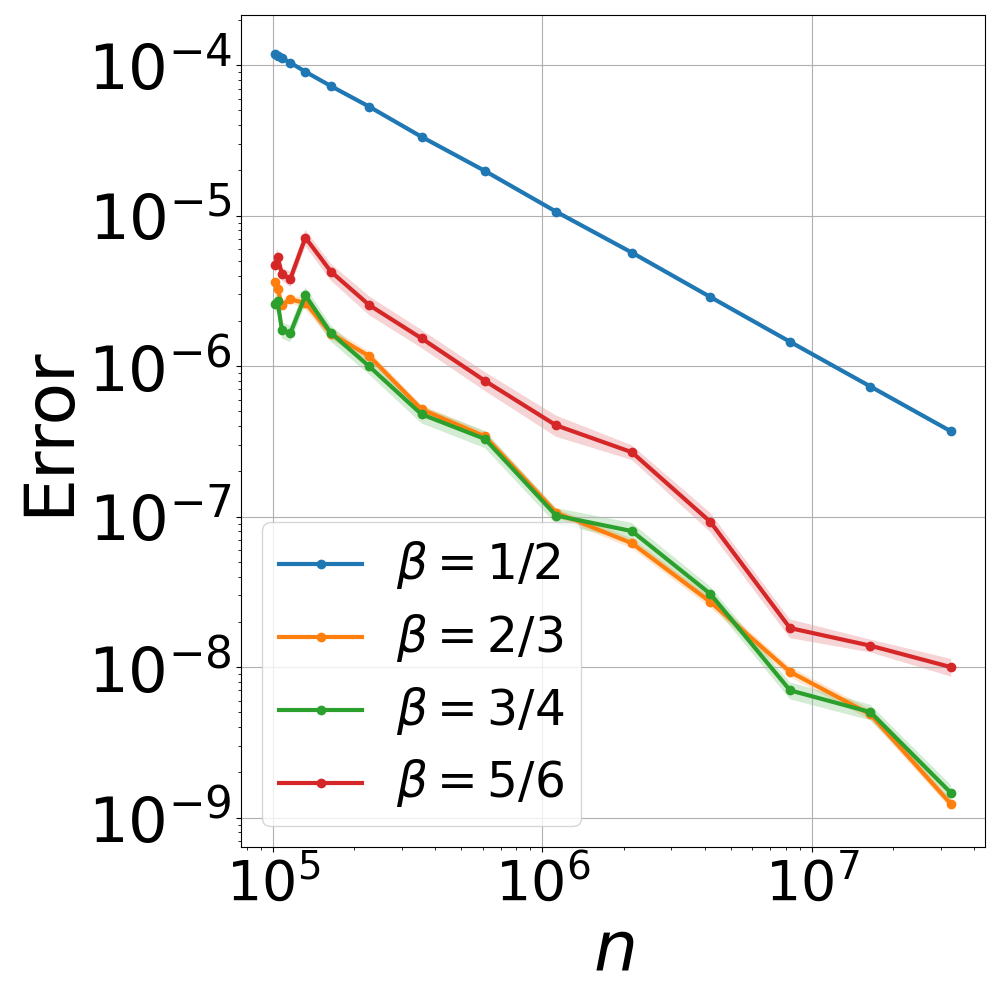}
        \caption{}
        \label{subfig:mse_remainder_pr}
    \end{subfigure}
    \begin{subfigure}[b]{0.24\textwidth}
        \centering
        \includegraphics[width=\textwidth]{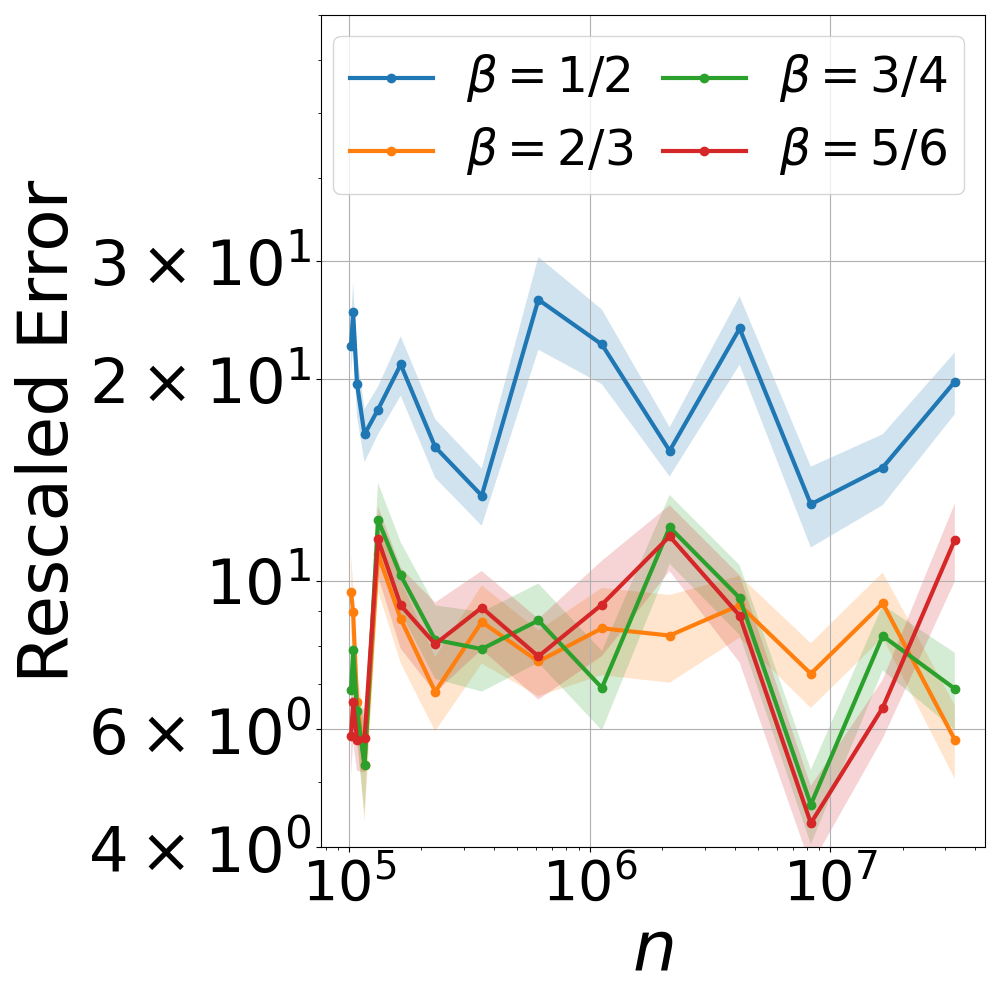}
        \caption{}
        \label{subfig:mse_remainder_rescaled}
    \end{subfigure}
    \begin{subfigure}[b]{0.24\textwidth}
        \centering
        \includegraphics[width=\textwidth]{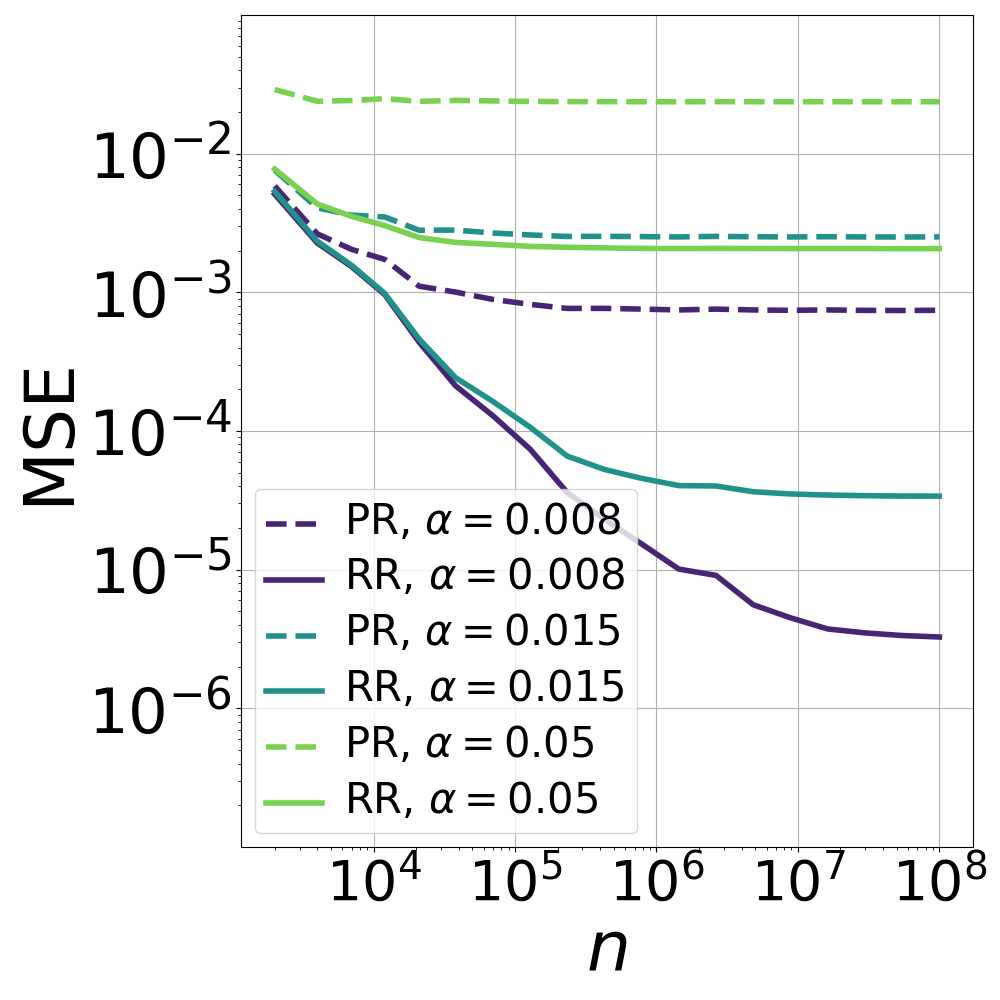}
        \caption{}
        \label{fig:mse_trajs}
    \end{subfigure}
    \caption{In subfigures (a), (b) and (c), we consider the step size $\alpha = n^{-\beta}$. Subfigure (a): error for RR iterations \eqref{eq:RR_iter_def}. Subfigure (b): error for PR iterations \eqref{eq:lsa_alpha}. Subfigure (c): error for RR, multiplied by a factor $n^{2-\beta}$, which corresponds to the leading term of \eqref{eq:def_fl_tr_RR} after substituting $\alpha$. Subfigure (d): MSE of Polyak-Ruppert and Richardson-Romberg iterations for different step sizes $\alpha$.}
    \label{fig:main_exps}
\end{figure*}

Note that the bound in \Cref{cor:rosenthal_J_nonstat_main} motivates the choice $\alpha = \mathcal{O}(n^{-1/2})$, aligning with the rate observed in the i.i.d case. Optimization over $\alpha$ gives us the following high-probability bound. Also, the term with $p^{7/2}$ could be slightly improved to $p^{5/2}$ through a more accurate analysis of the \Cref{lem:contr_wasser}. Additionally, following the discussion in \Cref{rmk:bias_remainder}, we expect that the remainder term $\mathcal{O}(\alpha^{3/2})$ in \Cref{thm:error_RR_iter} could be improved to $\mathcal{O}(\alpha^2)$, though this would require a technically complicated analysis of $\Jnalpha{n}{3,\alpha}$. Using Markov's inequality, we derive the following high-probability bound.
\begin{corollary}
\label{corr:error_RR_hp_iter}
    Assume \Cref{assum:A-b}, \Cref{assum:noise-level} and \Cref{assum:drift}. For $2 \leq p < \infty$ and any $n \geq \taumix$, we consider the step size
    \begin{equation}
    \label{eq:opt_alpha_RR}
        \alpha(n, d, \taumix, p) = {\alpha_{p,\infty}^{(\sf b)} \over 2} n^{-1/2} \eqsp.
    \end{equation}
    Substituting \eqref{eq:opt_alpha_RR} into \eqref{eq:error_RR_iter} with $p = \ln{(3\rme/\delta)}$, it holds with probability at least $1 - \delta$, that
    \begin{align}
        \norm{\bA &(\prthalpha{n}{\alpha, \sf{RR}} - \thetalim)} \lesssim \sqrt{\trace \noisecov^{(\Markov)} \log{(1/\delta)}} n^{-1/2} \\
        &+ (1 + \log^{3/2}{(n)}\log^{3}{(1/\delta)}) \log{(1/\delta)}n^{-3/4} \\
        &+ n^{-1/2}\log{(1/\delta)}\norm{\theta_0 - \thetalim} \exp\left\{-c_{\text{tr}, p} n^{1/2} \right\} \eqsp,
    \end{align}
    where $c_{\text{tr}, p}$ is defined in .
\end{corollary}
\paragraph{Discussion}
Our analysis establishes high-order moment bounds and, as a consequence, high-probability bounds for RR iterations in Markovian LSA. Moreover, the leading term in \eqref{eq:error_RR_iter} scales with $\{\trace \noisecov^{(\Markov)}\}^{1/2}$, which is known to be locally asymptotically minimax optimal for the Polyak-Ruppert iterates \cite{mou2022optimal} and aligns with the CLT covariance matrix $\Sigma_{\infty}$(see \eqref{eq:sigma_inf}). In \cite{dieuleveut2020bridging}, the authors study the bias and MSE for SGD with i.i.d noise, and propose the Richardson-Romberg extrapolation to reduce this bias. However, they only consider MSE bounds and do not obtain the proper factor for the leading term. In the Markovian LSA literature, the authors similarly consider only MSE and do not explicitly emphasize the leading term \cite{huo2024collusion, huo2023bias, huo2023effectivenessconstantstepsizemarkovian, zhang2024constant}. The closest result, \cite{sheshukova2024nonasymptotic}, shows high-order bounds with the leading term properly aligned with the optimal covariance, but in this work, the authors consider general SA with i.i.d. noise, the analysis of which differs significantly from our case.

\section{Experiments}
\label{sec:exps}

In this section, we aim to demonstrate the effect of reduced bias achieved through Richardson-Romberg extrapolation and to validate the accuracy of the bound obtained in \Cref{thm:error_RR_iter}. For this purpose, we adopt an example introduced in \cite{lauand2024revisiting}. More precisely, we consider the Markovian noise $\{Z_k, k \geq 1\}$ on the space $\Zset = \{0, 1\}$ with transition matrix 
$P=\begin{pmatrix}
a & 1 - a \\
1 - a & a
\end{pmatrix}$ and $a \in (0, 1)$. For any $z \in \{0,1\}$, we consider the noisy observations
\begin{align}
    &\funcA{z} = z \cdot A^{(1)} + (1 - z) \cdot A^{(0)},\\
    &\funcb{z} = z \cdot b^{(1)} + (1 - z)\cdot b^{(0)} \eqsp,
\end{align}
where we set
\begin{align}
    &A^{(0)} = -2\begin{pmatrix}
        -2 & 0\\
        1 & -2
    \end{pmatrix}, \quad b^{(0)} = \begin{pmatrix}
        0\\
        0
    \end{pmatrix},\\
    &A^{(1)} = -2\begin{pmatrix}
        1 & 0\\
        -1 & 1
    \end{pmatrix}, \quad b^{(1)} = 2\begin{pmatrix}
        1\\
        1
    \end{pmatrix}\eqsp. 
\end{align}
Hence, we have $\bA = \Id$ and $\barb = (1/2)b^{(1)}$. In the following experiments, we set $a = 0.3$ and ran \( N_{\sf{traj}} = 400 \) trajectories from $\theta_0 = \thetalim$ following \eqref{eq:lsa_alpha}.

\Cref{fig:mse_trajs} illustrates the significant reduction in bias achieved by the Richardson-Romberg scheme, estimating $\PE[\norm{\prthalpha{n}{\alpha} - \thetalim}^2]$ and $\PE[\norm{\prthalpha{n}{\alpha, \sf{RR}} - \thetalim}^2]$. These results justify that, after a few iterations, the error of the RR procedure starts to decrease faster than for PR averaging. Additionally, in \Cref{fig:main_exps}, we show that the resulting dependence on $\alpha$ and $n$ in the bounds \eqref{eq:def_fl_tr_RR} is tight.
To achieve this, for different sample size $n$ we select different step sizes of the form $\alpha = n^{-\beta}$ for $\beta \in [1/2, 1)$, substitute these into \eqref{eq:def_fl_tr_RR}, and compute the scaling of the term $R_{n,p,\alpha}^{(\sf fl)}$ w.r.t. $n$. For $\beta \geq 1/2$, with mentioned choice of $\alpha$, $R_{n,p,\alpha}^{(\sf fl)}$ scales as $n^{\beta-2}$. 
\par 
To verify numerically this rate, we consider the following procedure. We approximate the terms $\PE[\norm{\prthalpha{n}{\alpha} - \thetalim + (1/n)\sum_{k=1}^n\funcnoise{Z_k}}^2]$ for PR averaging, and 
\begin{equation}
\label{eq:RR-remainder}
\Delta_n^{(RR)} = \PE[\norm{\prthalpha{n}{\alpha, \sf{RR}} - \thetalim + (1/n)\sum_{k=1}^n\funcnoise{Z_k}}^2]
\end{equation}
for Richardson-Romberg iterations. The moments of the latter term should scale with $n^{\beta-2}$. We verify this effect numerically setting $\alpha = n^{-\beta}$ for $\beta \in \{1/2, 2/3, 3/4, 5/6\}$ and providing the plots for $\Delta_n^{(RR)}$ and $n^{2-\beta} \Delta_n^{(RR)}$ in \Cref{subfig:mse_remainder_pr_rr} and \Cref{subfig:mse_remainder_rescaled}, respectively. Additionally, in \Cref{subfig:mse_remainder_pr_rr} and \Cref{subfig:mse_remainder_pr}, we compare the error for different choices of step $\alpha$. We can see that the step $\alpha = n^{-1/2}$ gives the smallest error for Richardson-Romberg iterations, while for Polyak-Ruppert averaging this choice of step introduces a large bias in the error.

\section{Conclusion}
We studied the high-order error bounds for Richardson-Romberg extrapolation in the setting of Markovian linear stochastic approximation. By applying the novel technique for bias characterization, we were able to obtain the leading term which aligns with the asymptotically optimal covariance matrix $\Sigma_{\infty}$. For further work, we consider the generalization of the obtained results to the setting of non-linear Markovian SA and SA with state-dependent noise.

\section*{Acknowledgments}
The authors are grateful to Eric Moulines for valuable discussions and feedback. The work was supported by the grant for research centers in the field of AI provided by the Ministry of Economic Development of the Russian Federation in accordance with the agreement 000000C313925P4E0002 and the agreement with HSE University № 139-15-2025-009. This research was supported in part through computational resources of HPC facilities at HSE University \cite{kostenetskiy2021hpc}.

\bibliography{references}


\onecolumn
\appendix

\section{Notations and Constants}
\label{appendix:constants}
Denote $\nsets = \nset \setminus \{0\}$ and $\nsetm = \zset \setminus \nsets$.
Let $d \in \nsets$ and $Q$ be a symmetric positive definite $d \times d$ matrix. For $x \in \rset^d$, we denote $\norm{x}[Q]= \{x^\top Q x\}^{\half}$. For brevity, we set $\norm{x}= \norm{x}[\Id_d]$. We denote $\normop{A}[Q]= \max_{\norm{x}[Q]=1} \norm{Ax}[Q]$, and the subscriptless norm $\normop{A} = \normop{A}[\Id]$ is the standard spectral norm.
For a function $g: \Zset \to \rset^d$, we denote $\supnorm{g} = \sup_{z \in \Zset}\norm{g(z)}$. For a random variable $\xi$, we denote its distribution by $\mathcal{L}(\xi)$.

We denote $\sphere^{d-1} = \{x \in \rset^{d} \, : \, \norm{x} = 1\}$. Let $A_{1},\ldots,A_N$ be $d$-dimensional matrices. We denote $\prod_{\ell=i}^j A_\ell = A_i \ldots A_j$ if $i\leq j$ and by convention $\prod_{\ell=i}^j A_\ell = \Id$ if $i > j$.

The readers can refer to the \Cref{tab:const_def_not} on the variables, constants and notations that are used across the paper for references.

\begin{table*}[htbp]
\centering
\caption{Constants, definitions, notations}
\label{tab:const_def_not} 
\begin{tabular*}{\textwidth}{@{\extracolsep{\fill}} l l l @{}}
\hline
\bfseries Variable & \bfseries Description & \bfseries Reference \\
\hline
$Q$ & Solution of Lyapunov equation for $\bA$ & \Cref{fact:Hurwitzstability} \\
$\qcond$ & $\lambda_{\sf min}^{-1}(Q) \lambda_{\sf max}(Q)$ & \Cref{fact:Hurwitzstability} \\
$a$ & Real part of minimum eigenvalue of $\bA$ & \Cref{fact:Hurwitzstability} \\
$\ProdBa_{m:n}$ & Product of random matrices with step size $\alpha$ & \eqref{eq:definition-Phi} \\
$\funnoisew(\State_n)$ & Noise in LSA procedure & \eqref{eq:def_center_version_and_noise} \\
$\utheta_{n}, \vtheta_{n}$ & Transient and fluctuation terms of LSA error & \eqref{eq:LSA_recursion_expanded} \\
$\alpha_{p,\infty}$ (resp. $\alpha^{(M)}_{p,\infty}$) & \multirow{2}{*}{\shortstack[l]{Stability threshold for $\ProdBa_{m:n}$ to have bounded \\ $p$-th moment under \Cref{assum:drift}}} & \eqref{eq:def_alpha_p_infty_Markov} \\
& & \\
$\alpha_{q, \infty}^{(\sf b)}$ & Threshold for the existence of invariant distribution $\Pi_{J^{(1)}}$  & \eqref{eq:def_alpha_tmix_p}\\
$\Jnalpha{n}{0}$ & Dominant term in $\vtheta_n$ & \eqref{eq:jn0_main} \\
$\Hnalpha{n}{0}$ & Residual term $\vtheta_n - \Jnalpha{n}{0}$ & \eqref{eq:jn0_main} \\
$\Jnalpha{n}{1},\Hnalpha{n}{1}$,$\Jnalpha{n}{2},\Hnalpha{n}{2}$ & Elements of the decomposition \eqref{eq:error_decomposition_LSA} & \eqref{eq:jn_allexpansion_main}-\eqref{eq:hn_allexpansion_main} \\
$\noisecov^{(\Markov)}$ & Noise covariance $\PE[\eps_1 \eps_1^{\transpose}]$ & \Cref{assum:noise-level} \\
$\costw[0]$ & Cost function associated with the vector $(\theta, z)$ & \eqref{eq:cost_function_wass_w_q}\\
$\costw[J]$ & Cost function associated with the vector $(\Jnalpha{}{0,\alpha}, z)$ & \eqref{eq:const_func_def}\\
$\costw$ & Cost function associated with the vector $(z, \Jnalpha{}{0,\alpha}, \Jnalpha{}{1,\alpha})$ & \eqref{eq:const_func_triplet_def_main}\\
$\costw[J^{(2)}]$ & Cost function associated with the vector $(z, \Jnalpha{}{0,\alpha}, \Jnalpha{}{1,\alpha}, \Jnalpha{}{2,\alpha})$ & \eqref{eq:cost_func_def_J2}\\
$\Pi_{\alpha}$ & Invariant distribution of $\{(\thalpha{t}{\alpha}, Z_{t+1}), t\geq 0\}$ & \Cref{theo:existence_pi_alpha}\\
$\Pi_{J,\alpha}$ & Invariant distribution of $\{(\Jnalpha{t}{0,\alpha}, Z_{t+1}), t \geq 0\}$ & \Cref{cor:cor_inv_distr_J0_main}\\
$\Pi_{J^{(1)},\alpha}$ & Invariant distribution of $\{(Z_{t+1}, \Jnalpha{t}{0,\alpha}, \Jnalpha{t}{1,\alpha}), t \geq 0\}$ & \Cref{lem:prop_inv_distribution}\\
$\bConst{\sf{Rm}, 1} = 60 \rme$ & Constant in martingale Rosenthal's inequality & \cite[Theorem~4.1]{pinelis_1994} \\
$\bConst{\sf{Rm}, 2} = 60$ & Constant in martingale Rosenthal's inequality & \cite[Theorem~4.1]{pinelis_1994} \\
$\bConst{\sf{Ros}, 1} = \frac{16\sqrt{19}}{3\sqrt{3}} \bConst{\sf{Rm}, 1}^{5/2}$, \\
$\bConst{\sf{Ros}, 2} = 64 (\bConst{\sf{Rm}, 1}^2\bConst{\sf{Rm}, 2}^{1/2}+\bConst{\sf{Rm}, 2})$& Constants in Rosenthal's inequallity under \Cref{assum:drift} & \Cref{theo:rosenthal_uge_arbitrary_init} \\
$\{\mcf_t\}_{t \in\nset}$ & filtration $\mcf_t= \sigma(\State_s \, : \, 1 \leq s \leq t)$ with $\mcf_0 = \{\emptyset, \msz\}$ & \\
$\PE^{\mcf_t}$ & the conditional expectation with respect to $\mcf_t$ & \\
\hline
\end{tabular*}
\end{table*}

\section{Bias decomposition}
\label{appendix:bias_bounds}
We define the constants
\begin{align}
   \label{eq:def_qcond_b_Q}
&  \qcond = \lambda_{\sf max}( Q )/\lambda_{\sf min}( Q )  \eqsp, \quad  b_{Q} = 2 \sqrt{\qcond} \bConst{A} \eqsp. \quad
\end{align}
Under \Cref{assum:A-b}, we define the quantity
\begin{align}
\label{eq:alpha_infty_makov}
\alpha_{\infty}^{(\Markov)} &= \left[\alpha_{\infty} \wedge \qcond^{-1/2} \bConst{A}^{-1}\, \wedge\, a/(6\rme \qcond \bConst{A})\right] \times \lceil{8 \qcond^{1/2}\bConst{A} / a\rceil}^{-1}\eqsp, \\
  \label{eq:def_C_Gamma}
\bConst{\Gamma} &= 4(\qcond^{1/2}\bConst{A} + a/6)^{2} \times \lceil 8  \qcond^{1/2}\bConst{A} /a \rceil \eqsp,
\end{align}
where $ \alpha_{\infty}$, $a,\qcond$ are defined in \eqref{eq: kappa_def} and \eqref{eq:def_qcond_b_Q}, respectively. Now we use $\alpha_{\infty}^{(\Markov)}$ and $\bConst{\Gamma}$ to define, for $q \geq 2$,
\begin{equation}
\label{eq:def_alpha_p_infty_Markov}
\alpha_{q,\infty}^{(\Markov)} = \alpha_{\infty}^{(\Markov)} \wedge \smallAconst^{(\Markov)}/q \eqsp, \quad \smallAconst^{(\Markov)} = a/\{12 \bConst{\Gamma}\}\eqsp.
\end{equation}
The upper bounds \eqref{eq:alpha_infty_makov} and \eqref{eq:def_alpha_p_infty_Markov} on the step size are required for the result on product of random matrices under Markov conditions \Cref{assum:drift}, which can be found in \cite{durmus2025finite}. We formulate this result in the \Cref{appendix:technical}.

\subsection{Proof of \Cref{theo:existence_pi_alpha}}
\label{appendix:proof_existence_pi_alpha}

\par 
We preface the proof by some definitions and properties of coupling. We follow 
Let $(\Xset,\Xsigma)$ be a measurable space. In all this section, $\QQ$ and $\QQ'$ denote two
probability measures on the canonical space  $(\Xset^\nset,\Xsigma^{\otimes \nset})$.
Fix $x^*\in \Xset$. For any $\Xset$-valued stochastic process $X=\sequence{X}[n][\nset]$ and any
$\bar \nset$-valued random variable $T$, define the $\Xset$-valued stochastic process $\shift_T X$
by $\shift_T X=\sequencen{X_{T+k}}[k \in \nset]$ on $\{T<\infty\}$ and $\shift_T X=(x^*,x^*,
x^*,\ldots)$ on $\{T=\infty\}$.
For any measure $\QQ$ on $(\Xset^\nset,\Xsigma^{\otimes \nset})$ and any
$\sigma$-field $\mcg \subset \Xsigma^{\otimes \nset}$, we denote by $\restric{\mu}{\mcg}$ the
restriction of the measure $\mu$ to $\mcg$. Moreover, for all $n \in \nset$, define the
$\sigma$-field $\mcg_n=\setens{\shift_n^{-1}(A)}{A \in \Xsigma^{\otimes \nset}}$.
We say that $(\Omega,\mcf,\PP,X,X',T)$ is an \emph{exact coupling} of $(\QQ,\QQ')$ (see \cite[Definition~19.3.3]{douc:moulines:priouret:soulier:2018}), if
  \begin{itemize}
  \item for all $A \in \Xsigma^{\otimes \nset}$, $\PP(X \in A)=\QQ(A)$ and $\PP(X' \in A)=\QQ'(A)$,
  \item $\shift_T X=\shift_{T} X'\eqsp, \eqspp \as[\PP]$
  \end{itemize}
The integer-valued random variable $T$ is a coupling time. An exact  coupling \\
$(\Omega,\mcf,\PP,X,X',T)$ of $(\QQ,\QQ')$  is \emph{maximal} (see \cite[Definition~19.3.5]{douc:moulines:priouret:soulier:2018}) if for all $n\in\nset$,
\[
\tvnorm{\restric{\QQ}{\mcg_n}-\restric{\QQ'}{\mcg_n}}=2 \PP(T >n)\eqsp.
\]
Assume that $(\Xset,\Xsigma)$ is a complete separable metric space and let $\QQ$ and $\QQ'$ denote two probability measures on $(\Xset^\nset,\Xsigma^{\otimes \nset})$.  Then, there exists a maximal exact coupling of $(\QQ,\QQ')$.
\par
We now turn to the special case of Markov chains. Let $\MK[]$ be a Markov kernel on $(\Xset,\Xsigma)$. Denote by $\sequence{X}[n][\nset]$ the coordinate process and define as before $\mcg_n=\setens{\shift_n^{-1}(A)}{A \in \Xsigma^{\otimes \nset}}$.
By \cite[Lemma~19.3.6]{douc:moulines:priouret:soulier:2018}, for any probabilities $\mu,\nu$ on $(\Xset,\Xsigma)$, we have
\begin{equation}
\label{eq:maximal_coupling_def}
\tvnorm{\restric{\PP_\mu}{\mcg_n}-\restric{\PP_\nu}{\mcg_n}}=\tvnorm{\mu \MK[]^n-\nu \MK[]^n} \eqsp.
\end{equation}
Moreover, if $(\Xset,\Xsigma)$ is Polish, then, there exists a maximal and exact coupling of $(\PP_\mu,\PP_\nu)$; see \cite[Theorem~19.3.9]{douc:moulines:priouret:soulier:2018}.

We apply this construction for the Markov kernel $\MKQ$ defined on the complete separable metric space $(\Zset,\metricz)$. For any two probabilities $\xi,\xi'$ on $(\Zset,\Zsigma)$, there exists a maximal exact coupling $(\Omega,\mcf,\PPcoupling{\xi}{\xi'},Z,Z',T)$ of $\PP^{\MKQ}_{\xi}$ and $\PP^{\MKQ}_{\xi'}$, that is, 
\begin{equation}
\label{eq:coupling_time_def_markov}
\tvnorm{\xi \MKQ^n- \xi'\MKQ^n} = 2 \PP(T >n)\eqsp.
\end{equation}
We write $\PEcoupling{\xi}{\xi'}$ for the expectation \wrt\ $\PPcoupling{\xi}{\xi'}$.

Also, we note that from \eqref{eq:drift-condition} it immediately follows that
\begin{equation}
\label{eq:crr_koef_sum_tau_mix}
\textstyle \sum_{k=0}^{\infty}\dobru{\MKQ^k} \leq (4/3)\taumix\eqsp.
\end{equation}

For $(\theta,z), (\theta',z') \in \rset^{d} \times \Zset$, define the cost function
\begin{equation}
\label{eq:cost_function_wass_w_q}
\costw[0]((\theta,z),(\theta',z')) = (\|\theta - \theta'\| + \indiacc{z \neq z'})\bigl(1 + \norm{\theta - \thetas} + \norm{\theta' - \thetas}\bigr)\eqsp,
\end{equation}
which is symmetric, lower semi-continuous and distance-like(see \cite[Chapter 20.1]{douc:moulines:priouret:soulier:2018}). Note that it can be lower bounded by the distance function
\begin{equation}
    d_0((\theta, z), (\theta', z')) = \norm{\theta - \theta'} + \indiacc{z \neq z'}
\end{equation}

Now, we consider two noise sequences $\{Z_n, n \in \nset\}$ and $\{\tilde{Z}_n, n \in \nset\}$ with coupling time $T$. For $n \geq 1$ and $\theta,\tilde{\theta} \in \rset$, we define
\begin{align}
\label{eq:coupling_constr_theta}
    \thalpha{n}{\alpha} &=  \thalpha{n-1}{\alpha} - \alpha \{ \funcA{Z_n} \thalpha{n-1}{\alpha} - \funcb{Z_n} \}, \quad \theta_0 = \theta \eqsp,\\
    \thalphat{n}{\alpha} &= \thalphat{n-1}{\alpha} - \alpha \{ \funcA{\tilde{Z}_n} \thalphat{n-1}{\alpha} - \funcb{\tilde{Z}_n} \}, 
\quad \theta_0 = \tilde{\theta} \eqsp.
\end{align}
\begin{proposition}
\label{lem:cost_function_decrease}
Assume \Cref{assum:A-b}, \Cref{assum:noise-level}, and \Cref{assum:drift}. Let $q \geq 8$. Then, for any $\alpha \in (0;(\alpha^{(M)}_{q,\infty} \wedge a^{-1})\taumix^{-1})$ with $\alpha^{(M)}_{q,\infty}$ defined in \eqref{eq:def_alpha_p_infty_Markov}, starting points $(z,\theta), (\tilde{z},\tilde{\theta}) \in \Zset \times \rset$ such that $(z,\theta) \neq (\tilde{z}, \tilde{\theta})$. Then for any $n \in \nset$, we get
\begin{equation}
\label{eq:cost_func_bound}
\PEcoupling{z}{\tilde{z}}[\costw[0]((\thalpha{n}{\alpha},\State_n),(\thalphat{n}{\alpha},\tilde{\State_n}))] \leq \ConstD_{\theta} d^{2/q} \rate[\alpha]^n \costw[0]((z,\theta), (\tilde{z},\tilde{\theta})) \eqsp,
\end{equation}
where
\begin{equation}
\label{eq:const_different_z_def_Markov}
\begin{split}
\ConstD_{\theta} &= \Auxconst_{\theta,6} \left(1 + 2\qcond^{1/2}e^2 d^{1/q} + 4 \ConstD_2 d^{1/q}\sqrt{\alpha a \taumix}\supconsteps\right), \\
\rate[\alpha] &= \rme^{-\alpha a/24}\eqsp,
\end{split}
\end{equation}
and $\Auxconst_{\theta,6}$ is defined in \eqref{eq:aux_const_6_Markov}.
\end{proposition}

\begin{proof}
Applying H\"older's and then Minkowski's inequalities, we get
\begin{align}
\label{eq:holder_c_function}
\PEcoupling{z}{\tilde{z}} [\costw[0]((\thalpha{n}{\alpha},\State_n),(\thalphat{n}{\alpha},\tilde{\State_n}))]
\leq \{\PEcoupling{z}{\tilde{z}}[ (\|\thalpha{n}{\alpha} - \thalphat{n}{\alpha}\| + \indiacc{\State_n \neq \tilde{\State_n}})^2]\}^{1/2} \\ \times (1+\{\PE_z[\norm{\thalpha{n}{\alpha} - \thetas}^{2}]\}^{1/2} + \{\PE_{\tilde{z}}[\norm{\thalphat{n}{\alpha} - \thetas}^{2}]\}^{1/2})\eqsp.
\end{align}
We bound the first term on the right-hand side of \eqref{eq:holder_c_function}.
Using \eqref{eq:LSA_recursion_expanded}, definition of the coupling time \eqref{eq:coupling_time_def_markov}, and $\shift_T Z= \shift_T \tilde{Z}$, we obtain
\begin{align}
\label{eq:error-decomposition}
\thalpha{n}{\alpha} - \thalphat{n}{\alpha} &=
\prod_{i=1}^n \{ \Id - \alpha \funcA{Z_i} \} (\theta - \thetas) - \prod_{i=1}^n \{ \Id - \alpha \funcA{\tilde{Z}_i} \} (\tilde{\theta} - \thetas) \\
&+ \alpha \sum_{i=1}^{n \wedge T} \prod_{i=j+1}^{n} (\Id - \alpha \funcA{Z_i}) \funcb{Z_j} + \alpha \sum_{i=1}^{n \wedge T} \prod_{i=j+1}^n (\Id - \alpha \funcA{\tilde{Z}_i}) \funcb{\tilde{Z}_j}\eqsp,
\end{align}
or, equivalently,
\begin{equation}
\label{eq:error-decomposition-upd}
\thalpha{n}{\alpha} - \thalphat{n}{\alpha}
= \prod_{i=n\wedge T+1}^n \{ \Id - \alpha \funcA{Z_i} \} (\thalpha{n \wedge T}{\alpha} - \thetas) - \prod_{i=n\wedge T+1}^n \{ \Id - \alpha \funcA{\tilde{Z}_i} \} (\thalphat{n \wedge T}{\alpha} - \thetas) \eqsp.
\end{equation}
Now we bound the two terms in the right-hand side of \eqref{eq:error-decomposition} separately. Using H\"older's inequality, we get
\begin{align}
\PEcoupling{z}{\tilde{z}}\bigl[ \norm{\prod\nolimits_{i=n\wedge T+1}^n \{ \Id - \alpha \funcA{Z_i} \}}^2 \norm{\thalpha{n \wedge T}{\alpha} - \thetas}^2] \leq \PE_z^{1/2}\bigl[ \norm{\thalpha{n}{\alpha} - \thetas}^4 \bigr] \PPcoupling{z}{\tilde{z}}^{1/2}( T \geq n) \\
+ \sum_{k=1}^{n-1} \PE^{1/4}_\xi \bigl[ \norm{\prod\nolimits_{i=k+1}^n \{ \Id - \alpha \funcA{Z_i} \}}^8\bigr] \PE_\xi^{1/4}\bigl[ \norm{\thalpha{k}{\alpha} - \thetas}^8 \bigr] \PPcoupling{\xi}{\tilde{\xi}}^{1/2}(T=k) =: T_1 + T_2 \eqsp.
\end{align}
We begin with estimating the term $T_2$. By definition of the maximal coupling \eqref{eq:maximal_coupling_def}, $\PPcoupling{\xi}{\tilde{\xi}}^{1/2}(T \geq k)  \leq \varsigma^{1/2} \rhotw^{k/2}$. Note also that \cite[Proposition 7]{durmus2025finite} implies
\begin{align}
\PE^{1/4}_\xi \bigl[ \norm{\prod\nolimits_{i=k+1}^n \{ \Id - \alpha \funcA{Z_i} \}}^8\bigr] \leq
\qcond \rme^4 d^{2/q} \rate[1,\alpha]^{2(n-k)}\eqsp,
\end{align}
where $\rate[1,\alpha] = e^{-\alpha a/12}$. Moreover, by \Cref{lem:drift_condition_LSA_Markov}, we get for any $k \in \nset$, that
\[
\PE_\xi^{1/4}\bigl[ \norm{\thalpha{k}{\alpha} - \thetas}^8 \bigr]
\leq 2\qcond e^4 d^{2/q} \rho_{1,\alpha}^{2k}\norm{\theta - \thetas}^2 + 8 \ConstD_2^2 d^{2/q} \alpha a \taumix \supconsteps^2\eqsp,
\]
Combining the bounds above, we obtain that
\begin{align}
T_2 \leq \Auxconst_{\theta, 1} d^{4/q} \rate[1,\alpha]^{2n} (\rhotw^{1/2}/(1-\rhotw^{1/2})) \norm{\theta - \thetas}^{2} + \Auxconst_{\theta, 2} d^{4/q} \alpha a \taumix \sum_{k=1}^{n-1} \rate[1,\alpha]^{2(n-k)} \rhotw^{k/2}\eqsp,
\end{align}
where
\begin{equation}
\label{eq:aux_constants_key_lemma_markov}
\Auxconst_{\theta, 1} = 2\qcond^{2} \rme^{8} \varsigma^{1/2}\eqsp, \quad \Auxconst_{\theta,2} = 8\qcond \rme^{4} \varsigma^{1/2} \ConstD_2^2 \eqsp.
\end{equation}
Note also that the condition $\alpha \leq 3a^{-1}\log\rho^{-1}$ implies $\rhotw^{1/2} \leq \rate[1,\alpha]^{2}$. Combining the above bounds yields
\[
\alpha a \sum_{k=1}^{n-1} \rate[1,\alpha]^{2(n-k)} \rhotw^{k/2} \leq \alpha a n \rate[1,\alpha]^{2n} \leq 12 \rme^{-1} \rate[1,\alpha]^{n}\eqsp.
\]
Hence, we obtain the final bound on $T_2$ as
\begin{align}
T_2 \leq \Auxconst_{\theta,1} d^{4/q} \rate[1,\alpha]^{2n} (\rhotw^{1/2}/(1-\rhotw^{1/2})) \norm{\theta - \thetas}^{2} + \Auxconst_{\theta,3} d^{4/q} \rate[1,\alpha]^{n}\eqsp,
\end{align}
where
\begin{equation}
\label{eq:aux_constants_c3_key_lemma_markov}
\Auxconst_{\theta,3} = 24 \qcond \rme^{3} \varsigma^{1/2} \ConstD_2^2\eqsp.
\end{equation}
Similarly, using \Cref{lem:drift_condition_LSA_Markov} and the definition of the coupling time $T$, we get
\[
T_1 \leq 2\qcond\rme^{4}\varsigma^{1/2} d^{2/q} \rate[1,\alpha]^{2n} \rhotw^{n/2} \norm{\theta-\thetas}^{2} + 8\varsigma^{1/2}\ConstD_2^2 d^{2/q} \alpha a \taumix \rhotw^{n/2}\eqsp.
\]
The previous bounds imply
\begin{align}
T_1 + T_2
&\leq \Auxconst_{\theta, 4} d^{4/q} \rate[1,\alpha]^{2n} \norm{\theta - \thetas}^{2} + \Auxconst_{\theta,5} d^{4/q} \rate[1,\alpha]^{n}\eqsp,
\end{align}
where
\begin{equation}
\label{eq:aux_constants_final_key_lemma_markov}
\Auxconst_{\theta,4} = 2\Auxconst_{\theta,1}\eqsp, \quad \Auxconst_{\theta,5} = 32 \qcond \rme^{3} \varsigma^{1/2} \ConstD_2^{2}\eqsp.
\end{equation}
Combining the bounds above and Minkowski's inequality, we get
\begin{align}
\{\PEcoupling{\xi}{\tilde{\xi}}[ (\|\thalpha{n}{\alpha} - \thalphat{n}{\alpha}\| + \indiacc{\State_n \neq \tilde{\State_n}})^2]\}^{1/2} \leq \Auxconst_{\theta,6} d^{2/q} \rate[\alpha]^{n} (1 + \norm{\theta- \thetas} + \norm{\tilde{\theta} - \thetas})\eqsp,
\end{align}
where
\begin{equation}
\label{eq:aux_const_6_Markov}
\Auxconst_{\theta,6} = \sqrt{\Auxconst_{\theta,4}} + 2\sqrt{\Auxconst_{\theta,5}} + \varsigma^{1/2}\eqsp.
\end{equation}
To conclude the proof, it remains to bound the second term in the \rhs\ of \eqref{eq:holder_c_function} by using \Cref{lem:drift_condition_LSA_Markov}.
\end{proof}

\begin{proof}[Proof of \Cref{theo:existence_pi_alpha}]
    We denote $y = (\theta, z)$ and $\tilde{y} = (\tilde{\theta}, \tilde{z})$ for $\theta, \tilde{\theta} \in \rset^d, z,\tilde{z} \in \Zset$. Using the coupling construction \eqref{eq:coupling_constr_theta} and the contraction of $\costw[0]$ in \Cref{lem:cost_function_decrease}, we get    
    \begin{align}
        \wasserdist_{\costw[0]}(\delta_y \MKjoint[\alpha]^n, \delta_{\tilde{y}} \MKjoint[\alpha]^n) \leq \ConstD_{\theta} d^{2/q} \rate[\alpha]^n \costw[0]((z,\theta), (\tilde{z},\tilde{\theta})) \eqsp.
    \end{align}
    Then, applying \cite[Theorem 20.3.4]{douc:moulines:priouret:soulier:2018}, we conclude that the Markov chain $\{(\thalpha{k}{\alpha}, Z_{k+1}), k \in \nset\}$ with the Markov kernel $\MKjoint[\alpha]$ admits the unique invariant distribution $\Pi_{\alpha}$. Finally, from \cite[Theorem 6.9]{villani:2009} we conclude that $\Pi_{\alpha}(\norm{\theta_0 - \thetalim}) < \infty$.
\end{proof}

\subsection{Contraction for Wasserstein semimetric}
Before the main result of \Cref{lem:contr_wasser_J1}, we should state a preliminary lemmas on contraction of $\{\Jnalpha{n}{0,\alpha}, n \geq 0\}$ and $\{\Jnalpha{n}{1,\alpha}, n \geq 0\}$ iterations.

\begin{lemma}
\label{lem:contr_J0}
    Assume \Cref{assum:A-b}, \Cref{assum:noise-level}, and \Cref{assum:drift}. Fix $J, \tilde{J} \in \rset^d$ and $z, \tilde{z} \in \Zset$. Denote pairs $y = (J, z)$ and $y' = (J', z')$ such that $y \neq y'$. Then, for any $n \geq 1$, $p \geq 1$ and $\alpha \in (0, \alpha_{\infty} \wedge (ap)^{-1} \ln{\rho^{-1}})$, we have
    \begin{align}
    \label{eq:contr_J0_expectation}
        \PEcoupling{y}{y'}^{1/p}[\norm{\Jnalpha{n}{0,\alpha} - \Jnalphat{n}{0,\alpha}}^p] \leq \Auxconst_{W,1} \taumix^{1/2} p^{1/2} \rho_{1,\alpha}^{n/p}(\norm{J} + \norm{J'} + \sqrt{\alpha a} \supconsteps) \eqsp,
    \end{align}
    where $\Auxconst_{W,1}$ is defined in \eqref{eq:auxconst_W1} and $\rho_{1,\alpha} = e^{-\alpha a / 12}$.
\end{lemma}

\begin{proof}
    Using the definition of exact coupling time, we get the decomposition
    \begin{align}
    \label{eq:decomp}
        \Jnalpha{n}{0,\alpha} - \Jnalphat{n}{0,\alpha} = (\Id - \alpha \bA)^{n- (n \wedge T)}(\Jnalpha{n\wedge T}{0,\alpha} - \Jnalphat{n\wedge T}{0,\alpha}) \eqsp.
    \end{align}
    Using Holder's and Minkowski's inequalities, we have
    \begin{align}
        \PEcoupling{y}{y'}[&\norm{(\Id - \alpha \bA)^{n-n\wedge T}}^{p} \cdot \norm{\Jnalpha{n\wedge T}{0,\alpha} - \Jnalphat{n\wedge T}{0,\alpha}}^{p}] \leq \PEcoupling{y}{y'}^{1/2}[\norm{\Jnalpha{n}{0,\alpha} - \Jnalphat{n}{0,\alpha}}^{2p}] \PPcoupling{z}{z'}^{1/2}(T \geq n)\\
        &+ \qcond^{p/2} \sum_{j=1}^{n-1} (1-\alpha a)^{p(n-j)/2} \PEcoupling{y}{y'}^{1/4}[\norm{\Jnalpha{j}{0,\alpha} - \Jnalphat{j}{0,\alpha}}^{4p}]  \PPcoupling{z}{z'}^{1/2}(T = j) = \Tterm{1}{2} + \Tterm{2}{2} \eqsp.
        \end{align}
    First, note that using \Cref{lem:stationary_rosenthal_jn_0} and Minkowski's inequality, we have uniform bound independent on $n, z$ and $z'$
    \begin{align}
    \label{eq:bound_Jn_diff}
        \PEcoupling{y}{y'}^{1/p}[\norm{\Jnalpha{n}{0,\alpha} - \Jnalphat{n}{0,\alpha}}^p] \leq \qcond^{1/2}(1- \alpha a)^{n/2}(\norm{J} + \norm{J'})  + 4 \ConstD_{1} \sqrt{\alpha a \taumix p}\supconsteps \eqsp.
    \end{align}
    Then, using this observation, the definition of the maximal coupling \eqref{eq:maximal_coupling_def}, $\PPcoupling{\xi}{\xi'}^{1/2}(T \geq k)  \leq \varsigma^{1/2} \rhotw^{k/2}$, and \Cref{lem:stationary_rosenthal_jn_0},  we get
    \begin{align}
        \Tterm{2}{2} &\leq 2^{2p}\qcond^{p} (1 - \alpha a)^{np/2} \varsigma^{1/2}{\rho^{1/2} \over 1 - \rho^{1/2}} (\norm{J}^{p} + \norm{J'}^{p}) \\
        &+ 2^{4p} \ConstD_1^{p} \varsigma^{1/2} \qcond^{p/2} (\alpha a \taumix p)^{p/2} \supconsteps^{p}  \sum_{j=1}^{n-1} (1 - \alpha a)^{p(n-j)/2} \rho^{j/2} \eqsp.
    \end{align}
    Thus, the sum in the last term can be bounded as
    \begin{align}
        \label{eq:contraction_sum}
       \sum_{j=1}^{n-1} (1 - \alpha a)^{p(n-j)/2} \rho^{j/2} &\leq \sum_{j=1}^{n-1}\rho_{1,\alpha}^{p(n-j)}\rho^{2j} \leq \rho_{1,\alpha}^{np} \sum_{j=1}^{n-1}(\rho^{1/2}\rho_{1,\alpha}^{-p})^j \leq 2\rho_{1,\alpha}^{np} \eqsp.
    \end{align}
    where we used that $\sum\limits_{j=1}^{n-1} (\rho^{1/2}\rho_{1,\alpha}^{-p})^j \leq 2$ whenever $\alpha \leq {12 \over ap}\ln{{1 \over \rho^{1/2}}}$. Therefore, we have
    \begin{align}
    \label{eq:contr_T22}
        \Tterm{2}{2} \leq 2^{2p} \qcond^{p} \varsigma^{1/2} \rho_{1,\alpha}^{np}(\norm{J}^{p} + \norm{J'}^{p}) + 2^{4(p+1)} \ConstD_1^{p} \varsigma^{1/2} \qcond^{p/2} (\alpha a)^{p/2} (\taumix p)^{p/2} \supconsteps^{p} \rho_{1,\alpha}^{np} \eqsp.
    \end{align}
    In what follows, we use the inequality $\rho^{1/2} \leq \rho_{1,\alpha}^2$, which holds for $\alpha \leq 3a^{-1}\log{\rho^{-1}}$. For the first term $\Tterm{1}{2}$ we can again use the inequality \eqref{eq:bound_Jn_diff}, and get
    \begin{align}
        \label{eq:contr_T21}
        \Tterm{1}{2} \leq 2^{2p} \qcond^{p/2} \varsigma^{1/2}\rho_{1,\alpha}^{np} (\norm{J}^{p} + \norm{J'}^{p}) + 2^{4p} \ConstD_1^{p} \varsigma^{1/2} (\alpha a\taumix p)^{p/2} \supconsteps^{p} \rho_{1,\alpha}^{n} \eqsp.
    \end{align}
    Combining together \eqref{eq:contr_T21} and \eqref{eq:contr_T22}, we obtain
    \begin{align}
    \label{eq:contr_J0_expectation}
        \PEcoupling{y}{y'}^{1/p}[\norm{\Jnalpha{n}{0,\alpha} - \Jnalphat{n}{0,\alpha}}^p] &\leq (\Tterm{1}{2})^{1/p} + (\Tterm{2}{2})^{1/p} \leq \Auxconst_{W,1} \taumix^{1/2} p^{1/2} \rho_{1,\alpha}^{n/p}(\norm{J} + \norm{J'} + \sqrt{\alpha a} \supconsteps) \eqsp,
    \end{align}
    where we set
    \begin{equation}
    \label{eq:auxconst_W1}
        \Auxconst_{W,1} = \varsigma^{1/2p}(4\qcond^{1/2}(\qcond^{1/2} + 1) + 2^{8} \qcond^{1/2}\ConstD_1 + 2^{4}\ConstD_1) \eqsp.
    \end{equation}
\end{proof}

\begin{lemma}
    \label{lem:contr_J1}
    Assume \Cref{assum:A-b}, \Cref{assum:noise-level}, and \Cref{assum:drift}. Fix $J, \tilde{J} \in \rset^d$ and $z, \tilde{z} \in \Zset$. Denote pairs $y = (J, z)$ and $y' = (J', z')$ such that $y \neq y'$. Then, for any $n \geq 1$, $p \geq 1$ and $\alpha \in (0, \alpha_{\infty} \wedge (ap)^{-1} \ln{\rho^{-1}})$, we have
    \begin{align}
    \label{eq:contr_J0_expectation}
        \PEcoupling{y}{y'}^{1/p}[\norm{\Jnalpha{n}{1,\alpha} - \Jnalphat{n}{1,\alpha}}^p] \leq \Auxconst_{W,1}^{(1)} p^2 \taumix^{3/2} \rho_{1,\alpha}^{n/p}\sqrt{\log{(1/\alpha a)}} (\norm{J^{(0)}} + \norm{\tilde{J}^{(0)}} + \norm{J^{(1)}} + \norm{\tilde{J}^{(1)}} + \sqrt{\alpha a}\supconsteps) \eqsp,
    \end{align}
    where $\Auxconst_{W,1}^{(1)}$ is defined in \eqref{eq:auxconst_W11} and $\rho_{1,\alpha} = e^{-\alpha a / 12}$.
\end{lemma}

\begin{proof}
    We use the exact coupling construction $\eqref{eq:coupling_time_def_markov}$ for the Markov chains
     $\{Z_k, k \geq 1\}$ and $\{\tilde{Z}_k, k \geq 1\}$ with coupling time $T$. We have the decomposition
    \begin{align}
    \label{eq:repr_Jn1}
        \Jnalpha{n}{1,\alpha} - \Jnalphat{n}{1,\alpha} &= (\Id - \alpha \bA)^{n- n\wedge T} (\Jnalpha{n\wedge T}{1,\alpha} - \Jnalphat{n\wedge T}{1,\alpha}) \\
        &- \alpha \indiacc{T \leq n} \sum_{k=1}^{n - n\wedge T + 1} (\Id - \alpha \bA)^{k-1}\zmfuncA{Z_{n-k+1}}(\Jnalpha{n-k}{0,\alpha} - \Jnalphat{n-k}{0,\alpha})\\
        &=: \Tterm{\Jnalpha{}{1}}{1} + \Tterm{\Jnalpha{}{1}}{2} \eqsp.
    \end{align}
    We bound the two terms separately. For the first term, we can proceed the similar steps as in \Cref{lem:contr_wasser}. Thus, using Holder's and Minkowski's inequalities, we get
    \begin{align}
        \PEcoupling{y}{\tilde{y}}[&\norm{(\Id - \alpha \bA)^{n-n\wedge T}}^{p} \cdot \norm{\Jnalpha{n\wedge T}{1,\alpha} - \Jnalphat{n\wedge T}{1,\alpha}}^{p}] \leq \PEcoupling{y}{\tilde{y}}^{1/2}[\norm{\Jnalpha{n}{1,\alpha} - \Jnalphat{n}{1,\alpha}}^{2p}] \PPcoupling{z}{\tilde{z}}^{1/2}(T \geq n)\\
        &+ \qcond^{p/2} \sum_{j=1}^{n-1} (1-\alpha a)^{p(n-j)/2} \PEcoupling{y}{\tilde{y}}^{1/4}[\norm{\Jnalpha{j}{1,\alpha} - \Jnalphat{j}{1,\alpha}}^{4p}]  \PPcoupling{z}{\tilde{z}}^{1/2}(T = j) = \Tterm{1}{3} + \Tterm{2}{3} \eqsp.
    \end{align}
    To bound the term $T_{1}^{(3)}$, we apply \Cref{lem:bound_Jn1}
    \begin{align} 
    \label{eq:bound_Jn1}
        \PEcoupling{y}{\tilde{y}}^{1/p}[\norm{\Jnalpha{n}{1,\alpha} - \Jnalphat{n}{1,\alpha}}^p] &\leq \qcond^{1/2}(1-\alpha a)^{n/2}(\norm{J^{(1)}} + \norm{\tilde{J}^{(1)}}) + 2(\ConstDM_{J,1} + \ConstDM_{J,2})\supconsteps p^2 \taumix^{3/2} \alpha a \sqrt{\log(1/\alpha a)} \eqsp.
    \end{align}
    Using \eqref{eq:bound_Jn1}, we get
    \begin{align}
        &\Tterm{2}{3} \leq 4^p \qcond^{p} (1 - \alpha a)^{np/2} \zeta^{1/2}{\rho^{1/2} \over 1 - \rho^{1/2}}(\norm{J^{(1)}}^p + \norm{\tilde{J}^{(1)}}^p)\\
        &+2^{6p}(\ConstD_{J,3}^{(M)})^p \qcond^{p/2} \zeta^{1/2}  p^{2p}\taumix^{3p/2} (\alpha a)^p (\log{(1/\alpha a)})^{p/2}\supconsteps^{p} \sum_{j=1}^{n-1} (1- \alpha a)^{p(n-j)/2}\rho^{j/2} \eqsp,
    \end{align}
    where we set $\ConstD_{J,3}^{(M)} = \ConstD_{J,1}^{(M)} + \ConstD_{J,2}^{(M)}$. Now, the bound for $\Tterm{2}{3}$ follows from \eqref{eq:contraction_sum}. We conclude that
    \begin{align}
        \Tterm{2}{3} &\leq 4^p \qcond^{p} (1 - \alpha a)^{np/2} \zeta^{1/2}{\rho^{1/2} \over 1 - \rho^{1/2}}(\norm{J^{(1)}}^p + \norm{\tilde{J}^{(1)}}^p)\\
        &+2^{6p+1}(\ConstD_{J,3}^{(M)})^p \qcond^{p/2} \zeta^{1/2}  p^{2p}\taumix^{3p/2}  (\alpha a)^p (\log{(1/\alpha a)})^{p/2} \rho_{1,\alpha}^{np} \supconsteps^{p} \eqsp.
    \end{align}
    Applying again \eqref{eq:bound_Jn1} and the fact that $\rho^{1/2} \leq \rho_{1,\alpha}^{2}$, we get
    \begin{align}
        \Tterm{1}{3} &\leq 2^{2p} \qcond^{p/2} \varsigma^{1/2}\rho_{1,\alpha}^{np} (\norm{J^{(1)}}^{p} + \norm{\tilde{J}^{(1)}}^{p}) \\
        &+ 2^{4p} (\ConstD_{J,3}^{(M)})^{p} \varsigma^{1/2} p^{2p} \taumix^{3p/2} (\alpha a)^{p} (\log{(1/\alpha a)})^{p/2}  \rho_{1,\alpha}^{n} \supconsteps^{p} \eqsp.
    \end{align}
    Now, we bound the term $\Tterm{J^{(1)}}{2}$. Firstly, we note that for any $j \geq 1$, using \Cref{lem:contr_J0} and Minkowski's inequality, we get
    \begin{align}
    \label{eq:bound_remainder_J1}
        &\PEcoupling{y}{\tilde{y}}^{1/p}[\norm{\sum_{k=1}^{n-j+1}(\Id - \alpha \bA)^{k-1}\zmfuncA{Z_{n-k+1}}(\Jnalpha{n-k}{0,\alpha} - \Jnalphat{n-k}{0,\alpha})}^p] \\
        &\leq \bConst{A} \qcond^{1/2}\sum_{k=1}^{n-j+1} (1 - \alpha a)^{(k-1)/2} \PE_{y,\tilde{y}}^{1/p}[\norm{\Jnalpha{n-k}{0,\alpha} - \Jnalphat{n-k}{0,\alpha}}^p]\\
        &\leq 4\Auxconst_{W,1} \bConst{A}\qcond^{1/2} \taumix^{1/2}p^{1/2} \rho_{1,\alpha}^{n/p} (\alpha a)^{-1}(\norm{J^{(0)}} + \norm{\tilde{J}^{(0)}} + \sqrt{\alpha a}\supconsteps) \eqsp.
    \end{align}
    Thus, using \eqref{eq:bound_remainder_J1} and Holder's inequality, we obtain
    \begin{align}
    \PEcoupling{y}{\tilde{y}}[\norm{\Tterm{J^{(1)}}{2}}^p]
    &\leq \alpha^p \sum_{j=1}^n \PE_{y,\tilde{y}}^{1/2}[\norm{\sum_{k=1}^{n-j+1}(\Id-\alpha \bA)^{k-1}\zmfuncA{Z_{n-k+1}}(\Jnalpha{n-k}{0,\alpha} - \Jnalphat{n-k}{0,\alpha})}^{2p}]\PPcoupling{z}{\tilde{z}}^{1/2}(T = j)\\
    &\leq 2^{3p} \Auxconst_{W,1}^p \bConst{A}^p \zeta^{1/2} \qcond^{p/2} (\taumix p)^{p/2} \rho_{1,\alpha}^{n} a^{-p} (\norm{J^{(0)}} + \norm{\tilde{J}^{(0)}} + \sqrt{\alpha a}\supconsteps)^p \sum_{j=1}^n \rho^{j/2}\\
    &\leq 2^{3p} \Auxconst_{W,1}^p \bConst{A}^p \zeta^{1/2} \qcond^{p/2} {\rho^{1/2} \over 1 - \rho^{1/2}}(\taumix p)^{p/2} \rho_{1,\alpha}^{n} a^{-p} (\norm{J^{(0)}} + \norm{\tilde{J}^{(0)}} + \sqrt{\alpha a}\supconsteps)^p \eqsp.
    \end{align}
    Thus, we get the bound for \eqref{eq:repr_Jn1}, that is
    \begin{align}
    \label{eq:bound_wasser_J1_term1}
    \PEcoupling{y}{\tilde{y}}^{1/p}[\norm{\Jnalpha{n}{1,\alpha} - \Jnalpha{n}{1,\alpha}}^p] 
    &\leq \PE_{y,\tilde{y}}^{1/p}[\norm{\Tterm{J^{(1)}}{1}}^p] + \PE_{y,\tilde{y}}^{1/p}[\norm{\Tterm{J^{(1)}}{2}}^p] \leq (\Tterm{1}{3})^{1/p} + (\Tterm{2}{3})^{1/p} + \PE_{y,\tilde{y}}^{1/p}[\norm{\Tterm{J^{(1)}}{2}}^p]\\
    &\leq \Auxconst_{W,1}^{(1)} p^2 \taumix^{3/2} \rho_{1,\alpha}^{n/p}\sqrt{\log{(1/\alpha a)}} (\norm{J^{(0)}} + \norm{\tilde{J}^{(0)}} + \norm{J^{(1)}} + \norm{\tilde{J}^{(1)}} + \sqrt{\alpha a}\supconsteps) \eqsp,
    \end{align}
    where we set
    \begin{equation}
    \label{eq:auxconst_W11}
        \Auxconst_{W,1}^{(1)} = \zeta^{1/2p}(148(\qcond^{1/2}(1 + \qcond^{1/2}) + \ConstDM_{J,3}) + 8 \Auxconst_{W,1}\qcond^{1/2}a^{-1})\eqsp.
    \end{equation}
\end{proof}

Now, we are going to establish the result about asymptotic bias. As we will show, this bias is closely related to the limiting distribution of the sequences $\{\Jnalpha{t}{1, \alpha}, t \geq 0\}$ and $\{\Jnalpha{t}{2, \alpha}, t \geq 0\}$. In order to accurately define these distributions, we consider the Markov chain $Y_t = (Z_{t+1}, \Jnalpha{t}{0,\alpha}, \Jnalpha{t}{1,\alpha}, \Jnalpha{t}{2,\alpha})$ for any $t \geq 0$ with kernel $\MKQ_{J^{(2)}}$. Denoting $Y = (z, J^{(0)}, J^{(1)}, J^{(2)})$ and $\tilde{Y} = (\tilde{z}, \tilde{J}^{(0)}, \tilde{J}^{(1)}, \tilde{J}^{(2)})$, we define the cost function
\begin{align}
    \label{eq:cost_func_def_J2}
    \costw[J^{(2)}](Y, \tilde{Y}) &= \norm{J^{(0)} - \tilde{J}^{(0)}} + \norm{J^{(1)} - \tilde{J}^{(1)}} + \norm{J^{(2)}-\tilde{J}^{(2)}} \\
    &+ (\norm{J^{(0)}} + \norm{\tilde{J}^{(0)}} + \norm{J^{(1)}} + \norm{\tilde{J}^{(1)}} + \norm{J^{(2)}} + \norm{\tilde{J}^{(2)}} + \sqrt{\alpha a}\supconsteps)\indiacc{z \neq \tilde{z}} \eqsp.
\end{align}
Now, we introduce the main result of this section on contraction of Wasserstein distance for the coupling of $Y_t$ and $\tilde{Y}_t$.
\begin{proposition}
\label{lem:contr_wasser_J2}
    Assume \Cref{assum:A-b}, \Cref{assum:noise-level}, and \Cref{assum:drift}. Fix $J^{(0)}, \tilde{J}^{(0)}, J^{(1)}, \tilde{J}^{(1)}, J^{(2)}, \tilde{J}^{(2)} \in \rset^d$ and $z, \tilde{z} \in \Zset$. Denote $y = (z, J^{(0)}, J^{(1)}, J^{(2)})$ and $\tilde{y} = (\tilde{z}, \tilde{J}^{(0)}, \tilde{J}^{(1)}, \tilde{J}^{(2)})$ such that $y \neq \tilde{y}$. Then, for any $n \geq 1$, $p \geq 1$ and $\alpha \in (0, \alpha_{\infty} \wedge (ap)^{-1} \ln{\rho^{-1}})$, we have
    \begin{align}
         \wasserdist_{\costw[J^{(2)}],p}^{1/p}(\delta_y \MKQ_{J^{(2)}}^n, \delta_{\tilde{y}}\MKQ_{J^{(2)}}^n) \leq \Auxconst_{W,3}^{(2)} p^{7/2} \taumix^{5/2} \rho_{1,\alpha}^{n/p} (\log{(1/\alpha a)})^{3/2} \costw[J^{(2)}](y, \tilde{y}) \eqsp,
    \end{align}
    where $\Auxconst_{W,3}^{(2)}$ is defined in \eqref{eq:def_cW3_2}.
\end{proposition}
\begin{proof}
     We use the similar construction with exact coupling as in \Cref{lem:contr_J1}. We have the decomposition
    \begin{align}
    \label{eq:repr_Jn2}
        \Jnalpha{n}{2,\alpha} - \Jnalphat{n}{2,\alpha} &= (\Id - \alpha \bA)^{n- n\wedge T} (\Jnalpha{n\wedge T}{2,\alpha} - \Jnalphat{n\wedge T}{2,\alpha}) \\
        &- \alpha \indiacc{T \leq n} \sum_{k=1}^{n - n\wedge T + 1} (\Id - \alpha \bA)^{k-1}\zmfuncA{Z_{n-k+1}}(\Jnalpha{n-k}{1,\alpha} - \Jnalphat{n-k}{1,\alpha}) = \Tterm{\Jnalpha{}{2}}{1} + \Tterm{\Jnalpha{}{2}}{2} \eqsp.
    \end{align}
    We bound the two terms separately. For the first term, we can proceed the similar steps as in \Cref{lem:contr_wasser}. Thus, using Holder's and Minkowski's inequalities, we get
    \begin{align}
        \PEcoupling{y}{\tilde{y}}[&\norm{(\Id - \alpha \bA)^{n-n\wedge T}}^{p} \cdot \norm{\Jnalpha{n\wedge T}{2,\alpha} - \Jnalphat{n\wedge T}{2,\alpha}}^{p}] \leq \PEcoupling{y}{\tilde{y}}^{1/2}[\norm{\Jnalpha{n}{2,\alpha} - \Jnalphat{n}{2,\alpha}}^{2p}] \PPcoupling{z}{\tilde{z}}^{1/2}(T \geq n)\\
        &+ \qcond^{p/2} \sum_{j=1}^{n-1} (1-\alpha a)^{p(n-j)/2} \PEcoupling{y}{\tilde{y}}^{1/4}[\norm{\Jnalpha{j}{2,\alpha} - \Jnalphat{j}{2,\alpha}}^{4p}]  \PPcoupling{z}{\tilde{z}}^{1/2}(T = j) = \Tterm{1}{4} + \Tterm{2}{4} \eqsp.
    \end{align}
    To bound the term $\Tterm{1}{4}$, we apply \Cref{prop:bound_Jn2_alpha}
    \begin{align} 
    \label{eq:bound_Jn2}
        \PEcoupling{y}{\tilde{y}}^{1/p}[\norm{\Jnalpha{n}{2,\alpha} - \Jnalphat{n}{2,\alpha}}^p] &\leq \qcond^{1/2}(1-\alpha a)^{n/2}(\norm{J^{(2)}} + \norm{\tilde{J}^{(2)}}) + 2\ConstD_{J} \taumix^{5/2} p^{7/2} \alpha^{3/2}\log^{3/2}(1/\alpha a) \eqsp.
    \end{align}
    Using \eqref{eq:bound_Jn2}, we get
    \begin{align}
        \Tterm{2}{4} &\leq 4^p \qcond^{p} (1 - \alpha a)^{np/2} \zeta^{1/2}{\rho^{1/2} \over 1 - \rho^{1/2}}(\norm{J^{(2)}}^p + \norm{\tilde{J}^{(2)}}^p)\\
        &+2^{6p}(\ConstD_{J})^p \qcond^{p/2} \zeta^{1/2}  p^{7p/2}\taumix^{5p/2} \alpha^{3p/2} (\log{(1/\alpha a)})^{3p/2}\supconsteps^{p} \sum_{j=1}^{n-1} (1- \alpha a)^{p(n-j)/2}\rho^{j/2} \eqsp.
    \end{align}
    Now, the bound for $\Tterm{2}{4}$ follows from \eqref{eq:contraction_sum}. We conclude that
    \begin{align}
        \Tterm{2}{4} &\leq 4^p \qcond^{p} (1 - \alpha a)^{np/2} \zeta^{1/2}{\rho^{1/2} \over 1 - \rho^{1/2}}(\norm{J^{(2)}}^p + \norm{\tilde{J}^{(2)}}^p)\\
        &+2^{6p+1}(\ConstD_{J})^p \qcond^{p/2} \zeta^{1/2}  p^{7p/2}\taumix^{5p/2}  \alpha^{3p/2} (\log{(1/\alpha a)})^{3p/2} \rho_{1,\alpha}^{np} \supconsteps^{p} \eqsp.
    \end{align}
    Applying again \eqref{eq:bound_Jn2} and the fact that $\rho^{1/2} \leq \rho_{1,\alpha}^{2}$, we get
    \begin{align}
        \Tterm{1}{4} &\leq 2^{2p} \qcond^{p/2} \varsigma^{1/2}\rho_{1,\alpha}^{np} (\norm{J^{(2)}}^{p} + \norm{\tilde{J}^{(2)}}^{p}) \\
        &+ 2^{4p} (\ConstD_{J})^{p} \varsigma^{1/2} p^{7p/2} \taumix^{5p/2} \alpha^{3p/2} (\log{(1/\alpha a)})^{3p/2}  \rho_{1,\alpha}^{n} \supconsteps^{p} \eqsp.
    \end{align}
    Now, we bound the term $\Tterm{J^{(2)}}{2}$. Firstly, we note that for any $j \geq 1$, using \Cref{lem:contr_J1} and Minkowski's inequality, we get
    \begin{align}
    \label{eq:bound_remainder_J2}
        &\PEcoupling{y}{\tilde{y}}^{1/p}[\norm{\sum_{k=1}^{n-j+1}(\Id - \alpha \bA)^{k-1}\zmfuncA{Z_{n-k+1}}(\Jnalpha{n-k}{1,\alpha} - \Jnalphat{n-k}{1,\alpha})}^p] \\
        &\leq \bConst{A} \qcond^{1/2}\sum_{k=1}^{n-j+1} (1 - \alpha a)^{(k-1)/2} \PE_{y,\tilde{y}}^{1/p}[\norm{\Jnalpha{n-k}{1,\alpha} - \Jnalphat{n-k}{1,\alpha}}^p]\\
        &\leq 4\Auxconst_{W,1}^{(1)} \bConst{A}\qcond^{1/2} p^2 \taumix^{3/2} \rho_{1,\alpha}^{n/p} (\alpha a)^{-1}\sqrt{\log{(1/\alpha a)}}(\norm{J^{(1)}} + \norm{\tilde{J}^{(1)}} + \sqrt{\alpha a}\supconsteps) \eqsp.
    \end{align}
    Thus, using \eqref{eq:bound_remainder_J2} and Holder's inequality, we obtain
    \begin{align}
    \PEcoupling{y}{\tilde{y}}[\norm{\Tterm{J^{(2)}}{2}}^p]
    &\leq \alpha^p \sum_{j=1}^n \PE_{y,\tilde{y}}^{1/2}[\norm{\sum_{k=1}^{n-j+1}(\Id-\alpha \bA)^{k-1}\zmfuncA{Z_{n-k+1}}(\Jnalpha{n-k}{1,\alpha} - \Jnalphat{n-k}{1,\alpha})}^{2p}]\PPcoupling{z}{\tilde{z}}^{1/2}(T = j)\\
    &\leq 2^{3p} (\Auxconst_{W,1}^{(1)})^p \bConst{A}^p \zeta^{1/2} \qcond^{p/2} p^{7p/2}\taumix^{5p/2} \rho_{1,\alpha}^{n} a^{-p} \sqrt{\log{(1/\alpha a)}}(\norm{J^{(1)}} + \norm{\tilde{J}^{(1)}} + \sqrt{\alpha a}\supconsteps)^p \sum_{j=1}^n \rho^{j/2}\\
    &\leq 2^{3p} (\Auxconst_{W,1}^{(1)})^p \bConst{A}^p \zeta^{1/2} \qcond^{p/2} {\rho^{1/2} \over 1 - \rho^{1/2}} p^{7p/2}\taumix^{5p/2} \rho_{1,\alpha}^{n} a^{-p} \sqrt{\log{(1/\alpha a)}} (\norm{J^{(1)}} + \norm{\tilde{J}^{(1)}} + \sqrt{\alpha a}\supconsteps)^p \eqsp.
    \end{align}
    Thus, we obtain the bound for \eqref{eq:repr_Jn2}, that is
    \begin{align}
    \label{eq:bound_wasser_J2_term1}
    \PEcoupling{y}{\tilde{y}}^{1/p}[\norm{\Jnalpha{n}{2,\alpha} - \Jnalpha{n}{2,\alpha}}^p] 
    &\leq \PE_{y,\tilde{y}}^{1/p}[\norm{\Tterm{J^{(2)}}{1}}^p] + \PE_{y,\tilde{y}}^{1/p}[\norm{\Tterm{J^{(2)}}{2}}^p] \leq (\Tterm{1}{4})^{1/p} + (\Tterm{2}{4})^{1/p} + \PE_{y,\tilde{y}}^{1/p}[\norm{\Tterm{J^{(2)}}{2}}^p]\\
    &\leq \Auxconst_{W,1}^{(2)} p^{7/2} \taumix^{5/2} \rho_{1,\alpha}^{n/p}(\log{(1/\alpha a)})^{3/2}(\norm{J^{(0)}} + \norm{\tilde{J}^{(0)}} + \norm{J^{(1)}} + \norm{\tilde{J}^{(1)}} + \sqrt{\alpha a}\supconsteps) \eqsp,
    \end{align}
    where we set
    \begin{equation}
        \Auxconst_{W,1}^{(2)} = \zeta^{1/2p}(148(\qcond^{1/2}(1 + \qcond^{1/2}) + \ConstD_{J}) + 8 \Auxconst_{W,1}^{(1)}\qcond^{1/2}a^{-1})\eqsp.
    \end{equation}
    Finally, using the Holder's and Minkowski's inequality, we get
    \begin{align}
    \label{eq:bound_wasser_J2_term2}
        &\PEcoupling{y}{\tilde{y}}^{1/p}[(\norm{\Jnalpha{n}{0,\alpha}} + \norm{\Jnalphat{n}{0,\alpha}} + \norm{\Jnalpha{n}{1,\alpha}} + \norm{\Jnalphat{n}{1,\alpha}} + \norm{\Jnalpha{n}{2,\alpha}} + \norm{\Jnalphat{n}{2,\alpha}} + \sqrt{\alpha a} \supconsteps)^p \indiacc{\State_n \neq \State_n'}] \\
        &\qquad \leq (\PEcoupling{y}{\tilde{y}}^{1/2p}[\norm{\Jnalpha{n}{0,\alpha}}^{2p}] + \PEcoupling{y}{\tilde{y}}^{1/2p}[\norm{\Jnalphat{n}{0,\alpha}}^{2p}] + \PEcoupling{y}{\tilde{y}}^{1/2p}[\norm{\Jnalpha{n}{1,\alpha}}^{2p}] + \PEcoupling{y}{\tilde{y}}^{1/2p}[\norm{\Jnalphat{n}{1,\alpha}}^{2p}] \\
        &\qquad + \PEcoupling{y}{\tilde{y}}^{1/2p}[\norm{\Jnalpha{n}{2,\alpha}}^{2p}] + \PEcoupling{y}{\tilde{y}}^{1/2p}[\norm{\Jnalphat{n}{2,\alpha}} +\sqrt{\alpha a} \supconsteps) \PPcoupling{z}{z'}^{1/2p}(T \geq n)\\
        &\qquad \leq \Auxconst_{W, 2}^{(2)} p^{7/2} \taumix^{5/2} \rho_{1,\alpha}^{n/p} (\log{(1/\alpha a)})^{3/2} (\norm{J^{(0)}} + \norm{\tilde{J}^{(0)}} + \norm{J^{(1)}} + \norm{\tilde{J}^{(1)}} + \norm{J^{(2)}} + \norm{\tilde{J}^{(2)}} + \sqrt{\alpha a}\supconsteps) \eqsp,
    \end{align}
    where we define
    \begin{equation}
        \Auxconst_{W, 2}^{(2)} = \zeta^{1/2p}(2\ConstD_1 + 4\qcond^{1/2} + 8\ConstD_J) \eqsp.
    \end{equation}
    Finally, combining the results \eqref{eq:bound_wasser_J2_term1} and \eqref{eq:bound_wasser_J2_term2}, we obtain
    \begin{align}
        \wasserdist_{\costw[J^{(2)}],p}^{1/p}(\delta_y \MKQ_{J^{(1)}}^n, \delta_{\tilde{y}}\MKQ_{J^{(1)}}^n) &\leq \PEcoupling{y}{\tilde{y}}[\costw[J^{(2)}]^p((Z_{n+1}, \Jnalpha{n}{0,\alpha}, \Jnalpha{n}{1,\alpha}), (\tilde{Z}_{n+1}, \Jnalphat{n}{0,\alpha}, \Jnalphat{n}{1,\alpha}))] \\
        &\leq \Auxconst_{W,3}^{(2)} p^{7/2} \taumix^{5/2} \rho_{1,\alpha}^{n/p} (\log{(1/\alpha a)})^{3/2} \costw[J^{(2)}](y, \tilde{y}) \eqsp,
    \end{align}
    where
    \begin{align}
    \label{eq:def_cW3_2}
        \Auxconst_{W,3}^{(2)} = \Auxconst_{W,1}^{(2)} + \Auxconst_{W,2}^{(2)} \eqsp.
    \end{align}
\end{proof}
\begin{corollary}
\label{lem:prop_inv_dist_J2}
Assume \Cref{assum:A-b}, \Cref{assum:noise-level} and \Cref{assum:drift}. Let $\alpha \in (0,\alpha_{\infty}^{(\sf b)})$. Then the process $\{Y_t\}_{t \in \nset}$ is a Markov chain with a unique stationary distribution $\Pi_{J^{(2)}, \alpha}$ .
\end{corollary}
\begin{proof}
    Using \Cref{lem:contr_wasser_J2}, we follow the lines of \Cref{appendix:inv_dist_J1_exist}.
\end{proof}

The similar result as in \Cref{lem:contr_wasser_J2} can be obtained for the Markov chain $\{(Z_{t+1}, \Jnalpha{t}{0,\alpha}, \Jnalpha{t}{1,\alpha}), t \geq 0\}$ with kernel $\MKQ_{J^{(1)}}$, but with a sharper bound. That is, we set $U = (z, \Jnalpha{}{0}, \Jnalpha{}{1})$, $\tilde{U} = (\tilde{z}, \Jnalphat{}{0}, \Jnalphat{}{1})$ for $\Jnalpha{}{0}, \Jnalphat{}{0}, \Jnalpha{}{1}, \Jnalphat{}{1} \in \rset^d$, $z, \tilde{z} \in \Zset$, and consider another cost function
\begin{align}
    \label{eq:cost_func_def_J1}
    \costw(U, \tilde{U}) &= \norm{J^{(0)} - \tilde{J}^{(0)}} + \norm{J^{(1)} - \tilde{J}^{(1)}} \\
    &+ (\norm{J^{(0)}} + \norm{\tilde{J}^{(0)}} + \norm{J^{(1)}} + \norm{\tilde{J}^{(1)}} + \sqrt{\alpha a}\supconsteps)\indiacc{z \neq \tilde{z}} \eqsp.
\end{align}
We establish the result on contraction of the Wasserstein semimetric for this cost function. 
\begin{lemma}
\label{lem:contr_wasser_J1}
    Assume \Cref{assum:A-b}, \Cref{assum:noise-level}, and \Cref{assum:drift}. Fix $J^{(0)}, \tilde{J}^{(0)}, J^{(1)}, \tilde{J}^{(1)}  \in \rset^d$ and $z, \tilde{z} \in \Zset$. Denote $y = (z, J^{(0)}, J^{(1)})$ and $\tilde{y} = (\tilde{z}, \tilde{J}^{(0)}, \tilde{J}^{(1)})$ such that $y \neq \tilde{y}$. Then, for any $n \geq 1$, $p \geq 1$ and $\alpha \in (0, \alpha_{\infty} \wedge (ap)^{-1} \ln{\rho^{-1}})$, we have
    \begin{align}
    \label{eq:contr_wasser_J1}
         \wasserdist_{\costw,p}^{1/p}(\delta_y \MKQ_{J^{(1)}}^n, \delta_{\tilde{y}}\MKQ_{J^{(1)}}^n) \leq \Auxconst_{W,3}^{(1)} p^2 \taumix^{3/2} \rho_{1,\alpha}^{n/p} \sqrt{\log{(1/\alpha a)}} \costw(y, \tilde{y}) \eqsp,
    \end{align}
    where $\Auxconst_{W,3}^{(1)}$ is defined in \eqref{eq:def_cW3_1}.
\end{lemma}
\begin{proof}
    Following the proof lines of \Cref{lem:contr_wasser_J2} but using \Cref{lem:contr_J0} instead of \Cref{lem:contr_J1}, we can obtain the result \eqref{eq:contr_wasser_J1} with
    \begin{align}
    \label{eq:def_cW3_1}
        &\Auxconst_{W,3}^{(1)} = \Auxconst_{W,1}^{(1)} + \Auxconst_{W,2}^{(1)} \eqsp,\\
        &\Auxconst_{W,1}^{(1)} = \zeta^{1/2p}(148(\qcond^{1/2}(1 + \qcond^{1/2}) + \ConstDM_{J,3}) + 8 \Auxconst_{W,1}\qcond^{1/2}a^{-1}) \eqsp,\\
        &\Auxconst_{W, 2}^{(1)} = \zeta^{1/2p}(2\ConstD_1 + 4\qcond^{1/2} + 8(\ConstDM_{J,1} + \ConstDM_{J,2})) \eqsp.
    \end{align}
    Particularly, we obtain 
    \begin{align}
    \label{eq:contr_cost_Y}
        \PEcoupling{y}{\tilde{y}}^{1/p}[\costw(Y_t, \tilde{Y_t})^p] \leq \Auxconst_{W,3}^{(1)} p^2 \taumix^{3/2} \rho_{1,\alpha}^{t/p} \sqrt{\log{(1/\alpha a)}} \costw(y, \tilde{y}) \eqsp.
    \end{align}
\end{proof}

\subsection{Proof of \Cref{lem:prop_inv_distribution}}
\label{appendix:inv_dist_J1_exist}
\begin{proof}
For any $Y = (z, \Jnalpha{}{0}, \Jnalpha{}{1}), \tilde{Y} = (\tilde{z}, \Jnalphat{}{0}, \Jnalphat{}{1})$, where $\Jnalpha{}{0}, \Jnalpha{}{1},\Jnalphat{}{0}, \Jnalphat{}{1} \in \rset^d$ and $z,\tilde{z} \in \Zset$, we consider the metric
    \begin{equation}
        d_J(Y,\tilde{Y}) = \norm{\Jnalpha{}{0} - \Jnalphat{}{0}} + \norm{\Jnalpha{}{1} - \Jnalphat{}{1}} + \sqrt{\alpha a}\supconsteps\indiacc{z\neq \tilde{z}} \eqsp.
    \end{equation}
This metric is upper bounded by the cost function, defined in \eqref{eq:cost_func_def_J1}, that is, $d_J \leq \costw$. Applying \cite[Theorem~20.3.4]{douc:moulines:priouret:soulier:2018} together with \Cref{lem:contr_wasser_J1}, we get the result.
\end{proof}

\subsection{Proof of \Cref{prop:Jnalpha_asymp_exp}}
\begin{proof}
We define the random variable $\Jnalpha{\infty}{1,\alpha}$ with distribution $\Pi_{J^{(1)},\alpha}$. Then, from \Cref{lem:contr_wasser_J1} it follows that $\lim\limits_{t\to \infty} \PE[\Jnalpha{t}{1}] = \PE[\Jnalpha{\infty}{1}]$. We omit the parameter $\alpha$ in the notation for the sake of simplicity. However, we note that the limiting random variable depends on the parameter $\alpha$. Thus, using \eqref{eq:jn_allexpansion_main}, we get
\begin{align}
    \PE[\Jnalpha{\infty+1}{1}] = \PE[\Jnalpha{\infty}{1}] - \alpha \bA \PE[\Jnalpha{\infty}{1}] - \alpha \PE[\zmfuncA{Z_{\infty+1}}\Jnalpha{\infty}{0}] \eqsp,
\end{align}
which is equivalent to
\begin{align}
    \bA \PE[\Jnalpha{\infty}{1}] &= -\PE[\zmfuncA{Z_{\infty+1}}\Jnalpha{\infty}{0}] = \alpha \PE\left[\sum_{k=1}^{\infty} \zmfuncA{Z_{\infty+1}} (\Id - \alpha \bA)^{k-1} \funcnoise{Z_{\infty-k+1}}\right]\\
    &= \alpha \sum_{k=1}^{\infty} \PE[\zmfuncA{Z_{\infty+1}} \funcnoise{Z_{\infty-k+1}}] + \sum_{j=1}^{\infty}(-1)^j \alpha^{j+1}\sum_{k=j+1}^{\infty} \binom{k-1}{j} \PE\left[\zmfuncA{Z_{\infty+1}} \bA^j \funcnoise{Z_{\infty-k+1}}\right] \eqsp.
\end{align}
For any $t \geq 0$, we define the $\sigma$-algebra $\mathcal{F}_{t}^{-} = \sigma(Z_{\infty-t}, Z_{\infty -t - 1}, \dots)$. Note that 
\begin{align}
    \PE[\zmfuncA{Z_{\infty+1}} \funcnoise{Z_{\infty-k+1}}] = \PE[\PE[\zmfuncA{Z_{\infty+1}} | \mathcal{F}_{\infty-k+1}^{-}]\funcnoise{Z_{\infty-k+1}}] = \PE[\MKQ^{k}\zmfuncA{Z_{\infty}} \funcnoise{Z_{\infty}}] \eqsp.
\end{align}
Therefore, we have
\begin{align}
    \PE[\Jnalpha{\infty}{1}] = \alpha \Delta + R(\alpha) \eqsp,
\end{align}
where we denote
\begin{align}
    &\Delta = \bA^{-1} \sum_{k=1}^\infty \PE[\{\MKQ^{k}\zmfuncA{Z_{\infty}}\} \funcnoise{Z_{\infty}}],\\
    &R(\alpha) = \bA^{-1}\sum_{j=1}^\infty (-1)^j \alpha^{j+1} \sum_{k=j+1}^\infty \binom{k-1}{j} \PE_{\pi}[\{\MKQ^k\zmfuncA{Z_{\infty}}\}\bA^j \funcnoise{Z_{\infty}}] \eqsp.
\end{align}
Now, we will prove that this decomposition is well defined. Setting $v_j(z) = \bA^j \funcnoise{z}$, we obtain
\begin{align}
    \norm{\PE_{\pi}[\{\MKQ^k \zmfuncA{Z_{\infty}}\}\bA^j \funcnoise{Z_{\infty}}]} = \sup_{u \in \mathbb{S}^{d-1}} |\int_{\Zset} u^T\MKQ^{k}\zmfuncA{z} v(z) \pi(dz)| &= \sup_{u \in \mathbb{S}^{d-1}} |\int_{\Zset} u^T(\MKQ^{k}\funcA{z} - \bA)v(z) \pi(dz) |\\
    &\leq \bConst{A}^{j+1} \supconsteps \dobru{\MKQ^{k}} \eqsp,
\end{align}
where we set $\bA^0 = \Id$ and $\dobru{\MKQ^{k}}$ is the Dobrushin's coefficient. Therefore, using \eqref{eq:crr_koef_sum_tau_mix}, we have
\begin{align}
    \sum_{k=1}^{\infty} \norm{\PE_{\pi}[\{\MKQ^{k}\zmfuncA{Z_{\infty}}\} \funcnoise{Z_{\infty}}} \leq (4/3)\taumix \eqsp.
\end{align}
Setting $q = (1/4)^{1/\taumix}$ and using \eqref{eq:drift-condition}, we get
\begin{align}
    \sum_{j=1}^\infty \alpha^{j+1} \sum_{k=j+1}^{\infty} \binom{k-1}{j} \norm{\PE_{\pi}[\{\MKQ^k \zmfuncA{Z_{\infty}}\}\bA^j \funcnoise{Z_{\infty}}]} &\leq \supconsteps \sum_{j=1}^{\infty} {\bConst{A}^{j+1}\alpha^{j+1} \over j!} \sum_{k=j+1}^{\infty} {(k-1)! \over (k-j-1)!}(1/4)^{\lfloor k/\taumix \rfloor}\\
    &\leq 4 \supconsteps \sum_{j=1}^{\infty} {\bConst{A}^{j+1}\alpha^{j+1} \over j!} {q^{j+1} j! \over (1-q)^{j+1}} \\
    &\leq  4\supconsteps \sum_{j=1}^{\infty} (\alpha \bConst{A}\taumix)^{j+1}\\
    &\leq 4\alpha^2 \bConst{A}^2\taumix^2 \supconsteps + 8\alpha^3 \bConst{A}^3 \taumix^3 \supconsteps \eqsp,
\end{align}
where we used that $q/(1-q) \leq \taumix$, which concludes the proof.
\end{proof}

\begin{proposition}
\label{prop:bias_decomp_J2}
    Assume \Cref{assum:A-b}, \Cref{assum:noise-level} and \Cref{assum:drift}. Then for $\alpha \in (0, \alpha_{1, \infty}^{(\sf b)})$, it holds that 
    \begin{equation}
        \lim_{n\to \infty} \PE[\Jnalpha{n}{2, \alpha}] = \PE[\Jnalpha{\infty}{2, \alpha}] = \alpha^2 \Delta_2 + R_2(\alpha)\eqsp,
    \end{equation}
    where $\Delta_2 \in \rset^{d}$ is defined as
    \begin{equation}
        \Delta_2 = -\sum_{k=1}^{\infty} \sum_{i=0}^{\infty} \PE[\zmfuncA{Z_{\infty+k+i+1}}\zmfuncA{Z_{\infty+i+1}}\funcnoise{Z_{\infty}}] \eqsp,
    \end{equation}
    and $R_2(\alpha)$ is a reminder term which can be bounded as
    \begin{equation}
        \norm{R_2(\alpha)} \leq \ConstD_b\taumix^{4} \alpha^{5/2} \supconsteps \eqsp,
    \end{equation}
    where we define
    \begin{equation}
    \label{eq:constD_b_def}
        \ConstD_b = \bConst{A}^3 (12\ConstD_1 \bConst{A} a^{1/2} + 24(e^{2/\taumix} - 1)) \eqsp.
    \end{equation}
\end{proposition}
\begin{proof}
    Firstly, we introduce the random variable $\Jnalpha{\infty}{2}$ with distribution $\Pi_{J^{(2)},\alpha}$. We again omit the parameter $\alpha$ in the notation for the sake of simplicity. However, we note that the distribution of $\Jnalpha{\infty}{2}$ depends on the parameter $\alpha$. Using the recursion for $\Jnalpha{n}{2}$ from \eqref{eq:jn_allexpansion_main}, we have
    \begin{equation}
        \PE[\Jnalpha{\infty + 1}{2}] = \PE[\Jnalpha{\infty}{2}] - \alpha \bA \PE[\Jnalpha{\infty}{2}] - \alpha \PE[\zmfuncA{Z_{\infty+1}}\Jnalpha{\infty}{1}] \eqsp,
    \end{equation}
    which in turn, using the recursion for $\Jnalpha{n}{1}$,  leads to
    \begin{align}
    \label{eq:bias_Jn2_decomp_1}
        \bA \PE[\Jnalpha{\infty}{2}] &= -\PE[\zmfuncA{Z_{\infty+1}}\Jnalpha{\infty}{1}] = \alpha \PE\left[\sum_{k=1}^{\infty}\zmfuncA{Z_{\infty+1}}(\Id-\alpha \bA)^{k-1}\zmfuncA{Z_{\infty-k+1}}\Jnalpha{\infty-k}{0}\right]\\
        &= \alpha \sum_{k=1}^{\infty}\PE[\zmfuncA{Z_{\infty+1}}\zmfuncA{Z_{\infty-k+1}}\Jnalpha{\infty-k}{0}] + \sum_{j=1}^{\infty}(-1)^j \alpha^{j+1}\sum_{k=j+1}^{\infty}\binom{k-1}{j}\PE[\zmfuncA{Z_{\infty+1}}\bA^j \zmfuncA{Z_{\infty-k+1}}\Jnalpha{\infty-k}{0}]\\
        &= T_{b,1} + T_{b,2}\eqsp.
    \end{align}
    We can further decompose the first term, that is,
    \begin{align}
        T_{b,1} &= -\alpha^2 \sum_{k=1}^{\infty}\sum_{i=0}^{\infty} \PE[\zmfuncA{Z_{\infty+1}}\zmfuncA{Z_{\infty-k+1}}(\Id-\alpha \bA)^i \funcnoise{Z_{\infty-k-i}}]\\
        &= -\alpha^2 \sum_{k=1}^{\infty} \sum_{i=0}^{\infty} \PE[\zmfuncA{Z_{\infty+1}}\zmfuncA{Z_{\infty-k+1}}\funcnoise{Z_{\infty-k-i}}] \\
        &- \sum_{k=1}^{\infty} \sum_{j=1}^{\infty} (-1)^j \alpha^{j+2} \sum_{i=j}^{\infty} \binom{i}{j} \PE[\zmfuncA{Z_{\infty+1}}\zmfuncA{Z_{\infty-k+1}} \bA^j \funcnoise{Z_{\infty-k-i}}] = T_{b, 11} + T_{b, 12}
    \end{align}
    For any $t \geq 0$, we define the $\sigma$-algebra $\mathcal{F}_{t}^{-} = \sigma(Z_{\infty-t}, Z_{\infty -t - 1}, \dots)$. For $T_{b,11}$, denoting $u_k(z) = \{\MKQ^k \zmfuncA{z}\}\zmfuncA{z}$ for $z \in \Zset$, we get
    \begin{align}
        \PE[\zmfuncA{Z_{\infty+1}}\zmfuncA{Z_{\infty-k+1}}\funcnoise{Z_{\infty-k-i}}] &= \PE[\PE[\zmfuncA{Z_{\infty+1}} | \mathcal{F}_{\infty-k+1}^{-}] \zmfuncA{Z_{\infty-k+1}}\funcnoise{Z_{\infty-k-i}}]\\
        &= \PE[\{\MKQ^k\zmfuncA{Z_{\infty-k+1}}\} \zmfuncA{Z_{\infty-k+1}}\funcnoise{Z_{\infty-k-i}}]\\
        &= \PE[\{\MKQ^{i+1}u_k(Z_{\infty})\}\funcnoise{Z_{\infty}}] = \PE[\{\MKQ^{i+1}\bar{u}_k(Z_{\infty})\}\funcnoise{Z_{\infty}}] \eqsp,
    \end{align}
    where we set $\bar{u}_k(z) = u_k(z) - \PE[u_k(Z_{\infty})]$. Note that for any $z \in \Zset$, we have $\norm{u_k(z)} \leq \bConst{A}^2 \dobru{\MKQ^k}$. Then, using Minkowski's inequality and \eqref{eq:crr_koef_sum_tau_mix}, we can bound the first term, as
    \begin{align}
        \norm{T_{b,11}} \leq 2 \bConst{A}^2 \alpha^2 \sum_{k=1}^{\infty}\sum_{i=0}^{\infty} \dobru{\MKQ^{i+1}}\dobru{\MKQ^k}\supconsteps \leq 8 \bConst{A}^2 \taumix^2 \alpha^2 \supconsteps \eqsp.
    \end{align}
    Similarly, for the term $T_{b,12}$, we have
    \begin{align}
        \PE[\zmfuncA{Z_{\infty+1}}\zmfuncA{Z_{\infty-k+1}} \bA^j \funcnoise{Z_{\infty-k-i}}] &= \PE[\{\MKQ^k \zmfuncA{Z_{\infty-k+1}}\}\zmfuncA{Z_{\infty-k+1}}\bA^j\funcnoise{Z_{\infty-k-i}}]\\
        &= \PE[\{\MKQ^{i+1} \bar{v}_{k,j}(Z_{\infty})\} \funcnoise{Z_{\infty}}] \eqsp,
    \end{align}
    where we define $v_{k,j}(z) = \{\MKQ^k\zmfuncA{z}\}\zmfuncA{z}\bA^j$ and $\bar{v}_{k,j}(z) = v_{k,j}(z) - \PE[v_{k,j}(Z_{\infty})]$. Thus, using the bound
    \begin{align}
        \norm{\PE[\{\MKQ^{i+1} \bar{v}_{k,j}(Z_{\infty})\} \funcnoise{Z_{\infty}}]} \leq \bConst{A}^{j+2}\dobru{\MKQ^{i+1}}\dobru{\MKQ^k} \supconsteps \eqsp,
    \end{align}
    and setting $q=(1/4)^{1/\taumix}$, we get
    \begin{align}
        \norm{T_{b,12}} &\leq \supconsteps \sum_{k=1}^{\infty}\sum_{j=1}^{\infty} {\bConst{A}^{j+2}\alpha^{j+2}\over j!} \dobru{\MKQ^k} \sum_{i=j}^{\infty} {i!\over (i-j)!} \dobru{\MKQ^{i+1}} \leq 4 {1-q\over q}\supconsteps \sum_{k=1}^{\infty} \dobru{\MKQ^k} \sum_{j=1}^{\infty} {\bConst{A}^{j+2}\alpha^{j+2}\over j!}  {q^{j+2} j! \over (1-q)^{j+2}}\\
        &\leq 8 (e^{2/\taumix} - 1)\taumix \supconsteps \sum_{j=1}^{\infty} (\bConst{A}\taumix\alpha)^{j+2} \leq 24 \bConst{A}^3 (e^{2/\taumix} - 1)\taumix^4 \alpha^3 \supconsteps \eqsp.
    \end{align}
    Now, we proceed with bounding the term $T_{b,2}$. Note that
    \begin{align}
        \PE[\zmfuncA{Z_{\infty+1}}\bA^j \zmfuncA{Z_{\infty-k+1}}\Jnalpha{\infty-k}{0}] = \PE[\{\MKQ^k \zmfuncA{Z_{\infty-k+1}}\}\bA^j\zmfuncA{Z_{\infty-k+1}}\Jnalpha{\infty-k}{0}] \eqsp,
    \end{align}
    Thus, applying \Cref{lem:stationary_rosenthal_jn_0}, for any $j \geq 1$ and $k \geq j+1$, we get
    \begin{align}
        \norm{\PE[\zmfuncA{Z_{\infty+1}}\bA^j \zmfuncA{Z_{\infty-k+1}}\Jnalpha{\infty-k}{0}]} \leq \bConst{A}^{j+2} \dobru{\MKQ^k} \PE[\norm{\Jnalpha{\infty}{0}}] \leq \ConstD_1 \bConst{A}^{j+2}\dobru{\MKQ^k} \sqrt{\alpha a \taumix}\supconsteps \eqsp.
    \end{align}
    Applying this result combined with Minkowski's inequality to $T_{b,2}$, we get
    \begin{align}
        \norm{T_{b,2}} &\leq \ConstD_1 \bConst{A}^2 \sqrt{\alpha a \taumix} \supconsteps\sum_{j=1}^{\infty} {\bConst{A}^{j+1}\alpha^{j+1}\over j!} \sum_{k=j+1}^{\infty} {(k-1)! \over (k-j-1)!} \dobru{\MKQ^k}\\
        &\leq 4 \ConstD_1 \bConst{A}^2 \sqrt{\alpha a \taumix} \supconsteps \sum_{j=1}^{\infty} (\bConst{A}\taumix\alpha)^{j+1} \leq 12 \ConstD_1 \bConst{A}^4 a^{1/2} \taumix^{5/2} \alpha^{5/2} \supconsteps \eqsp.
    \end{align}
\end{proof}

\subsection{Proof of \Cref{prop:bias_asymp_exp}}
\label{appendix:bias_asymp_exp}
\begin{proof}
    Using \eqref{eq:tr_fl_decomp} and \eqref{eq:fluct_decomp}, we get
    \begin{equation}\label{eq:expectation_bias_decomp}
        \PE[\thalpha{n}{\alpha}] - \thetalim = \PE[\utheta_{n}] + \PE[\Jnalpha{n}{0,\alpha}] + \PE[\Jnalpha{n}{1,\alpha}] + \PE[\Jnalpha{n}{2,\alpha}] + \PE[\Hnalpha{n}{2,\alpha}] \eqsp.
    \end{equation}
    Using \Cref{lem:cost_function_decrease} and \cite[Theorem 6.9]{villani:2009}, we get that $\lim\limits_{n \to \infty}\PE[\theta_n] = \Pi_{\alpha}(\theta_0)$. Similarly, from \Cref{lem:contr_wasser_J1} it follows that $\lim\limits_{n\to \infty}\wasserdist_{p}(\mathcal{L}(\Jnalpha{n}{1,\alpha}), \mathcal{L}(\Jnalpha{\infty}{1,\alpha})) = 0$, hence $\lim\limits_{n\to \infty}\PE[\Jnalpha{n}{1,\alpha}] = \PE[\Jnalpha{\infty}{1,\alpha}]$. Due to \cite[Proposition 7]{durmus2025finite} the term $\PE[\utheta_n]$ tends to $0$ geometrically fast. Since $\Jnalpha{n}{0,\alpha}$ is the linear statistics of $\{\funcnoise{Z_k}\}$, using \Cref{assum:drift} we get that $\lim_{n\to \infty} \PE[\Jnalpha{n}{0,\alpha}] = 0$. Now, we can rewrite the equation \eqref{eq:expectation_bias_decomp} as
    \begin{equation}
        \PE[\thalpha{n}{\alpha}] - \thetalim - \PE[\utheta_{n}] - \PE[\Jnalpha{n}{0,\alpha}] - \PE[\Jnalpha{n}{1,\alpha}] = \PE[\Jnalpha{n}{2,\alpha}] + \PE[\Hnalpha{n}{2,\alpha}] \eqsp.
    \end{equation}
    From the arguments above, it follows that the left-hand side of this equation converges, hence, the right-hand side converges as well. Applying \Cref{prop:Jnalpha_asymp_exp} and using \Cref{prop:bound_Jn2_alpha}, \Cref{prop:bound_Hn2}, we get the result.
\end{proof}

\section{Rosenthal-type inequality}
\label{appendix:rosenthal_Jnalpha}
\par
We begin with the preliminary fact on the boundness of iterations $\{\Jnalpha{n}{0,\alpha}\}$.
\begin{lemma}
\label{lem:stationary_rosenthal_jn_0}
Assume \Cref{assum:A-b}, \Cref{assum:noise-level} and \Cref{assum:drift}. Let $p \geq 2$. Then, for any $\alpha \in (0;\alpha_{\infty})$, where $\alpha_{\infty}$ is defined in \eqref{eq: kappa_def}, initial probability distribution $\xi$ on $(\Zset,\Zsigma)$, $n \in \nset$, it holds that
\begin{equation}
\label{eq:J_n_0_Markov_UGE_stationary}
\PE_{\xi}^{1/p}\left[\norm{\Jnalpha{n}{0,\alpha}}^{p}\right] \leq \ConstD_{1} \sqrt{\alpha a p \taumix} \supconsteps \eqsp,
\end{equation}
where $\ConstD_{1}$ is defined as
\begin{equation}
\label{eq:const_D_M_1_def}
\ConstD_{1} = 2^{7/2}\qcond^{1/2} a^{-1} \{e^{-1/4} + \sqrt{2\pi e}\bConst{A} a^{-1}\}\eqsp.
\end{equation}
\end{lemma}
\begin{proof}
See \cite[Proposition 8]{durmus2025finite}.
\end{proof}

In this section we consider a Markov Chain $((\Jnalpha{t}{0, \alpha}, Z_{t+1}), t \geq 0)$ with a transition kernel $\MKQ_J$ and a function $\psi(J,z) =\zmfuncA{z}J$. In what follows, the Markov kernel $\MKQ_J$ admits unique stationary distribution which we denote $\Pi_{J,\alpha}$. Also, for any $t \geq 0$ we denote 
\begin{align}
    \bar{\psi}(J, z) = \psi(J, z) - \PE_{\Pi_{J,\alpha}}[\psi_0], \quad \psi_t = \psi(\Jnalpha{t}{0,\alpha}, Z_{t+1}), \quad \bar{\psi}_t = \bar{\psi}(\Jnalpha{t}{0,\alpha}, Z_{t+1}) \eqsp,
\end{align}
We define the cost function $\costw[J]: \rset^d \times \Zset \times \rset^d \times \Zset \to \rset_{+}$ as
\begin{equation}
    \label{eq:const_func_def}
    \costw[J]((J, z), (\tilde{J}, \tilde{z})) = \norm{J - \tilde{J}} + (\norm{J} + \norm{\tilde{J}} + \sqrt{\alpha a} \supconsteps)\indiacc{z \neq \tilde{z}} \eqsp.
\end{equation}
For this cost function, we get
\begin{equation}\label{eq:lip_c_psi}
    \norm{\psi(J, z) - \psi(\tilde{J}, \tilde{z})} \leq 2\bConst{A} \costw[J]((J, z), (\tilde{J}, \tilde{z})) \eqsp.
\end{equation}
Before the main result of this section, formulated in \Cref{prop:rosenthal_Jn0}, we state additional lemmas. In the following results we use the notation for pairs $(J, z)$ where $J  \in \rset^d$ and $z \in \Zset$. We denote the $n$-th step transition of our Markov Chain $Y_n = (J_n, Z_{n+1})$ starting from some distribution $\xi$ and $\tilde{Y}_n = (\Jnalphat{n}{0,\alpha}, \tilde{Z}_{n+1})$ from distribution $\tilde{\xi}$.
Now, we prove the result on contraction of Wasserstein distance, which, in sequel, will give the existence of invariant measure.

\begin{lemma}\label{lem:contr_wasser}
    Assume \Cref{assum:A-b}, \Cref{assum:noise-level}, and \Cref{assum:drift}. Fix $J, \tilde{J} \in \rset^d$ and $z, \tilde{z} \in \Zset$. Denote pairs $y = (J, z)$ and $\tilde{y} = (\tilde{J}, \tilde{z})$ such that $y \neq \tilde{y}$. Then, for any $n,p \geq 1$ and $\alpha \in (0, \alpha_{\infty}^{(\Markov)} \wedge (ap)^{-1}\ln{\rho^{-1}})$, we have 
    \begin{align}
         \wasserdist_{\costw[J],p}^{1/p}(\delta_y \MKQ_J^n, \delta_{\tilde{y}}\MKQ_J^n) \leq \Auxconst_{W,3}^{(1)} p^2 \taumix^{3/2} \rho_{1,\alpha}^{n/p} \sqrt{\log{(1/\alpha a)}} \costw[J](y, \tilde{y}) \eqsp,
    \end{align}
    where $\Auxconst_{W,3}^{(1)}$ is defined in \eqref{eq:def_cW3_1} and $\rho_{1,\alpha} = e^{-\alpha a / 12}$.
\end{lemma}

\begin{proof}
    Consider $u=(z, J, J^{(1)})$ and $\tilde{u} = (\tilde{z}, \tilde{J}, \tilde{J}^{(1)})$, where $J^{(1)}, \tilde{J}^{(1)} \in \rset$. Note that
    \begin{equation}
    \label{eq:cost_funcs_ineq}
        \costw[J](y, \tilde{y}) \leq \costw(u, \tilde{u}) \eqsp,
    \end{equation}
    where $\costw(\cdot,\cdot)$ is defined in \eqref{eq:cost_func_def_J1}. Let $\mu \in \Pi(\delta_u \MKQ_{J^{(1)}}^n, \delta_{\tilde{u}}\MKQ_{J^{(1)}}^n)$ be an arbitrary coupling. We can match it with some coupling $\nu \in \Pi(\delta_y \MKQ_{J}^n, \delta_{\tilde{y}}\MKQ_{J}^n)$ such that for any $A,A'\in \borel{\rset^d}$ and $C,C' \in \Zsigma$, we have 
    \begin{equation}
        \nu(A\times C, A'\times C') = \mu(C \times A\times \rset^d,  C'\times A' \times \rset^d) \eqsp.
    \end{equation}
    Hence, taking expectation on both sides of \eqref{eq:cost_funcs_ineq}, we get
    \begin{align}
    \label{eq:copling_projection}
        \PE_{\nu}[\costw[J](Y, \tilde{Y})^p] = \PE_{\mu}[\costw[J](Y,\tilde{Y})^p] \leq \PE_{\mu}[\costw(U, \tilde{U})^p] \eqsp.
    \end{align}
    Therefore, it follows that
    \begin{equation}
        \wasserdist_{\costw[J],p}^{1/p}(\delta_y \MKQ_J^n, \delta_{\tilde{y}}\MKQ_J^n) \leq \wasserdist_{\costw,p}^{1/p}(\delta_u \MKQ_{J^{(1)}}^n, \delta_{\tilde{u}}\MKQ_{J^{(1)}}^n) \eqsp.
    \end{equation}
    To conclude the proof, we set $J^{(1)} = \tilde{J}^{(1)} = 0$, so that $\costw[J](y, \tilde{y}) = \costw(u, \tilde{u})$, and apply \Cref{lem:contr_wasser_J1}.
\end{proof}

\begin{corollary}
\label{cor:cor_inv_distr_J0}
Assume \Cref{assum:A-b}, \Cref{assum:noise-level} and \Cref{assum:drift}. Let $\alpha \in (0,\alpha_{1, \infty}^{(\sf b)})$. Then the process $\{Y_t\}_{t \in \nset}$ is a Markov chain with a unique stationary distribution $\Pi_{J, \alpha}$.
\end{corollary}
\begin{proof}
    We can apply the similar arguments as in \Cref{lem:prop_inv_distribution}, but with \Cref{lem:contr_wasser} instead of \Cref{lem:contr_wasser_J1}.
\end{proof}

\begin{proposition}
\label{prop:rosenthal_Jn0}    

Assume \Cref{assum:A-b}, \Cref{assum:noise-level}, and \Cref{assum:drift}. We set step size $\alpha \in (0, \alpha_{1,\infty}^{(\sf b)} \wedge (ap)^{-1}\ln\rho^{-1})$. Then, for $p \geq 2$ and $n \geq 1$, it holds that
     \begin{align}
         \PE_{\Pi_{J,\alpha}}^{1/p}[\norm{\sum_{t=0}^{n-1}\bar{\psi}_t}^p] \leq  \Auxconst_\psi p^{3/2} \taumix^2  (\alpha n)^{1/2} \eqsp,
     \end{align}
where
    \begin{align}
        \Auxconst_\psi = 32(1 + \ConstD_1)\bConst{A}\qcond^{1/2}(1+a)^{1/2}\supconsteps \eqsp.
    \end{align}
\end{proposition}
\begin{proof}
    We start with the Poisson's equation which corresponds to $\bar{\psi}$. That is, we have the unique (up to additive $\Pi_{J,
    \alpha}$-a.s. constant vector) solution $U_\alpha : \rset^d \times \Zset \to \rset^d$ to the following equation
    \begin{align}
    \label{eq:poisson_Ualpha}
        U_\alpha(J, z) - \MKQ_J U_\alpha(J, z) = \bar{\psi}(J, z) \eqsp,
    \end{align}
    where $U_\alpha(J, z) = \sum\limits_{k \geq 0} \MKQ_J^k \bar{\psi}(J, z)$. Now, in order to bound $U_\alpha$, we note that
    \begin{align}
        \MKQ_J^k \psi(J, z) = \PE[\psi_k \mid J_0=J, Z_1=z] &= \MKQ^k \zmfuncA{z}B^k J - \alpha \sum_{j=1}^k \PE[\zmfuncA{Z_{k+1}}B^{k-j}\funcnoise{Z_j} \mid Z_1=z]\\
        &= \MKQ^k \zmfuncA{z}B^k J - \alpha\sum_{j=1}^k \PE[\MKQ^{k-j+1}\zmfuncA{Z_j}B^{k-j}\funcnoise{Z_j} \mid Z_1=z] \eqsp,
    \end{align}
    where we have set $B = \Id-\alpha \bA$. Hence, denoting a random vector $(J_\infty, Z) \sim \Pi_{J,\alpha}$, we get
    \begin{align}
        U_\alpha(J, z) = \sum_{k \geq 0} \MKQ_J^k \bar{\psi}(J, z) &= \sum_{k \geq 0} \left\{\MKQ^k \zmfuncA{z}B^k J - \PE_{\Pi_{J,\alpha}}[\MKQ^k\zmfuncA{Z}B^k J_\infty] \right\} \\
        &- \alpha \sum_{k \geq 0}\sum_{r=0}^{k-1} \left\{ \PE[\MKQ^{r+1} \zmfuncA{Z_{k-r}}B^r \funcnoise{Z_{k-r}} \mid Z_1=z] - \PE_{\Pi_{J,\alpha}}[\MKQ^{r+1} \zmfuncA{Z_{k-r}}B^r \funcnoise{Z_{k-r}}] \right\}\\
        &=\sum_{k \geq 0}\left\{ \MKQ^k \zmfuncA{z}B^k J - \PE_{\Pi_{J,\alpha}}[\MKQ^k \zmfuncA{Z}B^k J_\infty] \right\} - \alpha \sum_{k \geq 0}\sum_{r=0}^{k-1} \left\{ \MKQ^{k-r-1}\bar{v}_r(z) \right\}\\
        &=\sum_{k \geq 0}\left\{ \MKQ^k \zmfuncA{z}B^k J - \PE_{\Pi_{J,\alpha}}[\MKQ^k \zmfuncA{Z}B^k J_\infty] \right\} - \alpha \sum_{r \geq 0}\sum_{m \geq 0} \left\{ \MKQ^{m}\bar{v}_r(z) \right\} = U_{1,\alpha} + U_{2,\alpha} \eqsp,
    \end{align}
    where we have set
    \begin{align}
        v_r(z) = \MKQ^{r+1}\zmfuncA{z}B^r \funcnoise{z}\eqsp,\quad  \bar{v}_r(z) = v_r(z) - \PE_\pi[v_r(Z)] \eqsp.
    \end{align}
    To bound the first term, we note that 
    \begin{align}
        \PE_{\Pi_{J,\alpha}}^{1/p}[\norm{\MKQ^k \zmfuncA{z}B^k J}^p] \leq \PE_{\Pi_{J,\alpha}}^{1/p}[\norm{\MKQ^k \zmfuncA{z}}^p \norm{B^k J}^p] \leq \|\MKQ^k \tilde{\funcAw}\|_\infty \PE_{\Pi_{J,\alpha}}^{1/p}[\norm{B^k J}^p] \eqsp.
    \end{align}
    Here, the first term is controlled by the \Cref{assum:drift}, $\|\MKQ^k\tilde{\funcAw}\|_\infty \leq 2\bConst{A}\dobru{\MKQ^k}$. Using that $J_k \xrightarrow{d} J_\infty$ and that $\|\cdot\|^p$ is non-negative continuous function, we have 
    \begin{align}
        \PE^{1/p}[\norm{J_\infty}^p] \leq \lim_k\inf \PE^{1/p}[\norm{J_k}^p] \eqsp.
    \end{align}
    Therefore, applying \Cref{lem:stationary_rosenthal_jn_0}, we obtain
    \begin{align}
        \PE_{\Pi_{J,\alpha}}^{1/p}[\norm{\MKQ^k \zmfuncA{z}B^k J}^p] \leq 2\ConstD_1\bConst{A}\qcond^{1/2} \sqrt{\alpha a p \taumix} \dobru{\MKQ^k} (1 -\alpha a)^{k/2} \supconsteps \eqsp.
    \end{align}
    Thus, applying Minkowski's inequality, we get
    \begin{align}
    \label{eq:U_1_alpha_bound}
        \PE_{\Pi_{J,\alpha}}^{1/p}[\norm{U_{1,\alpha}}^p] \leq 2\sum_{k \geq 0} \PE_{\Pi_{J,\alpha}}^{1/p}[\norm{\MKQ^k \zmfuncA{z}B^k J}^p] &\leq 4\ConstD_1 \bConst{A} \qcond^{1/2}(ap \taumix)^{1/2} \alpha^{1/2} \supconsteps\sum_{k \geq 0} \dobru{\MKQ^k} \\
        &\leq 8 \ConstD_1 \bConst{A} \qcond^{1/2}(ap)^{1/2}\taumix^{3/2} \alpha^{1/2} \supconsteps \eqsp.
    \end{align}
    For the second term, we note that
    \begin{align}
        \|v_r\|_\infty \leq \|\MKQ^{r+1}\tilde{\funcAw}\|_\infty \|B^r\| \supconsteps \leq 2\bConst{A}\qcond^{1/2}\dobru{\MKQ^{r+1}} (1-\alpha a)^{r/2} \supconsteps \eqsp.
    \end{align}
    Therefore, again applying Minkowski's inequality, we get
    \begin{align}
    \label{eq:U_2_alpha_bound}
        \PE_{\Pi_{J,\alpha}}^{1/p}[\norm{U_{2,\alpha}}^p] \leq 2\alpha \sum_{r \geq 0} \sum_{m \geq 0} \dobru{\MKQ^m} \|v_r\|_\infty \leq 8\bConst{A} \qcond^{1/2} \taumix^2 \alpha  \supconsteps \eqsp.
    \end{align}
    Now, combining \eqref{eq:U_1_alpha_bound} and \eqref{eq:U_2_alpha_bound}, we get
    \begin{align}
        \PE_{\Pi_{J,\alpha}}^{1/p}[\norm{U_\alpha(J_\infty,Z)}^p] \leq 8 (1 + \ConstD_1)\bConst{A}\qcond^{1/2}(1+a)^{1/2}p^{1/2}\taumix^2 \alpha^{1/2} \supconsteps \eqsp.
    \end{align}
    Therefore, applying \eqref{eq:poisson_Ualpha}, we get
    \begin{align}
        \sum_{t=0}^{n-1}\bar{\psi}_t = U_{\alpha}(Y_0) - U_\alpha(Y_n) + \sum_{t=0}^{n-1} \left\{ U_\alpha(Y_{t+1}) - \MKQ_{J} U_\alpha(Y_t) \right\} \eqsp.
    \end{align}
    The last term here is the sum of martingale-difference sequence $\{U_\alpha(Y_{t+1}) - \MKQ_{J} U_\alpha(Y_t)\}_{t=0}^{n-1}$, and thus, applying Burkholder's inequality, we obtain
    \begin{align}
        \PE_{\Pi_{J,\alpha}}^{1/p}[\norm{\sum_{t=0}^{n-1}\{ U_\alpha(Y_{t+1}) - \MKQ_{J} U_\alpha(Y_t) \}}^p] &\leq 2p\left( \sum_{t=0}^{n-1} \PE_{\Pi_{J,\alpha}}^{2/p}[\norm{U_\alpha(Y_t)}^p] \right)^{1/2} \\
        &\leq 16(1 + \ConstD_1)\bConst{A}\qcond^{1/2}(1+a)^{1/2} p^{3/2} \taumix^2  (\alpha n)^{1/2} \supconsteps \eqsp.
    \end{align}
\end{proof}

\begin{corollary}
\label{cor:rosenthal_J_nonstat}
     Assume \Cref{assum:A-b}, \Cref{assum:noise-level}, and \Cref{assum:drift}. Then for any probability measure $\xi$ on $ \Zset$, setting $J_0 = 0$, $Z_1 \sim \xi$ and $\alpha \in (0, \alpha_{1,\infty}^{(\sf b)} \wedge (ap)^{-1}\log \rho^{-1})$, we get:
     \begin{align}
        \label{eq:rosenthal_J_nonstat}
        \PE_{\xi}^{1/p}[\norm{\sum_{t=0}^{n-1} \bar{\psi}_t}^p] \leq \Auxconst_{W,1}^{(2)} p^{3/2} (\alpha n)^{1/2} + \Auxconst_{W,2}^{(2)} p^{7/2} \alpha^{-1/2} \sqrt{\log{(1/\alpha a)}} \eqsp,
     \end{align}
     where
     \begin{align}
     \label{eq:const_def_rosent_nonstat}
        \Auxconst_{W,1}^{(2)} &= \Auxconst_\psi \taumix^2 \eqsp,\\
        \Auxconst_{W,2}^{(2)} &= 48\bConst{A} (2\ConstD_1 + 1)(1 + \Auxconst_{W,3}^{(1)})\taumix^2 a^{-1/2} \supconsteps \eqsp.
    \end{align}
\end{corollary}
\begin{proof}
    We use the exact coupling construction described in \Cref{appendix:proof_existence_pi_alpha}. It naturally induces the coupling on the processes $Y_t$ and $\tilde{Y}_t$. Using Minkowski's inequality, we have
    \begin{align}
        \label{eq:decomp_sum_nonstat}
        \PE_{\xi}^{1/p}[\norm{\sum_{t=0}^{n-1} \bar{\psi}_t}^p] \leq \PE_{\Pi_{J,\alpha}}^{1/p}[\norm{\sum_{t=0}^{n-1} \bar{\psi}_t}^p] + \left( \PEcoupling{\xi}{\Pi_{J,\alpha}}[\norm{\sum_{t=0}^{n-1} \{\psi(Y_t) - \psi(\tilde{Y}_t)\}}^p] \right)^{1/p} \eqsp.
    \end{align}
    Applying the result of \Cref{prop:rosenthal_Jn0}, we can bound the first term. For the second term, we can apply Minkowski's inequality together with \eqref{eq:lip_c_psi}, thus
    \begin{align}
        \left( \PEcoupling{\xi}{\Pi_{J,\alpha}}[\norm{\sum_{t=0}^{n-1} \{\psi(Y_t) - \psi(\tilde{Y}_t)\}}^p] \right)^{1/p} \leq \sum_{t=0}^{n-1}(\PEcoupling{\xi}{\Pi_{J,\alpha}}[\norm{\psi(Y_t) - \psi(\tilde{Y}_t)}^p])^{1/p} \leq 2 \bConst{A} \sum_{t=0}^{n-1} (\PEcoupling{\xi}{\Pi_{J,\alpha}}[\costw{}_J^{p}(Y_t, \tilde{Y}_t)])^{1/p} \eqsp.
    \end{align}
    Therefore, for the chosen exact coupling, using \eqref{eq:copling_projection} together with \eqref{eq:contr_cost_Y}, we get
    \begin{align}
        (\PEcoupling{\xi}{\Pi_{J,\alpha}}[\costw{}_J(Y_t, \tilde{Y}_t)^{p}])^{1/p} \leq \Auxconst_{W,3}^{(1)} p^2 \taumix^{3/2} \rho_{1,\alpha}^{t/p} \sqrt{\log{(1/\alpha a)}} \PE_{\xi, \Pi_{J,\alpha}}^{1/p}[\costw{}_J(Y, \tilde{Y})^p] \eqsp.
    \end{align}
    Hence, since for our choice of $\alpha$ it holds that $\sum_{t=0}^{n-1}\rho_{1,\alpha}^{t/p} \leq 24p(\alpha a)^{-1}$, we get
    \begin{align}
        \label{eq:bound_cost_stept}
        \left( \PEcoupling{\xi}{\Pi_{J,\alpha}}[\norm{\sum_{t=0}^{n-1} \{\psi(Y_t) - \psi(\tilde{Y}_t)\}}^p] \right)^{1/p} &\leq \Auxconst_{W,2}^{(2)} p^{7/2} \alpha^{-1/2} \sqrt{\log(1/\alpha a)}\eqsp,
    \end{align}
    where we have used
    \begin{align}
        \PE_{\xi, \Pi_{J,\alpha}}^{1/p} [\costw{}_J(Y,\tilde{Y})^p] \leq (\alpha a)^{1/2}(2\ConstD_1 (p\taumix)^{1/2} + 1)\supconsteps \eqsp.
    \end{align}
    Finally, to obtain the bound \eqref{eq:rosenthal_J_nonstat}, we combine \eqref{eq:decomp_sum_nonstat} and \eqref{eq:bound_cost_stept}. The same bounds holds for $\sum_{t=s}^{s+m-1}\bar{\psi}_t$, with $n$ replaced my $m$, for every $s \geq 0$.
\end{proof}

\section{Results for Richardson-Romberg procedure}
\label{appendix:main_results}
Define
\begin{align}
    &\ConstD_{J,1} = 10 \Auxconst_{J,5} + 2\Auxconst_{J,3} + 24 \Auxconst_{J,5} + 4\Auxconst_{J,6}, \quad \ConstD_{J,2} = \Auxconst_{J,4}+ 13,\\
    &\ConstD_{J,3} = 2(\Auxconst_{J,1} + \Auxconst_{J,2}), \quad \ConstD_{J} = \ConstD_{J,1} + \ConstD_{J,2} + \ConstD_{J,3} \eqsp,
    \label{eq:constD_Jn2}
\end{align}
where $\Auxconst_{J,1}, \Auxconst_{J,2}, \Auxconst_{J,3}, \Auxconst_{J,4}, \Auxconst_{J,5}$ and $\Auxconst_{J,6}$ are defined in \eqref{eq:auxconst_1_def}, \eqref{eq:auxconst_2_def}, \eqref{eq:auxconst_3_def}, \eqref{eq:auxconst_4_def}, \eqref{eq:auxconst_5_def} and \eqref{eq:auxconst_6_def}.

For simplicity of notation, in this section we use $\prthalpha{n}{\sf{RR}}$ instead of $\prthalpha{n}{\alpha, \sf{RR}}$. We preface the proof of \Cref{prop:bound_Jn2_alpha} by giving a statement of the Berbee lemma, which plays an essential role. Consider the extended measurable space $\tmszn = \msz^{\nset} \times [0,1]$, equipped with the $\sigma$-field $\tmczn = \mcz^{\otimes \nset} \otimes \mathcal{B}([0,1])$. For each probability measure $\xi$ on $(\Zset,\Zsigma)$, we consider the probability measure $\PPext_{\xi} = \PP_{\xi} \otimes \mathbf{Unif}([0,1])$ and denote by $\PEext_{\xi}$ the corresponding expectated value. Finally, we denote by $(\tZ_k)_{k \in\nset}$ the canonical process $\tZ_k : ((z_i)_{i\in\nset},u) \in \tmszn \mapsto z_k$ and $U :((z_i)_{i\in\nset},u) \in \tmszn \mapsto u$.
Under $\PPext_{\xi}$, $\sequence{\tZ}[k][\nset]$ is by construction a Markov chain with initial distribution $\xi$ and Markov kernel $\MKQ$ independent of $U$. The distribution of $U$ under $\PPext_{\xi}$ is uniform over $\ccint{0,1}$.
\begin{lemma}
  \label{lem:construction_berbee}
  Assume  \Cref{assum:drift}, let $m \in \nsets$ and $\xi$ be a probability measure on $(\msz,\mcz)$.   Then,  there exists a random process $(\tZs_{k})_{k\in\nset}$ defined on $(\tmszn, \tmczn, \PPext_{\xi})$ such that for any $k \in \nset$,
  \begin{enumerate}[wide,label=(\alph*)]
  \item \label{lem:construction_a} $\tZs_{k}$ is independent of $\tilde{\mcf}_{k+m} = \sigma\{\tZ_{\ell} \, : \, \ell \geq k+m\}$;
  \item \label{lem:construction_b} $\PPext_{\xi}(\tZs_k \neq \tZ_k) \leq  \dobru{\MKQ^m}$;
   \item \label{lem:construction_c} the random variables $\tZs_k$ and $\tZ_k$ have the same distribution under $\PPext_\xi$.
  \end{enumerate}
\end{lemma}
\begin{proof}[Proof]  Berbee's lemma \cite[Lemma~5.1]{riobook} ensures that for any $k$,
  there exists $\tZs_k$ satisfying \ref{lem:construction_a},
  \ref{lem:construction_c} and
  $\PPext_{\xi}(\tZs_k \neq \tZ_k) = \beta_{\xi}(\sigma(\tZ_{k}),
  \tilde{\mcf}_{k+m})$. Here for two sub $\sigma$-fields $\mathfrak{F}$,
  $\mathfrak{G}$ of $\tmczn$,
\begin{equation}
\label{eq:1}
\textstyle
    \beta_{\xi}(\mathfrak{F},\mathfrak{G}) = (1/2) \sup \sum_{i \in \msi} \sum_{j \in \msj} | \PPext_{\xi}( \msa_i \cap \msb_j)- \PPext_{\xi}(\msa_i)\PPext_{\xi}(\msb_j)|\eqsp,
\end{equation}
  and the supremum is taken over all pairs of partitions $\{\msa_i\}_{i\in\msi} \in \mathfrak{F}^\msi$ and $\{\msb_j\}_{j\in\msj}\in \mathfrak{G}^\msj$ of $\tmszn$ with $\msi$ and $\msj$ finite. Applying \cite[Theorem 3.3]{douc:moulines:priouret:soulier:2018} with \Cref{assum:drift} completes the proof.
\end{proof}

\begin{proposition}
\label{prop:bound_Jn2_alpha}
    Assume \Cref{assum:A-b}, \Cref{assum:noise-level} and \Cref{assum:drift}. Fix $2 \leq p < \infty$, $\alpha \in \ocint{0, \alpha_\infty}$ and initial probability measure $\xi$ on $(\Zset, \mcz)$, we have the following bound
    \begin{align}
        \PE_{\xi}^{1/p}[\norm{\Jnalpha{n}{2, \alpha}}^p] &\leq \ConstD_{J} \taumix^{5/2} p^{7/2} \alpha^{3/2}\log^{3/2}(1/\alpha a) \eqsp,
    \end{align}
    where $\ConstD_{J}$ is defined in \eqref{eq:constD_Jn2}.
\end{proposition}

\begin{proof}[Proof]
To bound $\Jnalpha{n}{2,\alpha}$ we define
\begin{align}
\label{eq:Sln2_def}
    &\Slnalpha{j+1}{i}{1} = \sum_{k=j+1}^i (\Id - \alpha \bA)^{i-k} \zmfuncA{Z_k}(\Id - \alpha \bA)^{k-j-1},\\
    &\Slnalpha{j+1}{n}{2} = \sum_{i=j+1}^{n} (\Id - \alpha \bA)^{n-i}\zmfuncA{Z_i}\Slnalpha{j+1}{i}{1} \eqsp.
\end{align}
Hence, following the definition \eqref{eq:jn_allexpansion_main} we have
\begin{equation}
    \Jnalpha{n}{2,\alpha} = -\alpha^3 \sum_{j=1}^{n-1} \Slnalpha{j+1}{n}{2} \funcnoise{Z_j} \eqsp.
\end{equation}
Now, we form blocks of size $m$ and let $N = \lfloor {n - 1\over m} \rfloor$ be a number of blocks. Then we can decompose
\begin{equation}
    \Jnalpha{n}{2,\alpha} = -\alpha^3 \sum_{j=1}^{(N-1)m} \Slnalpha{j+1}{n}{2} \funcnoise{Z_j} -\alpha^3 \sum_{j=(N-1)m + 1}^{n-1} \Slnalpha{j+1}{n}{2} \funcnoise{Z_j} = -\alpha^3 T_1 - \alpha^3 T_2 \eqsp.
\end{equation}
First, we are going to bound $T_2$. Using \Cref{lem:bound_Sln2}, we get
\begin{align}
    \label{eq:bound_T_2}
    \PE_{\xi}^{1/p}[\norm{T_2}^p] &\leq \sum_{j=(N-1)m + 1}^{n-1} \PE_{\xi}^{1/p}[\norm{\Slnalpha{j+1}{n}{2}\funcnoise{Z_j}}^p] \leq \Auxconst_{J,1} m^{3/2}\taumix^{3/2} p^{2} \alpha^{-1} \eqsp,
\end{align}
where we set
\begin{equation}
    \label{eq:auxconst_1_def}
   \Auxconst_{J,1} := {(\ConstD_1^{(1)} + \ConstD_2^{(1)}) \supconsteps \over a} \eqsp,
\end{equation}
and we used that $n-(N-1)m \leq 2m$ with $\log(x) \leq x^{1/2}$ for $x > 0$. To bound $T_1$ we should note a decomposition for $\Slnalpha{j+1}{i}{1}$
\begin{equation}
\label{eq:blocksepSln1}
    \Slnalpha{j+1}{i}{1} = (\Id - \alpha \bA)^{i-m-j}\Slnalpha{j+1}{j+m}{1} + \Slnalpha{j+m+1}{i}{1}(\Id - \alpha \bA)^m \eqsp.
\end{equation}
Substituting \eqref{eq:blocksepSln1} into $\Slnalpha{j+1}{n}{2}$, we get
\begin{align}
\label{eq:blocksepSln2}
    \Slnalpha{j+1}{n}{2} &= \sum_{i=j+1}^{j+m} (\Id - \alpha \bA)^{n-i} \zmfuncA{Z_i} \Slnalpha{j+1}{i}{1} + \sum_{i=j+m+1}^n (\Id - \alpha \bA)^{n-i} \zmfuncA{Z_i}\Slnalpha{j+1}{i}{1}\\
    &= (\Id - \alpha \bA)^{n-j-m} \Slnalpha{j+1}{j+m}{2} + \Slnalpha{j+m+1}{n}{1}(\Id - \alpha \bA)\Slnalpha{j+1}{j+m}{1} + \Slnalpha{j+m+1}{n}{2} (\Id - \alpha \bA)^m \eqsp.
\end{align}
Thus, $T_1$ can be represented as $T_1 = T_{11} + T_{12} + T_{13}$, where
\begin{align}
    &T_{11} = \sum_{j=1}^{(N-1)m} (\Id - \alpha \bA)^{n-j-m}\Slnalpha{j+1}{j+m}{2}\funcnoise{Z_j},\\
    &T_{12} =  \sum_{j=1}^{(N-1)m} \Slnalpha{j+m+1}{n}{1} (\Id - \alpha \bA)\Slnalpha{j+1}{j+m}{1} \funcnoise{Z_j},\\
    &T_{13} = \sum_{j=1}^{(N-1)m} \Slnalpha{j+m+1}{n}{2} (\Id - \alpha \bA)^m \funcnoise{Z_j} \eqsp.
\end{align}
For the first term, using \Cref{lem:bound_Sln2} we get
\begin{align}
    \label{eq:bound_T11}
    \PE_{\xi}^{1/p}[\norm{T_{11}}^p] &\leq \qcond^{1/2} \sum_{j=1}^{(N-1)m} (1 - \alpha a)^{(n-j-m)/2} \PE_{\xi}^{1/p}[\norm{\Slnalpha{j+1}{j+m}{2} \funcnoise{Z_j}}^p]\\
    &\leq \qcond^{1/2} (\ConstD_1^{(1)} + \ConstD_2^{(1)}) \taumix^{3/2} p^{2} \supconsteps m^{3/2} \sum_{j=1}^{(N-1)m} (1 - \alpha a)^{(n-j-1)/2} \leq  \Auxconst_{J,2} m^{3/2} \taumix^{3/2} p^{2} \alpha^{-1} \eqsp,
\end{align}
where
\begin{equation}
    \label{eq:auxconst_2_def}
    \Auxconst_{J,2} := {\qcond^{1/2} (\ConstD_1^{(1)} + \ConstD_2^{(1)})\supconsteps \over a} \eqsp.
\end{equation}
For the second term, we have
\begin{align}
    \PE_{\xi}^{1/p}&[\norm{T_{12}}^p] \\
    &\leq \sum_{j=1}^{(N-1)m} \sum_{k=j+1}^{j+m}  \PE_{\xi}^{1/p}[\norm{\Slnalpha{j+m+1}{n}{1} (\Id - \alpha \bA)^{j+m-k+1} \zmfuncA{Z_k} (\Id - \alpha \bA)^{k-j-1}\funcnoise{Z_j}}^p]\\
    &\leq \sum_{j=1}^{(N-1)m} \sum_{k=j+1}^{j+m}\PE^{1/p}[\norm{v_{j,k}}^p\PE^{\mcf_{j+m}}[\norm{\Slnalpha{j+m+1}{n}{1} v_k / \norm{v_{j,k}}}^p]] \eqsp,
\end{align}
where
\begin{equation}
    v_{j,k} = (\Id - \alpha \bA)^{j+m-k+1}\zmfuncA{Z_k}(\Id - \alpha \bA)^{k-j-1}\funcnoise{Z_j} \eqsp.
\end{equation}
Let
\begin{equation}
    B_1(\alpha) = \sum_{k=j+1}^{j+m} (n-j-m)^{1/2}(1 - \alpha a)^{(n-j-m-1)/2}\PE_{\xi}^{1/p}[\norm{v_{j,k}}^p] \eqsp.
\end{equation}
Then, we have
\begin{align}
    B_1(\alpha) &\leq \qcond \bConst{A}\supconsteps m (1-\alpha a)^{m/2} \sum_{j=1}^{(N-1)m}(n-j-m)^{1/2} (1 - \alpha a)^{(n-j-m-1)/2} \leq 8\sqrt{\pi}(\alpha a)^{-3/2}
\end{align}
Thus, using \cite[Lemma 5]{durmus2025finite}, we get
\begin{align}
    \label{eq:bound_T12}
    \PE_{\xi}^{1/p}[\norm{T_{12}}^p] &\leq \sum_{j=1}^{(N-1)m} \sum_{k=j+1}^{j+m}\PE_{\xi}^{1/p}[\norm{v_{j,k}}\sup_{u \in \mathbb{S}^{d=1}, \xi' \in \mathcal{P}(\Zset)}\PE_{\xi'}[\norm{\Slnalpha{j+m+1}{n}{1} u}^p]]\\
    &\leq 16\qcond \bConst{A} (\taumix p)^{1/2} B_1(\alpha) \leq \Auxconst_{J,3} m (\taumix p)^{1/2} \alpha^{-3/2} \eqsp,
\end{align}
where
\begin{equation}
    \label{eq:auxconst_3_def}
    \Auxconst_{J,3} := {128\sqrt{\pi} \qcond^2 \bConst{A}^2 \supconsteps \over a^{3/2}} \eqsp.
\end{equation}
To bound the third term we should switch to the extended space $(\tmszn, \tmczn, \PPext_{\nset})$. From \Cref{lem:construction_berbee} it follows that $\PE_{\xi}^{1/p}[\norm{T_{13}}^p] = \PEext_{\xi}^{1/p}[\norm{\tilde{T}_{13}}^p]$ with
\begin{equation}
    \tilde{T}_{13} = \sum_{j=1}^{(N-1)m} \Slnalphaext{j+m+1}{n}{2} (\Id - \alpha \bA)^m \funcnoise{\tilde{Z}_j} \eqsp,
\end{equation}
where $\Slnalphaext{j+m+1}{n}{2}$ is a counterpart of $\Slnalpha{j+m+1}{n}{2}$ but defined on the extended space. Thus, we have
\begin{align}
    \tilde{T}_{13} &= \sum_{s=0}^{N-2}\sum_{j=1}^m \Slnalphaext{(s+1)m + j + 1}{n}{2} (\Id - \alpha \bA)^{m} \funcnoise{\tilde{Z}_{sm+j}^*}\\
        &+ \sum_{s=0}^{N-2}\sum_{j=1}^m \Slnalphaext{(s+1)m+j+1}{n}{2}(\Id - \alpha \bA)^{m}(\funcnoise{\tilde{Z}_{sm+j}} - \funcnoise{\tilde{Z}_{sm+j}^*)} = \tilde{T}_{131} + \tilde{T}_{132} \eqsp.
\end{align}
Now, define the function $g(z) : \Zset \to \mathbb{R}^d$, $g(z) = (\Id - \alpha \bA)^m \funcnoise{z}$. Using \Cref{fact:Hurwitzstability}, we can bound this function by $\|g\|_\infty \leq \qcond^{1/2}(1-\alpha a)^{m/2} \supconsteps$ while $\pi(g) = 0$. Using \Cref{lem:construction_berbee} and \cite[Lemma 6]{durmus2025finite} we can estimate $\tilde{T}_{131}$ as follows
\begin{align}\label{eq:bound_T131}
    \PEext_{\xi}^{1/p}[\norm{\tilde{T}_{131}}^p] &\leq \sum_{j=1}^m  2p \|g\|_{\infty} \left\{ \sum_{s=0}^{N-2} \sup_{u \in \mathbb{S}^{d-1}}\PEext_\xi^{2/p}[\norm{\Slnalphaext{(s+1)m+j+1}{n}{2} u}^p] \right\}^{1/2} \\
    &+ \sum_{j=1}^m\sum_{s=0}^{N-2} \norm{\xi \MKQ^{sm+j}g}\sup_{u \in \mathbb{S}^{d-1}}\PEext_\xi^{1/p}[\norm{\Slnalphaext{(s+1)m+j+1}{n}{2}u}^p] \eqsp.
\end{align}
Further, using \Cref{lem:bound_Sln2} and $\norm{\xi \MKQ^{sm+j}g} \leq \dobru{\MKQ^{sm+j}}\|g\|_{\infty}$, we get
\begin{align}
    \label{eq:bound_T1312}
    \sum_{j=1}^m\sum_{s=0}^{N-2} &\norm{\xi \MKQ^{sm+j}g}\sup_{u \in \mathbb{S}^{d-1}}\PEext_\xi^{1/p}[\norm{\Slnalphaext{(s+1)m+j+1}{n}{2}u}^p]\\
    &\leq 2\|g\|_\infty (\ConstD_1^{(1)} + \ConstD_2^{(1)}) \taumix^{3/2} p^{2} \sup_{x\geq 1}\{x^{3/2}(1 -\alpha a)^{x/2}\} \sum_{\ell=0}^{+\infty} \dobru{\MKQ^\ell} \leq  \Auxconst_{J,4}\taumix^{5/2} p^{2} (1 - \alpha a)^{(m-1)/2} \alpha^{-3/2} \eqsp,
\end{align}
where we used that $\sup_{x\geq 1}\{x^{3/2}(1 -\alpha a)^{x/2}\} \leq 3 (\alpha a)^{-3/2}$ and
\begin{equation}
    \label{eq:auxconst_4_def}
    \Auxconst_{J,4} := {12\qcond^{1/2} (\ConstD_1^{(1)} + \ConstD_2^{(1)})\supconsteps \over a^{3/2}} \eqsp.
\end{equation}
Denote
\begin{equation}
    B_2(\alpha) = \sum_{j=1}^{(N-1)m} (n-j-m)^{2}\log^2(n-j-m)(1 - \alpha a)^{n-j-m-1} \eqsp.
\end{equation}
We can bound $B_2(\alpha)$ as
\begin{align}
    B_2(\alpha) \leq \int_{0}^{+\infty} t^2 \log^2(t)e^{-\alpha a t/2}dt &\leq 16 (\alpha a)^{-3} \log^2(2/\alpha a) \int_0^{+\infty} t^2 e^{-t}dt + 16 (\alpha a)^{-3}\int_0^{+\infty}t^2 \log^2(t) e^{-t}dt\\
    &\leq (32 \log^2(2/\alpha a) + 112)(\alpha a)^{-3} \eqsp,
\end{align}
For the first term of \eqref{eq:bound_T131}, using Jensen's inequality and \Cref{lem:bound_Sln2}, we obtain
\begin{align}
    \label{eq:bound_T1311}
    2p\|g\|_{\infty} \sum_{j=1}^m \left\{ \sum_{s=0}^{N-2} \sup_{u \in \mathbb{S}^{d-1}}\PEext_\xi^{2/p}[\norm{\Slnalphaext{(s+1)m+j+1}{n}{2} u}^p] \right\}^{1/2} &\leq  2(\ConstD_1^{(1)} + \ConstD_2^{(1)}) \taumix^{3/2}p^{3} m^{1/2} \|g\|_{\infty} B_1^{1/2}(\alpha)\\
    &\leq  \Auxconst_{J,5} \taumix^{3/2} p^{3} m^{1/2} (1 - \alpha a)^{(m-1)/2} \alpha^{-3/2}(8\log(1/\alpha a) + 17) \eqsp,
\end{align}
where we set
\begin{equation}
    \label{eq:auxconst_5_def}
    \Auxconst_{J,5} = {2\qcond^{1/2}(\ConstD_1^{(1)} + \ConstD_2^{(1)}) \supconsteps \over a^{3/2}} \eqsp,
\end{equation}
combined with the fact that $\int_0^{+\infty}t^2 \log^2(t)e^{-t}dt \leq 7$.
Now we can bound $\tilde{T}_{132}$. Set $V_l = \funcnoise{\tilde{Z}_l} - \funcnoise{\tilde{Z}_l^*}$ and $\mcf_l^* = \sigma(\tilde{Z}_i, \tilde{Z}_i^* | 1 \leq i \leq l)$. For the term $\tilde{T}_{132}$, we have
\begin{align}
    \PEext_{\xi}^{1/p}[\norm{\tilde{T}_{132}}^p] &\leq \sum_{s=0}^{N-2}\sum_{j=1}^m \PEext_{\xi}^{1/p}[\norm{\Slnalphaext{(s+1)m+j+1}{n}{2}(\Id - \alpha \bA)^m V_{sm+j}}^p] \\
    &\leq \sum_{s=0}^{N-2}\sum_{j=1}^m \PEext_{\xi}^{1/p}[\norm{V_{sm+j}}^p \PEext_{\xi}^{\mcf_{sm+j}^*}[\norm{\Slnalphaext{(s+1)m+j+1}{n}{2}(\Id - \alpha \bA)^m V_{sm+j}/\norm{V_{sm+j}}}^p]]\\
    &\leq \sum_{s=0}^{N-2}\sum_{j=1}^m \PEext_{\xi}^{1/p}[\norm{V_{sm+j}}^p \sup_{u \in \mathbb{S}^{d-1}, \xi' \in \mathcal{P}(\Zset)} \PEext_{\xi'}[\norm{\Slnalphaext{(s+1)m+j+1}{n}{2} (\Id - \alpha \bA)^m u}^p]] \eqsp,
\end{align}
where $\mathcal{P}(\Zset)$ is the set of probability measure on $(\Zset,\Zsigma)$. Under \Cref{assum:noise-level} and \Cref{assum:drift}, we have $\norm{V_{sm+j}} \leq 2\supconsteps \mathbb{I}\{\tilde{Z}_{sm+j} \neq \tilde{Z}_{sm+j}^*\}$ and $\PPext[\tilde{Z}_{sm+j} \neq \tilde{Z}_{sm+j}^*] \leq \dobru{\MKQ^m} \leq (1/4)^{\lfloor m/\taumix \rfloor}$. Denote
\begin{align}
    B_3(\alpha) = \sum_{j=1}^{(N-1)m}(n-j-m)\log(n-j-m)(1 - \alpha a)^{(n-j-m-1)/2} \eqsp.
\end{align}
Then, as in case with $B_2(\alpha)$, we have
\begin{align}
    B_3(\alpha) \leq \int_{0}^{+\infty}t \log(t)e^{-\alpha a t/2} dt \leq (\alpha a)^{-2} (4\log(1/\alpha a) + 7)
\end{align}
Applying \Cref{lem:bound_Sln2}, we obtain
\begin{align}
    \label{eq:bound_T132}
    \PEext_{\xi}^{1/p}[\norm{\tilde{T}_{132}}^p] &\leq 2 \supconsteps \taumix^{3/2} p^{2} (1/4)^{(1/p)\lfloor m/\taumix \rfloor}(1 - \alpha a)^{m/2} B_2(\alpha) \\
    &\leq \Auxconst_{J,6} \taumix^{3/2} p^{2} (1/4)^{(1/p)\left\lfloor {m \over \taumix} \right\rfloor}(1 - \alpha a)^{(m-1)/2} \alpha^{-2} (4 \log(1/\alpha a) + 7) \eqsp,
\end{align}
where
\begin{equation}
    \label{eq:auxconst_6_def}
    \Auxconst_{J,6} := {2 (\ConstD_1^{(1)} + \ConstD_2^{(1)})\supconsteps \over a^{2}} \eqsp.
\end{equation}
Finally, we set
\begin{equation}
    m = \taumix \left\lceil {p\log{(1/\alpha a)} \over 2 \log{2}} \right\rceil  \eqsp.
\end{equation}
With this choice of $m \geq \taumix$, we have $(1/4)^{(1/p)\lfloor m/\taumix \rfloor} \leq (\alpha a)^{1/2}$ and\\
$m \leq 2 \taumix p \log{(1/\alpha a)} / (2\log{2})$. Thus, substituting such $m$ into the \eqref{eq:bound_T_2}, \eqref{eq:bound_T11}, \eqref{eq:bound_T12}, \eqref{eq:bound_T1311}, \eqref{eq:bound_T1312}, \eqref{eq:bound_T132} we obtain the result.
\end{proof}


\subsection{Proof of \Cref{thm:error_RR_iter}}
\label{appendix:proof_error_RR_iter}
We preface the proof of main result by auxiliary lemma and proposition. 
\begin{lemma}
\label{lem:drift_condition_LSA_Markov}
Assume \Cref{assum:A-b}, \Cref{assum:noise-level}, and \Cref{assum:drift}.  Let $2 \leq p \leq q/2$. Then, for any $\alpha \in (0;\alpha_{q,\infty}^{(\Markov)}\taumix^{-1})$ with $\alpha_{q,\infty}^{(\Markov)}$ defined in \eqref{eq:def_alpha_p_infty_Markov}, $\thetainit \in \rset^d$, probability $\xi$ on $(\Zset,\Zsigma)$, and $n \in \nset$, it holds
\begin{equation}
\label{eq:n_step_drift}
\PE_\xi^{1/p}[\norm{\thalpha{n}{\alpha} - \thetalim}^p] \leq \sqrt{\qcond} e^2 d^{1/q} \rate[1, \alpha]^n \norm{\theta_0 - \thetalim} + \ConstD_2 d^{1/q} \sqrt{\alpha a p \taumix}\supconsteps,
\end{equation}
where $\ConstD_2$ and $\rate[1, \alpha]$ are defined as
\begin{equation}
\label{eq:definition:ConstDM_2}
    \textstyle 
    \ConstD_2 = \ConstD_{1} (1 + 24\sqrt{2}\rme^2 \sqrt{\qcond} \bConst{A} a^{-1}) \,, \quad \rate[1, \alpha] = e^{-\alpha a /12} \eqsp.
\end{equation}
\end{lemma}
\begin{proof}
See \cite[Proposition 9]{durmus2025finite}.
\end{proof}

\begin{proposition}
    \label{prop:bound_Hn2}
    Assume \Cref{assum:A-b}, \Cref{assum:noise-level} and \Cref{assum:drift}. Fix $2 \leq p \leq q/2$, $\alpha \in (0, \alpha_{q,\infty}^{(\Markov)}\taumix^{-1})$ and probability distribution $\xi$ on $(\Zset, \Zsigma)$. Then, we have
    \begin{align}
        \PE_{\xi}^{1/p}[\norm{\Hnalpha{n}{2,\alpha}}^p] &\leq \ConstD_{H}d^{1/q} \taumix^{5/2} p^{7/2} \alpha^{3/2}\log^{3/2}(1/\alpha a) \eqsp,
    \end{align}
    where 
    \begin{align}\label{eq:constD_Hn2}
        \ConstD_{H} = 384 \qcond^{1/2}\bConst{A}a^{-1}e^2 \ConstD_{J} \eqsp,
    \end{align}
    and $\ConstD_{J}$ is defined in \eqref{eq:constD_Jn2}.
\end{proposition}
\begin{proof}
    Unrolling the recursion \eqref{eq:jn_allexpansion_main}, we get
    \begin{equation}
        \Hnalpha{n}{2,\alpha} = -\alpha \sum_{l=1}^n \ProdBa_{l+1:n} \zmfuncA{Z_l} \Jnalpha{l-1}{2,\alpha}
    \end{equation}
    Thus, using Minkowski's and Holder inequalities, we have
    \begin{align}
        \PE_{\xi}^{1/p}[\norm{\Hnalpha{n}{2,\alpha}}^p] \leq \alpha\sum_{l=1}^n \PE^{1/2p}[\norm{\ProdBa_{l+1:n}\zmfuncA{Z_l}}^{2p}] \PE_{\xi}^{1/2p}[\norm{\Jnalpha{l-1}{2,\alpha}}^{2p}]
    \end{align}
    Using \cite[Proposition 7]{durmus2025finite}, we can bound the first factor as
    \begin{equation}
        \PE_{\xi}^{1/2p}[\norm{\ProdBa_{l+1:n}\zmfuncA{Z_l}}^{2p}] \leq 2\sqrt{\qcond}\bConst{A}e^2 d^{1/q}e^{-\alpha a(n-l)/12}
    \end{equation}
    Combining the inequalities above and using \Cref{prop:bound_Jn2_alpha}, we obtain
    \begin{align}
        \PE_{\xi}^{1/p}[\norm{\Hnalpha{n}{2,\alpha}}^p] \leq 16\ConstD_{J} \qcond^{1/2}\bConst{A}e^2 d^{1/q} \taumix^{5/2} p^{7/2} \alpha^{5/2}\log^{3/2}(1/\alpha a)\sum_{l=1}^n e^{-\alpha a l /12}
    \end{align}
    Finally, using that $e^{-x} \leq 1 - x/2$ for $x\in (0, 1)$, we get the result.
\end{proof}

Define the quantities
\begin{align}
\label{eq:const_D_RR}
    c_{\text{tr}, p} &= {a \over 48} \alpha_{p,\infty}^{(b)} \eqsp,\\
    &\ConstD_1^{(\sf RR)} = 12\ConstD_2\rme^{1/p}(a \taumix)^{1/2}\supconsteps + 8\Auxconst_{W,2}^{(2)},\\
    &\ConstD_2^{(\sf RR)} = 14\Auxconst_\psi \taumix^2\eqsp,\\
    &\ConstD_{3}^{(\sf RR)} = \bConst{A}(6\ConstD_J + 3\ConstD_H \rme^{1/p})\taumix^{5/2} + 28 \norm{\bA}\norm{\bA^{-1}}\taumix^{5/2}\supconsteps, \\
    &\ConstD_4^{(\sf RR)} = 16\ConstD_J \taumix^{5/2}, \quad \bConst{\sf{Ros},p} = 2(\bConst{\sf{Ros}, 1}^{(\Markov)} + \bConst{\sf{Ros}, 2}^{(\Markov)}) \taumix^{3/4} \log_2(2p) \eqsp,
\end{align}
and
\begin{align}
\label{eq:constants_RR}
    R_{n,p,\alpha,\taumix}^{(\sf fl)} &= \bConst{\sf{Ros},p} p n^{-3/4} + (\ConstD_1^{(\sf RR)}p^{2}(\alpha n)^{-1/2}\sqrt{\log{(1/\alpha a)}} + \ConstD_2^{(\sf RR)} \alpha^{1/2})p^{3/2} n^{-1/2}\\
    &+ (\ConstD_{3}^{(\sf RR)}\alpha + \ConstD_4^{(\sf RR)} n^{-1}) p^{7/2} \alpha^{1/2}\log^{3/2}(1/\alpha a) \eqsp,\\
    \MoveEqLeft[4] R_{n,p,\alpha,\taumix}^{(\sf tr)} = 13 (1 + \bConst{A})\qcond^{1/2} \rme^{2 + 1/p} (\alpha n)^{-1} \eqsp.
\end{align}

\begin{proof}[Proof of \Cref{thm:error_RR_iter}]
We start with applying \eqref{eq:decompo_e_theta_z} to the decomposition \eqref{eq:RR_err_decompose}. Setting $n_0 = n/2$, we get
    \begin{align}
        \bA(\prthalpha{n}{\sf RR} - \thetalim) &= \underbrace{{4 \over \alpha n}(\thalpha{n/2}{\alpha} - \thalpha{n}{\alpha}) - {1 \over \alpha n}(\thalpha{n/2}{2\alpha} - \thalpha{n}{2\alpha})}_{\Tterm{1}{2}} + \underbrace{{2\over n}\Etralpha{2\alpha}_{n} - {4\over n}\Etralpha{\alpha}_n}_{\Tterm{\sf tr}{2}} - {2 \over n}\sum_{t=n/2}^{n-1} \funcnoise{Z_{t+1}} \\
        &+ \sum_{l=0}^1 \underbrace{\left\{{2 \over n}\sum_{t=n/2}^{n-1} \zmfuncA{Z_{t+1}}\Jnalpha{t}{l, 2\alpha} - {4 \over n} \sum_{t=n/2}^{n-1} \zmfuncA{Z_{t+1}} \Jnalpha{t}{l, \alpha}\right\}}_{\Tterm{J,l}{2}}\\
        &+ \underbrace{{2 \over n}\sum_{t=n/2}^{n-1} \zmfuncA{Z_{t+1}}\Hnalpha{t}{2, 2\alpha} - {4 \over n}\sum_{t=n/2}^{n-1} \zmfuncA{Z_{t+1}}\Hnalpha{t}{2, \alpha}}_{\Tterm{H}{2}} \eqsp.
    \end{align}
Now, we use \Cref{lem:drift_condition_LSA_Markov} to bound the terms which correspond to the deviation of the last iterate. Hence, using Minkowski's inequality we can bound $\Tterm{1}{2}$, as
\begin{align}
    \PE_{\xi}^{1/p}[\norm{\Tterm{1}{2}}^p] &\leq 10 (\alpha n)^{-1} \sqrt{\qcond}e^2d^{1/q}e^{-\alpha an/24}\norm{\theta_0 -\thetalim} + 12 \ConstD_2 d^{1/q}(p \taumix a)^{1/2} \alpha^{-1/2}n^{-1}\supconsteps \eqsp.
\end{align}
To bound the transient terms we should use the exponential stability for the product of random matrices. That is, using \cite[Proposition 7]{durmus2025finite}, we get
\begin{align}
    \PE_{\xi}^{1/p}[\norm{\Etralpha{\alpha}_n}^p] \leq (n/ 2) \sqrt{\qcond} e^2 d^{1/q}\bConst{A}e^{-\alpha a n/24}\norm{\theta_0 - \thetalim} \eqsp.
\end{align}
Thus, we can bound $\Tterm{\sf tr}{2}$ as
\begin{equation}
    \PE_{\xi}^{1/p}[\norm{\Tterm{\sf tr}{2}}^p] \leq 3\sqrt{\qcond} e^2 d^{1/q}\bConst{A}e^{-\alpha a n/24}\norm{\theta_0 - \thetalim} \eqsp.
\end{equation}
The leading term $(2/n)\sum_{t=n/2}^{n-1}\funcnoise{Z_{t+1}}$ is a linear statistic of UGE Markov chain. Thus, using \Cref{theo:rosenthal_uge_arbitrary_init}, we get
\begin{align}
    \PE_{\xi}^{1/p}[\norm{\sum_{t=n/2}^{n-1}\funcnoise{Z_{t+1}}}^p] &\leq \bConst{\sf{Rm}, 1} p^{1/2} n^{1/2} \{\trace \noisecov\}^{1/2} + \bConst{\sf{Ros}, 1}^{(\Markov)}n^{1/4}\taumix^{3/4}p\log_2(2p) + \bConst{\sf{Ros}, 2}^{(\Markov)} \taumix p \log_2(2p) \eqsp.
\end{align}
Now, we can bound $\Tterm{J,1}{2}$ through $\Jnalpha{t}{2,\alpha}$. Indeed, using the expansion \eqref{eq:jn_allexpansion_main}, we have
\begin{align}
    \sum_{t=n/2}^{n-1}\zmfuncA{Z_{t+1}}\Jnalpha{t}{1,\alpha} = \alpha^{-1} (\Jnalpha{n/2}{2,\alpha} - \Jnalpha{n}{2,\alpha}) - \sum_{t=n/2}^{n-1}\funcA{Z_{t+1}}\Jnalpha{t}{2,\alpha} \eqsp.
\end{align}
The first term can be bounded directly using \Cref{prop:bound_Jn2_alpha}. Also, using Minkowski's inequality, we can bound the second term, as
\begin{align}
    \PE_{\xi}^{1/p}[\norm{\sum_{t=n/2}^{n-1}\funcA{Z_{t+1}}\Jnalpha{t}{2,\alpha}}^p] &\leq (n/2)\bConst{A} \sup_{n/2 \leq t \leq n} \PE^{1/p}[\norm{\Jnalpha{t}{2,\alpha}}^p] \eqsp.
\end{align}
Hence, we get
\begin{align}
    \PE_{\xi}^{1/p}[\norm{\sum_{t=n/2}^{n-1}\zmfuncA{Z_{t+1}}\Jnalpha{t}{1,\alpha}}^p] \leq (2 \alpha^{-1} + (n/2)\bConst{A})\sup_{n/2 \leq t \leq n} \PE_{\xi}^{1/p}[\norm{\Jnalpha{t}{2,\alpha}}^p] \eqsp.
\end{align}
Thus, using \Cref{prop:bound_Jn2_alpha}, we can bound $\Tterm{J,1}{2}$, as follows

\begin{align}
    \PE_{\xi}^{1/p}[\norm{\Tterm{J,1}{2}}^p] &\leq (2/n)(\alpha^{-1} + (n/2)\bConst{A})\sup_{n/2 \leq t \leq n}\PE_{\xi}^{1/p}[\norm{\Jnalpha{t}{2,2\alpha}}^p] + (4/n)(2\alpha^{-1}+ (n/2)\bConst{A})\sup_{n/2 \leq t \leq n} \PE_{\xi}^{1/p}[\norm{\Jnalpha{t}{2,\alpha}}^p]\\
    &\leq (16(\alpha n)^{-1} + 6\bConst{A}) \sup_{t \in \nsets}\PE_{\xi}^{1/p}[\norm{\Jnalpha{t}{2,\alpha}}^p] \leq (16 (\alpha n)^{-1} + 6 \bConst{A}) \ConstD_{J} \taumix^{5/2} p^{7/2} \alpha^{3/2}\log^{3/2}(1/\alpha a) \eqsp.
\end{align}
Using the notation of \Cref{appendix:rosenthal_Jnalpha}, we have the following expansion
\begin{align}
    \label{eq:exp_Jn0}
    \sum_{t=n/2}^{n-1} \zmfuncA{Z_{t+1}}&\Jnalpha{t}{0, 2\alpha} - 2 \sum_{t=n/2}^{n-1} \zmfuncA{Z_{t+1}}\Jnalpha{t}{0, \alpha} = \sum_{t=n/2}^{n-1} \bar{\psi}_t^{(2\alpha)} -2 \sum_{t=n/2}^{n-1}\bar{\psi}_t^{(\alpha)}\\
    &+ \sum_{t=n/2}^{n-1}\left\{ \PE_{\pi_J}[\zmfuncA{Z_{t+1}}\Jnalpha{t}{0,2\alpha}] - 2\PE_{\pi_J}[\zmfuncA{Z_{t+1}}\Jnalpha{t}{0,\alpha}] \right\} \eqsp.
\end{align}
To bound the last term, we apply \Cref{prop:Jnalpha_asymp_exp}, and get
\begin{align}
    &\norm{\sum_{t=n/2}^{n-1}\left\{ \PE_{\pi_J}[\zmfuncA{Z_{t+1}}\Jnalpha{t}{0,2\alpha}] - 2\PE_{\pi_J}[\zmfuncA{Z_{t+1}}\Jnalpha{t}{0,\alpha}] \right\}} \\
    &\leq (n/2)\norm{2\bA R(\alpha) - \bA R(2\alpha)} \leq 14\norm{\bA}\norm{\bA^{-1}} \bConst{A} \taumix^2 n \alpha^2 \supconsteps \eqsp.
\end{align}
For the other terms, we apply \Cref{cor:rosenthal_J_nonstat}, and obtain
\begin{align}
     (n/2)\PE_{\xi}^{1/p}[\norm{\Tterm{J,0}{2}}^p] &\leq  7\Auxconst_\psi \taumix^2 p^{3/2} (\alpha n)^{1/2}\\
        &+ 4\Auxconst_{W,2}^{(2)} p^{7/2} \alpha^{-1/2} \sqrt{\log(1/\alpha a)}\\
        &+ 14\norm{\bA}\norm{\bA^{-1}} \bConst{A} \taumix^2 n\alpha^2 \supconsteps \eqsp.
\end{align}
Now, to bound $\Tterm{H}{2}$ we apply Minkowski's inequality
\begin{align}
    \PE_{\xi}^{1/p}[\norm{\sum_{t=n/2}^{n-1} \zmfuncA{Z_{t+1}}\Hnalpha{t}{2,\alpha}}^p] \leq (n/2)\bConst{A} \sup_{n/2 \leq t \leq n}\PE_{\xi}^{1/p}[\norm{\Hnalpha{t}{2,\alpha}}^p] \eqsp.
\end{align}
Using this bound, we get
\begin{align}
    \PE_{\xi}^{1/p}[\norm{\Tterm{H}{2}}^p] \leq 3\bConst{A} \sup_{t \in \nsets} \PE_{\xi}^{1/p}[\norm{\Hnalpha{t}{2,\alpha}}^p] \eqsp.
\end{align}
Finally, we apply \Cref{prop:bound_Hn2} and obtain the result \eqref{eq:error_RR_iter}.
\end{proof}


\section{Technical lemmas}
\label{appendix:technical}
Recall that $\Slnalpha{\ell+1}{\ell+m}{1}$ is defined, for $\ell,m \in \nsets$, as
\begin{equation}
\label{eq:S_ell_n_def_tech_markov}
\textstyle
\Slnalpha{\ell+1}{\ell+m}{1} =  \sum_{k = \ell+1}^{\ell+m} \funcBw_k(\State_{k}) \eqsp, \text{ with } \funcBw_k(z) =  (\Id - \alpha \bA)^{\ell+m-k} \zmfuncA{z} (\Id - \alpha \bA)^{k-1 - \ell} \eqsp.
\end{equation}

\begin{lemma}\label{lem:bound_Jn1}
    Assume \Cref{assum:A-b}, \Cref{assum:noise-level} and \Cref{assum:drift}. Then, for any $p \geq 2$, $\alpha \in (0, \alpha_{\infty}]$, and initial probability measure $\xi$ on $(\Zset, \Zsigma)$, it holds that
        \begin{equation}
            \PE_{\xi}^{1/p}[\norm{\Jnalpha{n}{1, \alpha}}^p] \leq \supconsteps(\alpha a \taumix)(\ConstD_{J,1}^{(M)} \sqrt{\log(1/\alpha a)}p^2 + \ConstD_{J,2}^{(M)}(\alpha a \taumix)^{1/2}p^{1/2}) \eqsp.
        \end{equation}
    Particularly, it holds that
        \begin{equation}
            \PE_{\xi}^{1/p}[\norm{\Jnalpha{n}{1, \alpha}}^p] \leq (\ConstD_{J,1}^{(\Markov)} + \ConstD_{J,2}^{(\Markov)})\supconsteps p^2 \taumix^{3/2} \alpha a \sqrt{\log{(1/\alpha a)}} \eqsp.
        \end{equation}
\end{lemma}

\begin{proof}
    The precise constants and proof can be found in \cite[Proposition 10]{durmus2025finite}.
\end{proof}

\begin{lemma}\label{lem:bound_Hn1}
    Assume \Cref{assum:A-b}, \Cref{assum:noise-level} and \Cref{assum:drift}. Then, for any $p,q \geq 2$, satisfying $2 \leq p \leq q/2$, $\alpha \in (0, \alpha_{q,\infty}^{(\Markov)}\taumix^{-1}]$, and initial probability measure $\xi$ on $(\Zset, \Zsigma)$, it holds that
        \begin{equation}
            \PE_{\xi}^{1/p}[\norm{\Hnalpha{n}{1, \alpha}}^p] \leq d^{1/q}\supconsteps(\alpha a \taumix)(\ConstD_{H,1}^{(\Markov)} \sqrt{\log(1/\alpha a)}p^2 + \ConstD_{H,2}^{(\Markov)}(\alpha a \taumix)^{1/2}p^{1/2}) \eqsp.
        \end{equation}
\end{lemma}

\begin{proof}
    The precise constants and proof can be found in \cite[Proposition 10]{durmus2025finite}.
\end{proof}

\begin{lemma}\label{lem:bound_Sln2}
    Assume \Cref{assum:A-b}, \Cref{assum:noise-level} and \Cref{assum:drift}. For any probability measure $\xi \in \mathcal{P}(\Zset)$, $j, r \in \mathbb{N}$ and $u \in \mathbb{S}^{d-1}$, step size $\alpha \in (0,\alpha_{\infty})$, we have
    \begin{align}
        \sup_{u \in \mathbb{S}^{d-1}} \PE_{\xi}^{1/p}[\norm{\Slnalpha{j+1}{j+r}{2} u}^p] &\leq (\ConstD_1^{(1)} p \log{(r)} + \ConstD_2^{(2)})\taumix^{3/2} p r (1 - \alpha a)^{(r-1)/2} \eqsp,
    \end{align}
    where
    \begin{equation}
        \ConstD_{1}^{(1)} = \qcond^{3/2}(48 \qcond^{1/2} + 1)\bConst{A}^2/\log(2), \quad \ConstD_{2}^{(1)} = \qcond(34\qcond + 1) \bConst{A}^2 \eqsp.
    \end{equation}
\end{lemma}
\begin{proof}
    Firstly, define for any $k \in \{j+1, \dots, j+m-1\}$ the function $g_k(z) = \zmfuncA{z} (\Id - \alpha \bA)^{k-j-1}u$, which is bounded by $\|g_k\|_{\infty} \leq \sqrt{\qcond} \bConst{A}(1-\alpha a)^{(k-j-1)/2}$. Then, following the definition \eqref{eq:Sln2_def} we get
    \begin{align}
        \Slnalpha{j+1}{j+r}{2} &= \sum_{i=j+1}^{j+r} \sum_{k=j+1}^i (\Id - \alpha \bA)^{j+r-i}\zmfuncA{Z_i} (\Id - \alpha \bA)^{i-k}g_k(Z_k)\\
        &= \sum_{k=j+1}^{j+r} (\Id - \alpha \bA)^{j+r-k} \zmfuncA{Z_k}g_k(Z_k) + \sum_{i=j+2}^{j+r}\sum_{k=j+1}^{i-1} (\Id - \alpha \bA)^{j+r-i} \zmfuncA{Z_i}(\Id - \alpha \bA)^{i-k} g_k(Z_k) = \Tterm{1}{1} + \Tterm{2}{1} \eqsp.
    \end{align}
    The first term can be bounded directly as
    \begin{align}
        \label{eq:bound_S2_T1}
        \PE_{\xi}^{1/p}[\norm{\Tterm{1}{1}}^p] \leq \qcond\bConst{A}^2 r (1- \alpha a)^{(r-1)/2} \eqsp.
    \end{align}
    For the second term we can use the Berbee's lemma technique established
    in \Cref{lem:construction_berbee}. Note that after switching the variables, we get
    \begin{align}\label{eq:T2_decompose}
        \Tterm{2}{1} &= \sum_{k=j+1}^{j+r-1}\left[ \sum_{i=k+1}^{j+r} (\Id - \alpha \bA)^{j+r-i} \zmfuncA{Z_i} (\Id - \alpha \bA)^{i-k}\right] g_k(Z_k)\\
        &= \sum_{k=j+1}^{j+r-1} M_{k+1} g_k(Z_k) = \sum_{k=j+1}^{j+r-1} \Slnalpha{k+1}{j+r}{1} (I - \alpha \bA) g_k(Z_k) \eqsp,
    \end{align}
    where 
    \begin{align}
        &M_{k+1} = \sum_{i=k+1}^{j+r} (\Id - \alpha \bA)^{j+r-i} \zmfuncA{Z_i} (\Id - \alpha \bA)^{i-k} \eqsp.
    \end{align}
    For any $m \geq \taumix$ we have the following decomposition
    \begin{align}
        \Slnalpha{k+1}{j+r}{1} = (\Id - \alpha \bA)^{j+r-m-k} \Slnalpha{k+1}{k+m}{1} + \Slnalpha{k+m+1}{j+r}{1}(\Id - \alpha \bA)^{m} \eqsp.
    \end{align}
    Let $N = \lfloor (r-1) / m \rfloor$. Substituting the above relation into \eqref{eq:T2_decompose}, we get
    \begin{align}
        \Tterm{2}{1} &= \sum_{k=(N-1)m + 1}^{j+r-1} \Slnalpha{k+1}{j+r}{1} (\Id - \alpha \bA) g_k(Z_k) + \sum_{k=j+1}^{(N-1)m} (\Id - \alpha \bA)^{j+r-m-k} \Slnalpha{k+1}{k+m}{1} (I - \alpha \bA) g_k(Z_k)\\
        &+ \sum_{k=j+1}^{(N-1)m} \Slnalpha{k+m+1}{j+r}{1} (I - \alpha \bA)^{m+1} g_k(Z_k) = \Tterm{21}{1} + \Tterm{22}{1} + \Tterm{23}{1} \eqsp.
    \end{align}
    Using Minkowski's inequality and \cite[Lemma 5]{durmus2025finite}, we can bound the first term as
    \begin{align}
    \label{eq:bound_S2_T21}
        \PE_{\xi}^{1/p}[\norm{\Tterm{21}{1}}^p] \leq 16 \qcond^2\bConst{A}^2 m r^{1/2} \taumix^{1/2} (1-\alpha a)^{r-1} \eqsp.
    \end{align}
    For the second term, again we can use Minkowski's inequality to get
    \begin{align}
    \label{eq:bound_S2_T22}
        \PE_{\xi}^{1/p}[\norm{\Tterm{22}{1}}^p] &\leq \sum_{k=j+1}^{(N-1)m}\sum_{i=k+1}^{k+m} \PE^{1/p}[\norm{(\Id - \alpha \bA)^{k+m-i} \zmfuncA{Z_i}(\Id - \alpha \bA)^{i-k} g_k(Z_k)}^p]\\
        &\leq \qcond^{3/2} \bConst{A}^2 m r (1 - \alpha a)^{(r-1)/2} \eqsp.
    \end{align}
    For the third term we should use the Berbee lemma technique established
    in \Cref{lem:construction_berbee}. Switching to the extended space $(\tmszn, \tmczn, \PPext_{\nset})$, we have $\PE_{\xi}^{1/p}[\norm{\Tterm{23}{1}}^p] = \PEext_{\xi}^{1/p}[\norm{\Ttermext{23}{1}}^p]$, where
    \begin{align}
        &\Ttermext{23}{1} = \sum_{k=j+1}^{(N-1)m} \Slnalphaext{k+m+1}{j+r}{1}(\Id - \alpha \bA)^{m+1} g_k(\tilde{Z}_k),\\
        &\Slnalphaext{k+m+1}{j+r}{1} = \sum_{i=k+m+1}^{j+r} (\Id -\alpha\bA)^{j+r-i} \zmfuncA{\tilde{Z}_i}(\Id - \alpha\bA)^{i-k-m-1} \eqsp.
    \end{align}
    Thus, we have
    \begin{align}
        \Ttermext{23}{1} &= \sum_{s=0}^{N-2} \sum_{k=j+1}^{j+m} \Slnalphaext{(s+1)m+k+1}{j+r}{1}(\Id - \alpha \bA)^{m+1}g_{sm+k}(\tilde{Z}_{sm + k}^*)\\
        &+ \sum_{s=0}^{N-2} \sum_{k=j+1}^{j+m} \Slnalphaext{(s+1)m+k+1}{j+r}{1}(\Id - \alpha \bA)^{m+1}(g_{sm+k}(\tilde{Z}_{sm+k}) - g_{sm+k}(\tilde{Z}_{sm+k}^*))\\
        &= \Tterm{231}{1} + \Tterm{232}{1} \eqsp.
    \end{align}
    We start with bounding $\Tterm{231}{1}$. Let
    \begin{equation}
        I_1(\alpha) = \sum_{k=j+1}^{j+m} 2p \left\{ \sum_{s=0}^{N-2} (1 -\alpha a)^{sm+k-j-1} \sup_{u' \in \mathbb{S}^{d-1}} \PEext_{\xi}^{2/p}[\norm{\Slnalphaext{(s+1)m+k+1}{j+r}{1}u'}^p] \right\}^{1/2} \eqsp.
    \end{equation}
    Applying \cite[Lemma 6]{durmus2025finite}, we obtain
    \begin{align}
        \PEext_{\xi}^{1/p}&[\norm{\Tterm{231}{1}}^p] \leq \qcond \bConst{A} (1 - \alpha a)^{(m+1)/2}I_1(\alpha)\\
        &+ \sum_{k=j+1}^{j+m}\sum_{s=0}^{N-2} \norm{\xi \{\MKQ^{sm+k}(\Id-\alpha\bA)^{m+1}g_{sm+k}\}}\sup_{u' \in \mathbb{S}^{d-1}} \PEext_{\xi}^{1/p}[\norm{\Slnalphaext{(s+1)m+k+1}{j+r}{1}u'}^p] \eqsp.
    \end{align}
    For the first term, using \cite[Lemma 5]{durmus2025finite}, we have
    \begin{align}
        \label{eq:bound_S2_T2311}
        \sum_{k=j+1}^{j+m} 2p &\left\{ \sum_{s=0}^{N-2} (1 -\alpha a)^{sm+k-j-1} \sup_{u' \in \mathbb{S}^{d-1}} \PEext_{\xi}^{2/p}[\norm{\Slnalphaext{(s+1)m+k+1}{j+r}{1}u'}^p] \right\}^{1/2}\\
        &\leq 32\qcond \bConst{A} \taumix^{1/2} p^{3/2} (1 - \alpha a)^{(r-m-2)/2} m^{1/2} \left\{ \sum_{k=j+1}^{(N-1)m} (j+r-k)\right\}^{1/2}\\
        &\leq 32\qcond \bConst{A} \taumix^{1/2} p^{3/2} m^{1/2} r (1 - \alpha a)^{(r-m-2)/2} \eqsp.
    \end{align}
    For the second term, we know that $\pi(g_{sm+k}) = 0$, and from \Cref{assum:drift} it follows that
    \begin{equation}
        \norm{\xi\MKQ^{sm+k}(\Id - \alpha \bA)^{m+1}g_{sm+k}} \leq \qcond^{1/2}(1- \alpha a)^{(m+1)/2} \dobru{\MKQ^{sm+k}}\|g_{sm+k}\|_{\infty} \eqsp,
    \end{equation}
    and thus
    \begin{align}
    \label{eq:bound_S2_T2312}
        \sum_{k=j+1}^{j+m}\sum_{s=0}^{N-2} &\norm{\xi \MKQ^{sm+k}(\Id-\alpha\bA)^{m+1}g_{sm+k}}\sup_{u' \in \mathbb{S}^{d-1}} \PEext_{\xi}^{1/p}[\norm{\Slnalphaext{(s+1)m+k+1}{j+r}{1}u'}^p] \\
        &\leq 32\qcond^2 \bConst{A}^2 \taumix^{3/2}p^{1/2} r^{1/2} (1 - \alpha a)^{(r-1)/2} \eqsp,
    \end{align}
    where we used that
    \begin{align}
        \sum_{k=j+1}^{j+m}\sum_{s=0}^{N-2} (j+r - (s+1)m -k)^{1/2}\dobru{\MKQ^{sm+k}} \leq 2\taumix r^{1/2} \eqsp.
    \end{align}
    Combining \eqref{eq:bound_S2_T2311} and \eqref{eq:bound_S2_T2312}, we get
    \begin{align}
        \label{eq:bound_S2_T231}
        \PEext_{\xi}^{1/p}[\norm{\Tterm{231}{1}}^p] \leq 32\qcond^2 \bConst{A}^2 ( p m^{1/2} r^{1/2} + \taumix) \taumix^{1/2} p^{1/2} r^{1/2} (1 - \alpha a)^{(r-1)/2} \eqsp.
    \end{align}
    Now, to bound $\Tterm{232}{1}$ we set $V_{l} = g_{l}(\tilde{Z}_{l}) - g_{l}(\tilde{Z}_{l}^*)$ and $\tilde{\mathcal{F}}_l^*=\sigma(\tilde{Z}_i, \tilde{Z}_i^* : i \leq l)$. Using \Cref{lem:construction_berbee}, we get
    \begin{align}
        \PEext_{\xi}^{1/p}[&\norm{\Slnalphaext{(s+1)m+k+1}{j+r}{1}(\Id - \alpha \bA)^{m+1}V_{sm+k}}^p] \\
        &= \PEext_{\xi}^{1/p}[\norm{\Slnalphaext{(s+1)m+k+1}{j+r}{1}(\Id - \alpha \bA)^{m+1}V_{sm+k} \indiacc{\tilde{Z}_{sm+k} \neq \tilde{Z}_{sm+k}^*}}^p]\\
        &\leq \PEext_{\xi}^{1/p}\left[\norm{V_{sm+k}}^p \PEext^{\tilde{\mathcal{F}}_{sm+k}^*}\left[\norm{\Slnalphaext{(s+1)m+k+1}{j+r}{1}(\Id - \alpha \bA)^{m+1}V_{sm+k} / \norm{V_{sm+k}}}^p\right] \right]\\
        &\leq \PEext_{\xi}^{1/p}\left[\norm{V_{sm+k}}^p \sup_{u' \in \mathbb{S}^{d-1}, \xi' \in \mathcal{P}(\Zset)}\PEext_{\xi'}\left[\norm{\Slnalphaext{(s+1)m+k+1}{j+r}{1}(\Id - \alpha \bA)^{m+1}u'}^p\right] \right] \eqsp,
    \end{align}
    where $\mathcal{P}(\Zset)$ is the set of probability measure on $(\Zset, \Zsigma)$. Let
    \begin{equation}
        I_2(\alpha) = \sum_{s=0}^{N-2}\sum_{k=j+1}^{j+m} (j+r - (s+1)m-k)^{1/2}(1-\alpha a)^{(j+r-sm-k)/2}\|g_{sm+k}\|_{\infty} (\dobru{\MKQ^{m}})^{1/p} \eqsp.
    \end{equation}
    Noting that $\norm{V_{sm+k}} \leq 2 \|g_{sm+k}\|_{\infty} \indiacc{\tilde{Z}_{sm+k} \neq \tilde{Z}_{sm+k}^*}$ and applying \cite[Lemma 5]{durmus2025finite}, we obtain
    \begin{align}
    \label{eq:bound_S2_T232}
        \sum_{s=0}^{N-2}\sum_{k=j+1}^{j+m}\PEext_{\xi}^{1/p}[\norm{\Slnalphaext{(s+1)m+k+1}{j+r}{1}&(\Id - \alpha \bA)^{m+1}V_{sm+k}}^p] \leq 2\qcond^{3/2}\bConst{A}(\taumix p)^{1/2}I_2(\alpha)\\
        &\leq 2\qcond^2\bConst{A}^2 (\taumix p)^{1/2} r^{3/2}(1-\alpha a)^{(r-1)/2} (1/4)^{(1/p)\lfloor m/\taumix \rfloor} \eqsp.
    \end{align}
    Setting
    \begin{equation}
        m = \taumix  \left\lceil {p\log{(r)} \over 2 \log{(2)}} \right\rceil \eqsp,
    \end{equation}
    we get $(1/4)^{(1/p)\lfloor m/\taumix \rfloor} \leq r^{-1/2}$ and $m \leq 2\taumix p\log{(r)} / (2\log(2))$. Combining together \eqref{eq:bound_S2_T1}, \eqref{eq:bound_S2_T21}, \eqref{eq:bound_S2_T22}, \eqref{eq:bound_S2_T231} and \eqref{eq:bound_S2_T232} the result follows.
\end{proof}

\begin{proposition}
\label{fact:Hurwitzstability}
Assume that $-\bA$ is Hurwitz. Then there exists a unique symmetric positive definite matrix $Q$ satisfying the Lyapunov equation 
$\bA^\top Q + Q \bA =  \Id$. In addition, setting
\begin{equation}
\label{eq: kappa_def}
a = \normop{Q}^{-1}/2\eqsp, \quad
\text{and} \quad \alpha_\infty = (1/2) \normop{\bA}[Q]^{-2} \normop{Q}^{-1} \wedge \normop{Q} \eqsp,
\end{equation}
it holds for any $\alpha \in [0, \alpha_{\infty}]$ that $\normop{\Id - \alpha \bA}[Q]^2 \leq 1 - a \alpha$, and $\alpha a \leq 1/2$.
\end{proposition}
\begin{proof}
    Proof of this result can be found in \cite[Proposition 1]{durmus2021tight}.
\end{proof}

For a bounded function $f: \Zset \to \rset^d$, we define 
\begin{equation}
    \label{eq:def_sigma_pi_f}
    \sigma_{\pi}^2(f) = \txts \lim_{n\to\infty} n^{-1}\PE[\normLine{\sum_{i=0}^{n-1}\{f(\State_i) - \pi(f)\}}^2] \eqsp.
\end{equation}
\begin{theorem}
\label{theo:rosenthal_uge_arbitrary_init}
Assume \Cref{assum:drift}. Then, for any measurable function $f :\msz\to \rset^{d}$, $ \supnorm{f}  \leq 1$,  $p \geq 2$, and $n \geq 1$, it holds
\begin{align}
\label{eq:rosenthal}
\txts \PE^{1/p}_{\xi}[\norm{\sum_{i=0}^{n-1} f(\State_i)- \pi(f)}^{p}]
&\leq \bConst{\sf{Rm}, 1} \sqrt{2} p^{1/2} n^{1/2} \sigma_\pi(f)  \\
&+ \bConst{\sf{Ros}, 1} n^{1/4}\taumix^{3/4}p\log_2(2p) + \bConst{\sf{Ros}, 2} \taumix p \log_2(2p) \eqsp,
\end{align}
where the constants
$\bConst{\sf{Ros}, 1}, \bConst{\sf{Ros}, 2}, \bConst{\sf{Rm}, 1}$ can be found in \cite[Theorem 6]{durmus2025finite} and $\sigma^2_{\pi}(f)$ is defined in \eqref{eq:def_sigma_pi_f}.
\end{theorem}


Below we establish the result similar to \cite[Lemma 9]{durmus2025finite}. But for our purpose we make it a bit sharper.

\begin{lemma}
\label{lem:bounded_differences_norms_markovian}
Assume \Cref{assum:drift}. Let $\{g_i\}_{i=1}^n$ be a family of measurable functions from $\Zset$ to  $\rset^{d}$ such that $c_i = \supnorm{g_i} < \infty$ for any $i \geq 1$ and $\pi(g_i)= 0$ for any $i \in\{1,\ldots,n\}$.
Then, for any initial probability $\xi$ on $(\Zset,\Zsigma)$, $n \in \nset$, $t \geq 0$, it holds
\begin{equation}
\label{eq:prob_for_norms_markov}
\PP_{\xi}\biggl(\normop{\sum\nolimits_{i=1}^{n}g_i(\State_{i})}\geq t\biggr) \leq 2 \exp\biggl\{-\frac{t^2}{2 u_n^{2}}\biggr\}\eqsp, \text{ where } u_n = 8 \left(\sum_{i=1}^n c_i^2\right)^{1/2} \sqrt{\taumix}\eqsp.
\end{equation}
\end{lemma}
\begin{proof}
    The proof follows the lines of \cite[Lemma 9]{durmus2025finite}.
\end{proof}

\section{Additional experiments}
\label{appendix:add_exps}

In \Cref{fig:mse_remainder}, \Cref{subfig:mse_rescaled_a} we compute $\PE[\norm{\btheta_n - \thetalim + (1/n)\sum_{k=1}^n\funcnoise{Z_k}}^2]$ and $\PE[\norm{\prthalpha{n}{\sf RR} - \thetalim + (1/n)\sum_{k=1}^n\funcnoise{Z_k}}^2]$ estimated by averaging over $N_{\sf{traj}}$ trajectories. The results show that after subtracting the leading term, the remainder term is optimized when \( \alpha \asymp n^{-1/2} \), as predicted by \Cref{thm:error_RR_iter}. In contrast, PR-averaged iterates are optimized in the range \( \alpha \asymp n^{-2/3} \), which is consistent with the theory presented in \cite{durmus2025finite}. Moreover, for $\alpha \asymp n^{-2/3}$, we note that the leading term in \eqref{eq:def_fl_tr_RR} is $(\alpha n)^{-1/2}n^{-1/2}$, and we observe this dependence on $\alpha$ in \Cref{subfig:mse_rescaled_a}. Additionally, \Cref{subfig:mse_rescaled_b} demonstrates that for $\alpha \asymp n^{-1/2}$, the remainder term indeed has an order of $n^{-2/3}$, as predicted by \Cref{corr:error_RR_hp_iter}.

\begin{figure}[ht!]
    \centering
    \begin{subfigure}[b]{0.45\textwidth}
        \centering
        \includegraphics[scale=0.2]{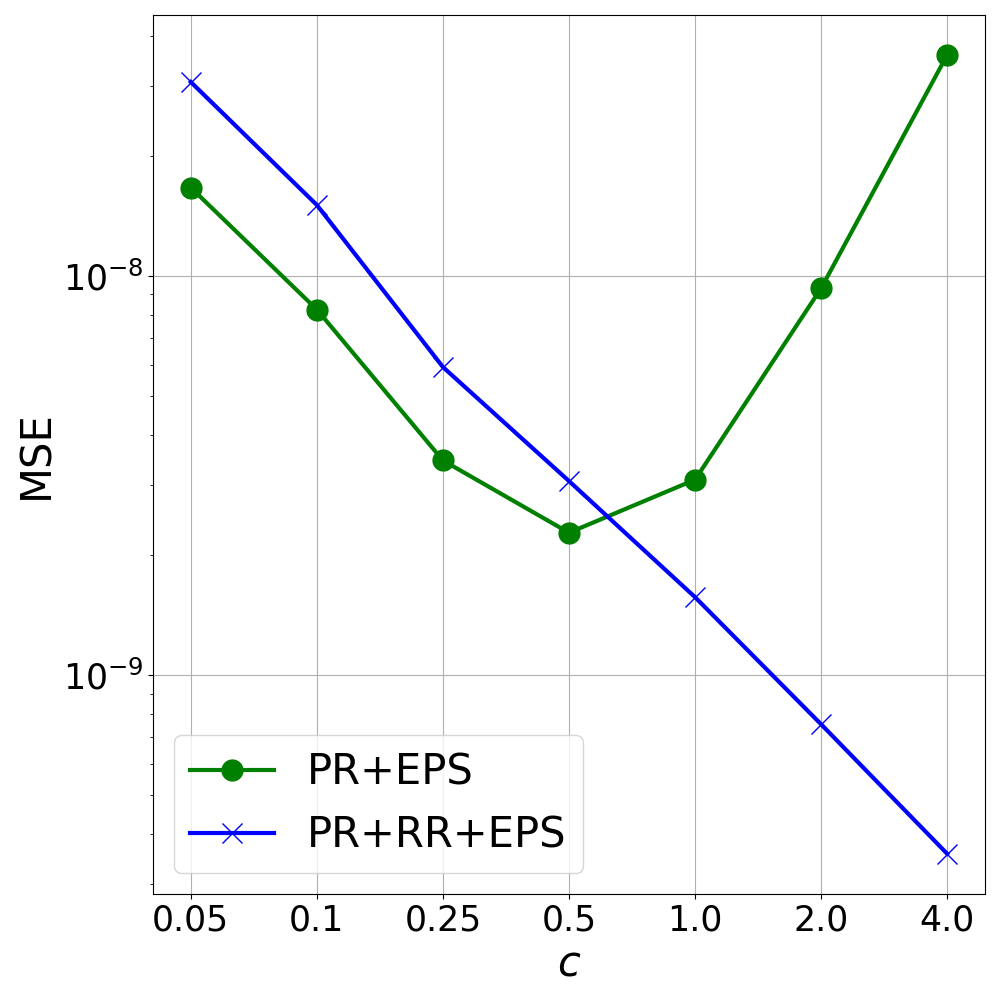}
        \caption{$\alpha = c n^{-2/3}$}
    \end{subfigure}
    \begin{subfigure}[b]{0.45\textwidth}
        \centering
        \includegraphics[scale=0.2]{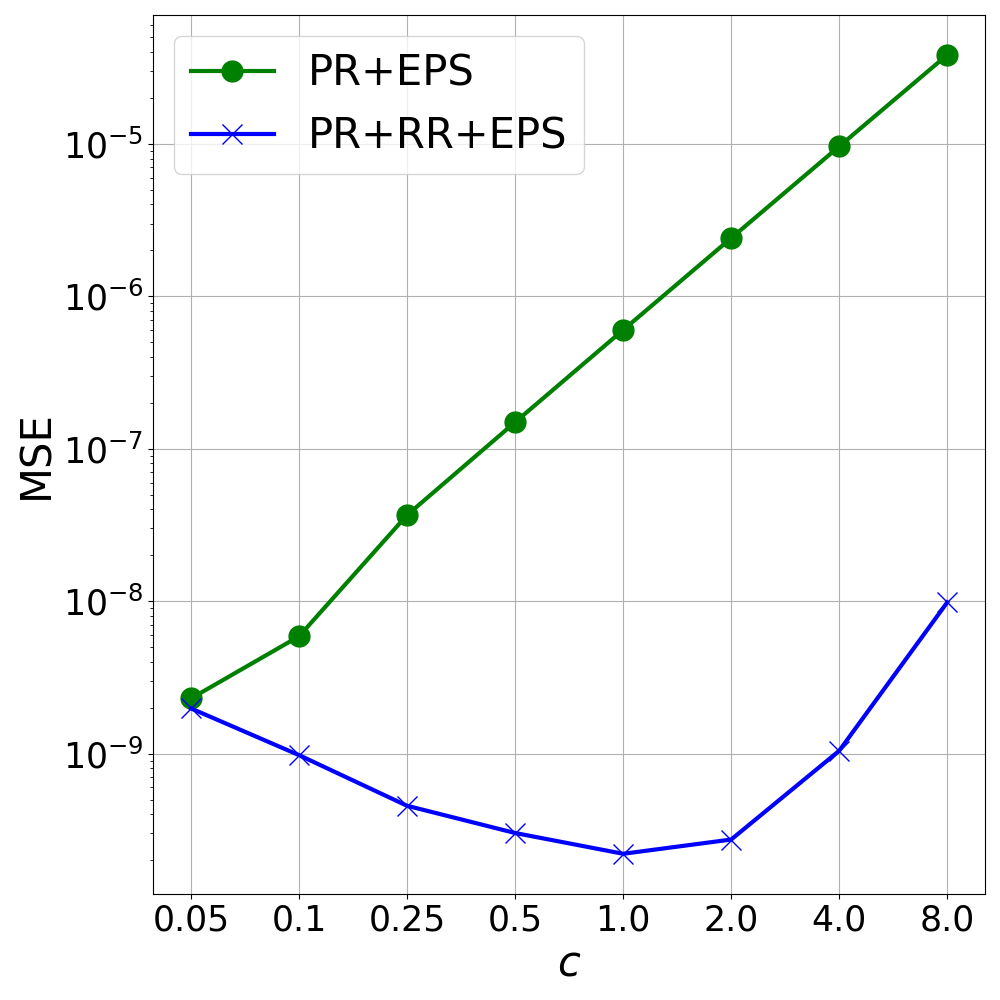}
        \caption{$\alpha = c n^{-1/2}$}
    \end{subfigure}
    \caption{Comparison of remainder errors for Polyak-Ruppert averaged and Richardson-Romberg iterates in two regimes. In the first regime (a), the error attains its optimum for PR averaging, whereas for RR iterates it decays as predicted by \Cref{thm:error_RR_iter}. Conversely, in the second regime (b), the optimum is achieved for RR iterates, which is also consistent with the theory.}
    \label{fig:mse_remainder}
\end{figure}

\begin{figure}[ht!]
    \centering
    \begin{subfigure}[b]{0.45\textwidth}
        \centering
        \includegraphics[scale=0.2]{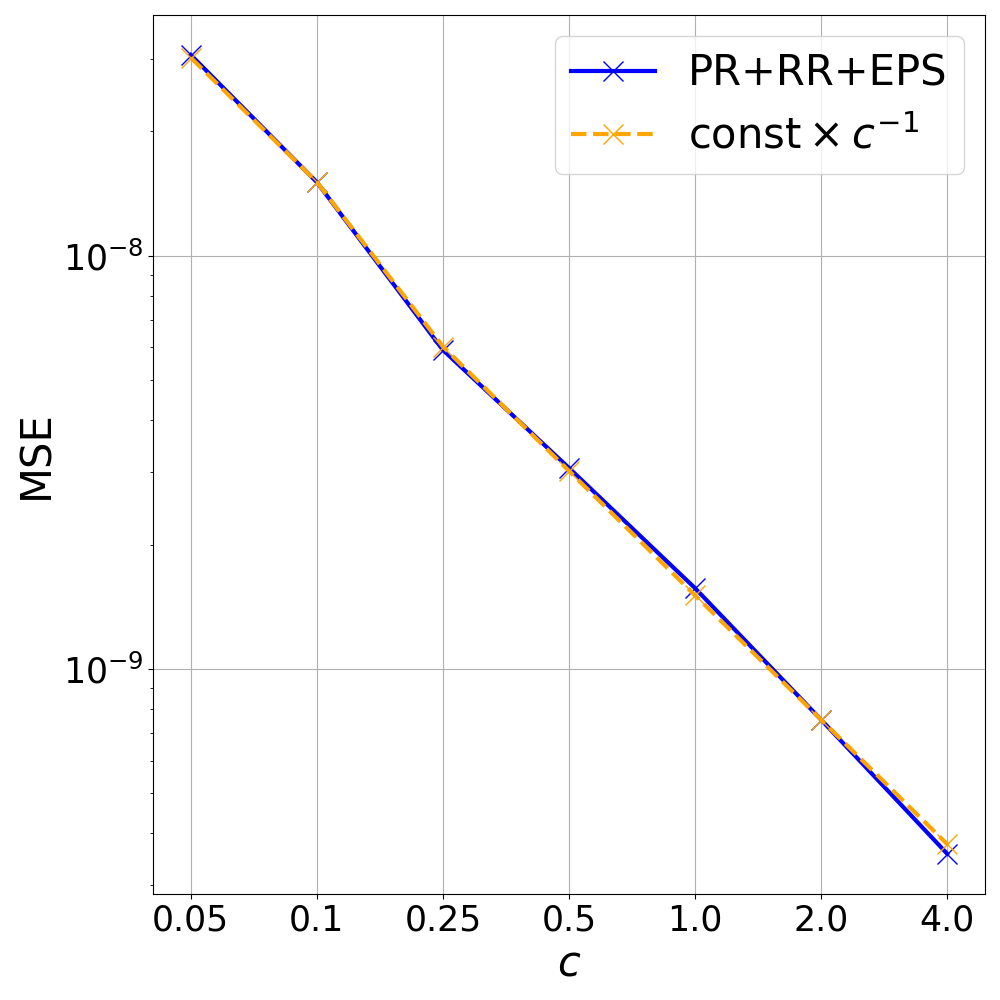}
        \caption{$\alpha = c n^{-2/3}$}
        \label{subfig:mse_rescaled_a}
    \end{subfigure}
    \begin{subfigure}[b]{0.45\textwidth}
        \centering
        \includegraphics[scale=0.2]{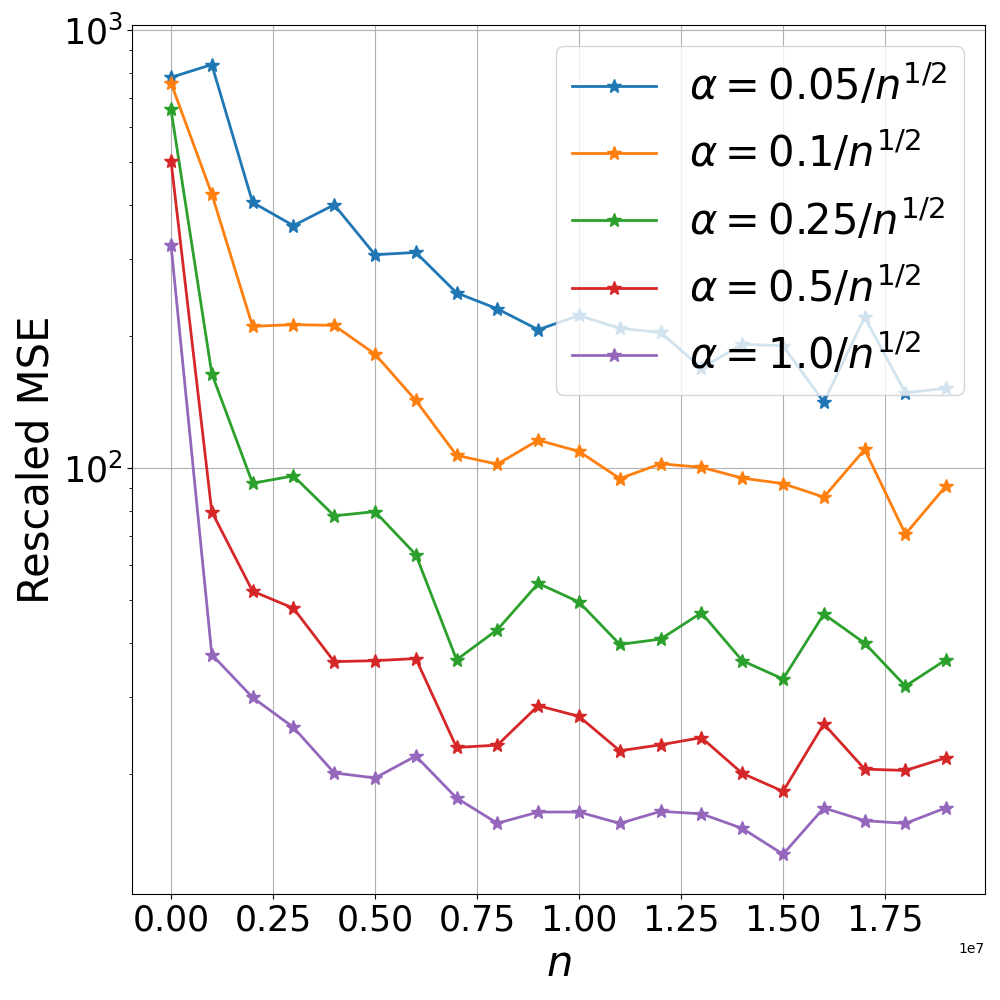}
        \caption{}
        \label{subfig:mse_rescaled_b}
    \end{subfigure}
    \caption{Subfigure (a): the MSE remainder term for the Richardson-Romberg iterates is well approximated by $r c^{-1}$ for some constant $r > 0$, matching the leading term $(\alpha n)^{-1}n^{-1}$ in \eqref{eq:def_fl_tr_RR}. Subfigure (b): rescaled by $n^{4/3}$ MSE remainder trajectories for varying step sizes $\alpha$. These plots cease decaying  and stabilize, confirming the predicted order of the remainder term.}
    \label{fig:mse_rescaled}
\end{figure}

\end{document}